\documentclass{article}

\usepackage{microtype}
\usepackage{graphicx}
\usepackage{subfigure}
\usepackage{booktabs} 

\usepackage{hyperref}

\usepackage[accepted]{icml2023}

\usepackage{amsmath}
\usepackage{amssymb}
\usepackage{mathtools}
\usepackage{amsthm}

\usepackage{natbib}
\usepackage{multirow}
\usepackage{varioref}
\usepackage{threeparttable}
\usepackage{makecell}
\usepackage{bbm}
\usepackage{enumitem}
\usepackage{float}
\usepackage{thmtools, thm-restate}

\usepackage[utf8]{inputenc}

\usepackage{amsmath,amsfonts,bm}

\def\thmref#1{Theorem~\ref{#1}}
\def\propref#1{Proposition~\ref{#1}}
\def\lmmref#1{Lemma~\ref{#1}}

\def\norm#1{\lVert#1\rVert}
\def\inner#1{\left\langle#1\right\rangle}

\def\bignorm#1{\left\lVert#1\right\rVert}
\def\bigabs#1{\left|#1\right|}
\def\bigopen#1{\left(#1\right)}
\def\bigset#1{\left\{#1\right\}}

\newcommand{\mycomment}[1]{}
\newcommand{\GG}[1]{}

\newcommand\note[1]{\textcolor{red}{#1}}

\renewcommand\note[1]{\textcolor{black}{#1}}

\newcommand{\sgdrr}{\textsf{SGD-RR}}

\def\1{\bm{1}}

\def\va{{\bm{a}}}
\def\vb{{\bm{b}}}

\def\vv{{\bm{v}}}

\def\vx{{\bm{x}}}
\def\vy{{\bm{y}}}
\def\vz{{\bm{z}}}

\def\mI{{\bm{I}}}

\DeclareMathAlphabet{\mathsfit}{\encodingdefault}{\sfdefault}{m}{sl}
\SetMathAlphabet{\mathsfit}{bold}{\encodingdefault}{\sfdefault}{bx}{n}

\def\gE{{\mathcal{E}}}
\def\gF{{\mathcal{F}}}

\def\gO{{\mathcal{O}}}

\def\gS{{\mathcal{S}}}

\def\sP{{\mathbb{P}}}

\newcommand{\E}{\mathbb{E}}

\newcommand{\R}{\mathbb{R}}

\DeclareMathOperator*{\argmax}{arg\,max}

\usepackage[capitalize,noabbrev]{cleveref}

\theoremstyle{plain}
\newtheorem{theorem}{Theorem}[section]
\newtheorem{proposition}[theorem]{Proposition}
\newtheorem{lemma}[theorem]{Lemma}

\theoremstyle{definition}
\newtheorem{definition}[theorem]{Definition}
\newtheorem{assumption}[theorem]{Assumption}
\theoremstyle{remark}

\usepackage[textsize=tiny]{todonotes}

\icmltitlerunning{Tighter Lower Bounds for Shuffling SGD: Random Permutations and Beyond}

\begin{document}

\twocolumn[
\icmltitle{Tighter Lower Bounds for Shuffling SGD: Random Permutations and Beyond}



\icmlsetsymbol{equal}{*}

\begin{icmlauthorlist}
    \icmlauthor{Jaeyoung Cha}{kaist}
    \icmlauthor{Jaewook Lee}{kaist}
    \icmlauthor{Chulhee Yun}{kaist}
\end{icmlauthorlist}

\icmlaffiliation{kaist}{Kim Jaechul Graduate School of AI, KAIST, Seoul, South Korea}

\icmlcorrespondingauthor{Chulhee Yun}{chulhee.yun@kaist.ac.kr} 

\icmlkeywords{Optimization, SGD, Shuffling, Lower Bounds} 

\vskip 0.3in
]



\printAffiliationsAndNotice{}  

\begin{abstract}
    We study convergence lower bounds of without-replacement stochastic gradient descent (SGD) for solving smooth (strongly-)convex finite-sum minimization problems. Unlike most existing results focusing on final iterate lower bounds in terms of the number of components $n$ and the number of epochs $K$, we seek bounds for arbitrary weighted average iterates that are tight in all factors including the condition number $\kappa$. For SGD with \emph{Random Reshuffling}, we present lower bounds that have tighter $\kappa$ dependencies than existing bounds. Our results are the first to perfectly close the gap between lower and upper bounds for weighted average iterates in both strongly-convex and convex cases. We also prove weighted average iterate lower bounds for \emph{arbitrary} permutation-based SGD, which apply to all variants that carefully choose the best permutation. Our bounds improve the existing bounds in factors of $n$ and $\kappa$ and thereby match the upper bounds shown for a recently proposed algorithm called \textsf{GraB}.
\end{abstract}



\section{Introduction} \label{sec:1}

One of the most common frameworks used in machine learning is the following finite-sum minimization problem,
\begin{align}
    \min_{\vx} F(\vx) = \frac{1}{n} \sum_{i=1}^{n} f_i(\vx). \label{eq:intro}
\end{align}
Stochastic gradient descent (SGD), an algorithm first proposed by \citet{10.1214/aoms/1177729586}, is highly capable of numerically solving finite-sum optimization problems. In the $t$-th iteration, SGD randomly samples a component index $i(t)$ and computes a gradient-based update equation of the form $\vx_{t} = \vx_{t-1} - \eta_t \nabla f_{i(t)}(\vx_{t-1})$, where $\eta_t$ is a step size parameter, often set to a fixed constant.

Many prior studies on SGD have shown convergence results assuming \textit{with-replacement} sampling of the component index $i(t)$ (\citet{SPS_1999__33__1_0, doi:10.1137/16M1080173, MAL-050} and many others), where we independently choose $i(t)$ from a uniform random distribution over the index set every time. This uniform sampling makes each step of SGD an unbiased noisy estimate of vanilla gradient descent (GD).

In real-world applications, however, it is much more common to use \textit{without-replacement} SGD, where each epoch runs over the entire shuffled set of $n$ components. Without-replacement SGD has gained popularity for both its simplicity and empirical observations of faster convergence rates \citep{bottou, DBLP:journals/mpc/RechtR13, pmlr-v134-open-problem-yun21a}. However, theoretical analysis on without-replacement SGD remains quite elusive, especially because of the lack of independence between iterates. Nevertheless, recent works have managed to successfully deal with without-replacement SGD in theoretical aspects \citep{pmlr-v97-haochen19a, pmlr-v97-nagaraj19a, pmlr-v23-recht12}.

A simple and popular method of without-replacement sampling is to randomly shuffle the $n$ components independently on each epoch, often referred to as \textit{Random Reshuffling} or \sgdrr. Some studies show upper bounds of convergence rates for a certain class of functions \citep{DBLP:journals/siamjo/GurbuzbalabanOP19, ahnyun}, while some others present lower bounds by analyzing a function contained in a certain class with a slow convergence rate \citep{safran2020good, rajput20}. These preliminary results highlight that without-replacement SGD is in fact capable of converging provably faster than its with-replacement counterpart.

A recent line of work \citep{rajput2022permutation, lu2021general, mohtashami2022characterizing} opens a new field of studies on \textit{permutation-based SGD}, which covers all cases where the permutation of the $n$ component functions is chosen according to a certain policy, instead of simple random reshuffling.
The aim of this line of research is to design a policy that yields \emph{faster} convergence compared to random permutations.
Indeed, a recent result by \citet{lu2022grab} proposes \textsf{GraB}, a permutation-based SGD algorithm that uses the gradient information from previous epochs to manipulate the permutation of the current epoch, and shows that \textsf{GraB} provably converges faster than \textit{Random Reshuffling}. This raises the following question:
\begin{equation}
\begin{aligned}
&\,\,\text{
\textit{Is \textsf{GraB} optimal, or can we find an even faster}
}\\
&\,\,\text{
\textit{permutation-based SGD algorithm?}
}
\end{aligned}
\label{eq:quote}
\end{equation}

\subsection{Related Work}

Before summarizing our contributions, we list up related prior results so as to better contextualize our results relative to them. In all convergence rates, we write $\gO (\cdot)$ for upper bounds and $\Omega (\cdot)$ for lower bounds. The tilde notation $\tilde{\gO} (\cdot)$ hides polylogarithmic factors. For simplicity, here we write convergence rates only with respect to the number of component functions $n$ and the number of epochs $K$ (i.e., number of passes through the entire components).

SGD with replacement is known to have a tight convergence rate of $\gO \left( \frac{1}{T} \right)$ after $T$ iterations, which translates to $\gO \left( \frac{1}{nK} \right)$ in our notation.
One of the first studies on \sgdrr{} by \citet{DBLP:journals/siamjo/GurbuzbalabanOP19} shows an \textit{upper bound} of $\tilde{\gO} \left( \frac{1}{K^2} \right)$ for strongly convex objectives with smooth components, along with the assumption that $n$ is a constant.
\citet{pmlr-v97-haochen19a} show a convergence rate of $\tilde{\gO} \left( \frac{1}{n^2 K^2} + \frac{1}{K^3} \right)$ for functions with Lipschitz-continuous Hessians, which explicitly depends on both $n$ and $K$. \citet{rajput20} further show that the upper bound for strongly convex quadratics is $\tilde{\gO} \left( \frac{1}{n^2 K^2} + \frac{1}{nK^3} \right)$. Follow-up studies prove upper bounds in broader settings, such as $\tilde{\gO} \left( \frac{1}{n K^2} \right)$ for strongly convex (but not necessarily quadratic) functions \citep{pmlr-v97-nagaraj19a, ahnyun, mish20}, or $\gO \left( \frac{1}{n^{1/3} K^{2/3}} \right)$ under convex assumptions \citep{mish20}. Some further generalize to other variants of \sgdrr, including Minibatch and Local SGD in federated learning~\citep{yun22}, Nesterov's acceleration~\citep{tran2022nesterov}, or Stochastic Gradient Descent-Ascent used in minimax problems~\citep{hanseul}. Meanwhile, investigations on \textit{lower bounds} have started from simple quadratic assumptions, where \citet{safran2020good} prove a lower bound of rate $\Omega \left( \frac{1}{n^2 K^2} + \frac{1}{n K^3} \right)$. Lower bounds were then extended to smooth and strongly convex settings, as in \citet{rajput20} and \citet{yun22} which both derive a lower bound of $\Omega \left( \frac{1}{n K^2} \right)$.

Recent works provide evidence of designing algorithms that converge faster than \sgdrr.
Concretely, \citet{rajput2022permutation} introduce a permutation-based SGD algorithm called \textsf{FlipFlop} and prove that it can outperform \sgdrr{} for quadratic objectives. 
The authors also propose a lower bound applicable to arbitrary permutation-based SGD, by proving that no algorithm can converge faster than $\Omega \bigopen{\frac{1}{n^3K^2}}$ for some strongly convex objectives.
\citet{lu2021general} and \citet{mohtashami2022characterizing} propose methods to find ``good'' permutations via a greedy strategy. Extending their previous work, \citet{lu2022grab} propose \textsf{GraB} and gain a convergence rate $\tilde{\gO} \bigopen{\frac{1}{n^2K^2}}$ for PŁ functions which is faster than $\tilde{\gO} \bigopen{\frac{1}{nK^2}}$ for \sgdrr{} \citep{ahnyun}.

Most prior results \citep{rajput20, rajput2022permutation} mainly concern achieving tight convergence rates with respect to $n$ and $K$, while recent studies delve deeper to unveil how other parameters can also affect the convergence properties.
The \textit{condition number} $\kappa$ (defined in \cref{sec:2}) is an example of such parameters, which is closely related to the problem's geometry. 
If we take $\kappa$ into account\footnote{For this section, we treat $\kappa = \Theta \left( 1 / \mu \right)$ for simplicity, following the convention of other existing results in the literature \citep{pmlr-v97-haochen19a,pmlr-v97-nagaraj19a,safran_shamir}.} in the strongly convex case, the best known upper and lower bounds for \sgdrr{} are $\tilde{\gO} \bigopen{\frac{\kappa^3}{n K^2}}$ \citep{pmlr-v97-nagaraj19a, ahnyun, mish20} and $\Omega \bigopen{\frac{\kappa}{n K^2}}$ \citep{rajput20, yun22}, which differ by a factor of $\kappa^2$, and those for permutation-based SGD are  $\tilde{\gO} \bigopen{\frac{\kappa^3}{n^2 K^2}}$ \citep{lu2022grab} and $\Omega \bigopen{\frac{1}{n^3 K^2}}$ \citep{rajput2022permutation}, which differ by both $n$ and some factors of $\kappa$---that is, the bounds are \textit{not} completely tight for all factors yet.

While it is tempting to neglect the looseness in $\kappa$ by treating factors in $\kappa$ as ``leading constants,'' characterizing the right dependence on $\kappa$ becomes imperative for understanding the regimes in which without-replacement SGD is faster than the with-replacement version. 
For example, the aforementioned rate $\tilde{\gO} \bigopen{\frac{\kappa^3}{n K^2}}$ of \sgdrr{} improves upon the known tight rate $\gO \bigopen{\frac{\kappa}{nK}}$ of with-replacement SGD only if $K\gtrsim \kappa^2$. It turns out that this requirement of \emph{large enough} $K$ is in fact unavoidable in the strongly convex case~\citep{safran_shamir};
by developing a lower bound, \citet{safran_shamir} show that \sgdrr{} cannot converge faster than with-replacement SGD when $K \lesssim \kappa$. 
Characterizing the exact threshold ($\kappa$ vs.\ $\kappa^2$) for faster convergence of \sgdrr{} requires a tighter analysis of the $\kappa$ dependence of its convergence rate.

\begin{table*}[tb] \footnotesize
    \centering
    \vspace{-10pt} 
    \caption{A comparison of existing convergence rates and our results for permutation-based SGD. Parameters $L$, $\mu$, $\nu$, and $D$ are defined in \cref{sec:2}. Algorithm outputs $\hat{\vx}$, $\hat{\vx}_{\text{tail}}$, and $\hat{\vx}_{\text{avg}}$ are defined in \cref{sec:3}. Function classes $\gF$ and $\gF_{\text{PŁ}}$ are defined in Sections \ref{sec:2} and \ref{sec:4}, respectively. The herding bound $H$, which closely relates to the convergence rate of \cref{alg:grab}, is defined in \cref{sec:4}. The upper bound results are colored white and the lower bound results are colored gray. For a more detailed comparison with prior work, please refer to \cref{tab:summary2} in \cref{app:table}.}
    \label{tab:summary}
    \vspace{3pt} 
    \begin{threeparttable}[t]
    \centering
    \begin{tabular}{|c|c|lcr|} 
        \hline
        \multicolumn{5}{|c|}{Random Reshuffling} \\
        \hline
        Function Class & Output & References & Convergence Rate & Assumptions \\
        \hline \hline
        \multirow{5.5}{*}{$\gF (L, \mu, 0, \nu)$} & \multirow{2.75}{*}{$\vx_n^K$} & \citet{mish20} & $\tilde{\gO} \left( \frac{L^2 \nu^2}{\mu^3 n K^2} \right)$ & $K \gtrsim \kappa$ \\
        & & \cellcolor{gray!20}Ours, \cref{thm:xhat} & \cellcolor{gray!20}$\Omega \left( \frac{L \nu^2}{\mu^2 n K^2} \right)$ & \cellcolor{gray!20}$\kappa \ge c$, $K \gtrsim \kappa$ \\ 
        \cline{2-5}
        & $\hat{\vx}_{\text{tail}}$ & Ours, \cref{thm:tailub} & $\tilde{\gO} \left( \frac{L \nu^2}{\mu^2 n K^2} \right)$ & $K \gtrsim \kappa$ \\
        \cline{2-5}
        & $\hat{\vx}$ & \cellcolor{gray!20}Ours, \cref{thm:xavg}\tnote{\dag} & \cellcolor{gray!20}$\Omega \left( \frac{L \nu^2}{\mu^2 n K^2} \right)$ & \cellcolor{gray!20}$\kappa \ge c$, $K \gtrsim \kappa$ \\
        \hline
        \multirow{2.75}{*}{$\gF (L, 0, 0, \nu)$} & $\hat{\vx}_{\text{avg}}$ & \citet{mish20} & $\gO \left( \frac{L^{1/3} \nu^{2/3} D^{4/3}}{n^{1/3} K^{2/3}} \right)$ & $K \gtrsim \frac{L^2 D^2 n}{\nu^2}$ \\
        \cline{2-5}
        & $\hat{\vx}$ & \cellcolor{gray!20}Ours, \cref{cor:cvxcase}\tnote{\dag} & \cellcolor{gray!20}$\Omega \left( \frac{L^{1/3} \nu^{2/3} D^{4/3}}{n^{1/3} K^{2/3}} \right)$ & \cellcolor{gray!20}$K \gtrsim \max \{ \frac{L^2 D^2 n}{\nu^2}, \frac{\nu}{\mu D n^{1/2}} \}$ \\
        \hline
        \multicolumn{5}{c}{\vspace{-3pt}}\\
        \hline
        \multicolumn{5}{|c|}{Arbitrary Permutations} \\
        \hline
        Function Class & Output & References & Convergence Rate & Assumptions \\
        \hline \hline
        \multirow{2.75}{*}{$\gF (L, \mu, 0, \nu)$} & $\vx_n^K$ & \citet{lu2022grab} (\textsf{GraB}) & $\tilde{\gO} \left( \frac{H^2 L^2 \nu^2}{\mu^3 n^2 K^2} \right)$ & $K \gtrsim \kappa$\\
        \cline{2-5}
        & $\hat{\vx}$ & \cellcolor{gray!20}Ours, \cref{thm:grabLBa} & \cellcolor{gray!20}$\Omega \left( \frac{L \nu^2}{\mu^2 n^2 K^2} \right)$ & \cellcolor{gray!20}-\\
        \hline
        \multirow{2.75}{*}{$\gF_{\text{PŁ}} (L, \mu, \tau, \nu)$} & $\vx_n^K$ & Ours, Proposition \ref{prop:grabUBb} (\textsf{GraB}) & $\tilde{\gO} \left( \frac{H^2 L^2 \nu^2}{\mu^3 n^2 K^2} \right)$ & $n \ge H$, $K \gtrsim \kappa (\tau + 1)$\\
        \cline{2-5}
        & $\hat{\vx}$ & \cellcolor{gray!20}Ours, \cref{thm:grabLBb} & \cellcolor{gray!20}$\Omega \left( \frac{L^2 \nu^2}{\mu^3 n^2 K^2} \right)$ & \cellcolor{gray!20}$\tau = \kappa \ge 8n$, $K \ge \max \{ \frac{\kappa^2}{n}, \kappa^{\frac{3}{2}}n^{\frac{1}{2}} \}$\\
        \hline
    \end{tabular}
    \vspace{2pt} 
    \begin{tablenotes}
        \item[\dag] \small{Additionally assumes $\eta \le \frac{1}{c_2 Ln}$}
    \end{tablenotes}
    \vspace{0pt}
    \end{threeparttable}
    \vspace{-15pt} 
\end{table*}

\subsection{Summary of Our Contributions}
Towards a complete understanding of \sgdrr{} and permutation-based SGD in general, we seek to close the existing gaps outlined above by developing tighter lower bounds with matching upper bounds.
We present results under two different kinds of algorithm settings: \cref{sec:3} contains lower bounds obtained for \sgdrr, and \cref{sec:4} presents lower bounds that are applicable to \textit{arbitrary permutation-based SGD algorithms}.

Our lower bounds are obtained for without-replacement SGD with constant step size, which is also the case in other existing results in the literature~\citep{safran2020good,safran_shamir,rajput20,rajput2022permutation,yun22}. While all lower bounds proved in the aforementioned papers are only applicable to the \emph{final iterate} of the algorithm, many of our results in this paper apply to arbitrary \emph{weighted average} of end-of-epoch iterates, which can be used to show tightness of matching upper bounds that employ iterate averaging.

Our main contributions are as follows. Here we include $\kappa = \Theta \left( 1 / \mu \right)$ in the convergence rates to better describe the results.
Please refer to \cref{tab:summary} for a complete summary.

\begin{itemize}[leftmargin=3.5mm, topsep=-1pt, itemsep=-1pt] 
    \item \cref{thm:xhat} derives a lower bound of rate $\Omega \left( \frac{\kappa^2}{n K^2} \right)$ for the final iterate of \sgdrr{} in the strongly convex case, which matches the best-known corresponding upper bound $\tilde{\gO} \bigopen{\frac{\kappa^3}{n K^2}}$ up to a factor of $\kappa$.
    
    \item \cref{thm:xavg} extends the lower bound $\Omega \left( \frac{\kappa^2}{n K^2} \right)$ under strongly convex settings to arbitrary weighted average iterates of \sgdrr{}. \cref{thm:tailub} shows a matching upper bound $\tilde{\gO} \left( \frac{\kappa^2}{n K^2} \right)$ for the tail average iterate, achieving tightness up to logarithmic factors.
    
    \item \cref{cor:cvxcase} shows a lower bound $\Omega \left( \frac{1}{n^{1/3} K^{2/3}} \right)$ for the average iterate of \sgdrr{} in the convex case, which matches the corresponding upper bound in \citet{mish20}.
    
    \item \cref{thm:grabLBa} provides a lower bound $\Omega \bigopen{\frac{\kappa^2}{n^2 K^2}}$ on arbitrary permutation-based SGD, which, to the best of our knowledge, is the first to match the best-known upper bound of \textsf{GraB} \citep{lu2022grab} in terms of $n$ and $K$.

    \item \cref{thm:grabLBb} relaxes the assumption of individual convexity and obtains a stronger lower bound $\Omega \bigopen{\frac{\kappa^3}{n^2 K^2}}$ in the scenario of arbitrary permutation-based SGD. This lower bound exactly matches the upper bound in all factors, including $\kappa$. Our results therefore answer the question in (\ref{eq:quote}): \textit{Yes, \textup{\textsf{GraB}} is an optimal permutation-based SGD algorithm.}
\end{itemize}

\section{Preliminaries} \label{sec:2}
First we summarize some basic notations used throughout the paper. For a positive integer $N$, we use the notation $[N] := \{ 1, 2, \dots, N \}$. For $\vv \in \R^d$, we denote its $L_2$ and $L_{\infty}$ norm as $\| \vv \|$ and $\| \vv \|_{\infty}$, respectively. We denote the number of component functions as $n$ and the number of epochs as $K$, where $n$ and $K$ are both positive integers. 

Some of our results require large $K$ and we will use $K \gtrsim x$ to express such an assumption. We use $K \gtrsim x$ to denote the condition $K \ge Cx \log \bigopen{\text{poly}(n, K, \mu, L, ...)}$ when $C$ is a numerical constant.

\subsection{Function Class} 

The following definitions help us to formally define the class of problems to which our objective function belongs.
\begin{definition}[Smoothness]
    \label{def:smo}
    A differentiable function $f$ is \textit{$L$-smooth} if
    \begin{align*}
        \| \nabla f (\vx) - \nabla f (\vy) \| \le L \| \vx - \vy \|, \ \forall \vx, \vy \in \R^d.
    \end{align*}
\end{definition}
\begin{definition}[Strong convexity]
    \label{def:sc}
    A differentiable function $f$ is \textit{$\mu$-strongly convex} if
    \begin{align*}
        f(\vy) \ge f(\vx) + \inner{\nabla f(\vx), \vy - \vx} + \frac{\mu}{2} \| \vy - \vx \|^2
    \end{align*}
    for all $\vx, \vy \in \R^d$. If the inequality holds for $\mu = 0$, then we say that $f$ is \textit{convex}.
\end{definition}
\begin{definition}[PŁ condition]
    \label{def:pl}
    A differentiable function $f$ satisfies the \textit{$\mu$-Polyak-Łojasiewicz (PŁ) condition} if
    \begin{align*}
        \frac{1}{2} \| \nabla f (\vx) \|^2 &\ge \mu (f(\vx) - f^*), \ \forall \vx \in \R^d,
    \end{align*}
    where $f^* := \inf_\vx f(\vx) > -\infty$ is the global minimum value of $f$.
\end{definition}
Additionally, we define the \textit{condition number} as $\kappa := L / \mu$, where $L$ is the smoothness constant and $\mu$ is either the strong convexity constant or the PŁ constant.

We also make a common assumption regarding the finite-sum setup in \eqref{eq:intro}, which is that the gradients of the objective function and its components are not too far from each other.
\begin{assumption}[Bounded gradient errors]
    \label{ass:bg}
    There exists $\tau \ge 0$ and $\nu \ge 0$ such that for all $i = 1, \dots, n$,
    \begin{align*}
        \| \nabla f_i (\vx) - \nabla F(\vx) \| &\le \tau \| \nabla F(\vx) \| + \nu, \ \forall \vx \in \R^d.
    \end{align*}
\end{assumption}
Now we define the function class $\gF$ as follows.
\begin{definition}[Function Class]
    \label{def:fclass}
    We define the function class $\gF (L, \mu, \tau, \nu)$ of objective functions $F$ as:
    \begin{align*}
        \gF (L, \mu, \tau, \nu) &:= \{ F: f_i \ \text{are} \ L\text{-\textit{smooth} and \textit{convex}}, \\
        &\phantom{{} := {} \{ F: {}} F \ \text{is} \ \mu\text{-\textit{strongly convex}}, \\
        &\phantom{{} := {} \{ F: {}} F \ \text{and} \ f_i \ \text{satisfy Assumption \ref{ass:bg}} \}.
    \end{align*}
\end{definition}
 Note that \cref{def:fclass} takes into account not only the properties of $F$ but that of the components $f_i$ as well. Also, as seen in \cref{def:sc}, $\gF (L, 0, \tau, \nu)$ corresponds to the case where $F$ is convex.

\paragraph{Remark.} One may concern that \cref{ass:bg} is too ``strong'' compared to common assumptions used for upper bounds, e.g., the bounded variance assumption:
\begin{align*}
    \mathbb{E} [\| \nabla f_i (\boldsymbol{x}) - \nabla F (\boldsymbol{x}) \|^2] \le \tau' \| \nabla F (\boldsymbol{x}) \|^2 + \nu'.
\end{align*}
However, we would like to emphasize that posing stronger assumptions does \textit{not} lead to weaker results in the case of lower bounds. This is because for two function classes with $\mathcal F \subset \mathcal F'$, a lower-bound-achieving function $f \in \mathcal{F}$ must also satisfy $f \in \mathcal{F}'$, i.e., $f$ also establishes the same lower bound for $\mathcal{F}'$. For our case, if the components satisfy \cref{ass:bg}, then the function will also satisfy the bounded variance assumption for constants $\tau' = 2 \tau^2$ and $\nu' = 2 \nu^2$.

\subsection{Algorithms}

\begin{algorithm}[tb]
    \caption{\textsf{Offline GraB} \citep{lu2022grab}}
    \label{alg:grab}
    \begin{algorithmic}
        \STATE {\bfseries Input:} Initial point $\vx_0 \in \R^d$, Learning rate $\eta > 0$, Number of epochs $K$, Nonnegative weights $\{\alpha_k\}_{k=1}^{K+1}$, Initial order $\sigma_1$
        \STATE Initialize $\vx_{0}^{1} = \vx_0$
        \FOR{$k = 1, \dots, K$}
            \FOR{$i = 1, \dots, n$}
                \STATE Compute gradient: $\nabla f_{\sigma_k(i)} \bigopen{\vx_{i-1}^k}$.
                \STATE Store the gradient: $\vz_i \leftarrow \nabla f_{\sigma_k(i)} \bigopen{\vx_{i-1}^k}$.
                \STATE Optimizer step: $\vx_{i}^{k} = \vx_{i-1}^{k} - \eta \vz_i$
            \ENDFOR
            \STATE $\vx_{0}^{k+1} = \vx_{n}^{k}$
            \STATE Compute gradient mean: $\vz \leftarrow \frac{1}{n} \sum_{i=1}^n \vz_i$
            \STATE Generate new order: $\sigma_{k+1} \leftarrow \text{Herding}\bigopen{\bigset{\vz_i - \vz}_{i=1}^n}$
        \ENDFOR
        \STATE {\bfseries Output:} $\hat{\vx} = \sum_{k=1}^{K+1} \alpha_k \vx_0^k / \sum_{k=1}^{K+1} \alpha_k$
    \end{algorithmic}
\end{algorithm}

We denote the $i$-th iterate of the $k$-th epoch of permutation-based SGD by $\vx_{i}^{k}$, where $i = 0, \dots, n$ and $k = 1, \dots, K$. We denote the distance between the initial point $\vx_0^1$ and the optimal point $\vx^*$ as $D := \| \vx_0^1 - \vx^* \|$. We also follow the conventional notation $\vx_{0}^{k+1} = \vx_{n}^{k}$, which indicates that the final result of an epoch becomes the initial point of its subsequent epoch. At the beginning of the $k$-th epoch, we choose a permutation $\sigma_k : [n] \rightarrow [n]$. The algorithm then accesses the component functions in the order of $f_{\sigma_k(1)}$, $\dots$, $f_{\sigma_k(n)}$. That is, we use the following update equation:
\begin{align*}
    \vx_{i}^{k} &= \vx_{i-1}^{k} - \eta \nabla f_{\sigma_k(i)} (\vx_{i-1}^{k})
\end{align*}
for $i = 1, \dots, n$, where $\eta > 0$ is a constant step size.

We particularly focus on two different types of permutation-based SGD. \cref{sec:3} states theoretical results based on \sgdrr, which assumes that the components are randomly shuffled independently in each epoch.

In \cref{sec:4}, we study the case when permutations can be carefully chosen to gain faster convergence.
We provide lower bounds that are applicable to \emph{any} kind of permutation-based SGD.
To show our lower bound is tight, it suffices to show that a \emph{specific} permutation-based SGD algorithm provides a matching upper bound. 
To this end, we use \textit{offline herding} SGD~\citep{lu2022grab}, where the components are manually ordered to ``balance'' the gradients.

Specifically, \citet{lu2021general} prove that as the gap between the partial sums of consecutive stochastic gradients and the full gradient diminishes faster, the optimizer converges faster as well.
In their subsequent work~\citep{lu2022grab}, they first propose \textit{offline herding SGD}, a permutation-based SGD algorithm that manages this gap via the herding algorithm but requires intensive memory consumption, and devise \textit{online herding SGD} (or \textsf{GraB}) that successfully overcomes the memory challenges. 
They prove that both algorithms guarantee the same convergence rate $\tilde{\gO} \bigopen{\frac{1}{n^2 K^2}}$. 
In our setting, since we are not interested in the usability of algorithms, we will focus on \textit{offline herding SGD} (or \textsf{Offline GraB}) just for simplicity. 
\cref{alg:grab} provides a pseudocode of \textsf{Offline GraB}.
For the description of Herding subroutine in \cref{alg:grab}, see Assumption \ref{ass:hb} and \cref{subsec:43}.
\section{Random Reshuffling} \label{sec:3}

Here we show lower bounds of \sgdrr{} on the last iterate and arbitrary weighted averaged iterates for strongly convex objectives and then extend results to convex functions. We stress that the lower bounds on weighted average iterates tightly match the upper bounds both for the strongly convex and convex case. 

\subsection{Lower Bound for the Final Iterate} \label{subsec:31}

\cref{thm:xhat} provides a lower bound for the final iterate of \sgdrr{} for arbitrary step sizes $\eta > 0$ in the $\mu$-strongly convex case.

\begin{restatable}{theorem}{xhat}
    \label{thm:xhat}
    For any $n \ge 2$ and $\kappa \ge c_1$, there exists a $3$-dimensional function $F \in \gF (L, \mu, 0, \nu)$ and an initialization point $\vx_0$ such that for any constant step size $\eta$, the final iterate $\vx_n^K$ of \sgdrr{} satisfies
    \begin{align*}
        \E \left[ F (\vx_n^K) - F^* \right] &= \begin{cases}
        \Omega \left( \frac{L \nu^2}{\mu^2 n K^2} \right), &\text{if} \ K \ge c_2 \kappa, \\
        \Omega \left( \frac{\nu^2}{\mu n K} \right), &\text{if} \ K < c_2 \kappa,
        \end{cases}
    \end{align*}
    for some universal constants $c_1, c_2$.
\end{restatable}
We take an approach similar to \citet{yun22}, which is to construct $F$ by aggregating three functions, each showing a lower bound for a different step size regime. The proof of \cref{thm:xhat} is deferred to \cref{app:xhat}.

We can compare \cref{thm:xhat} with results of \citet{yun22} for $M = B = 1$, which establishes a lower bound of $\Omega ( \frac{\nu^2}{\mu n K^2} )$ for the large epoch regime $K \gtrsim \kappa$ and $\Omega ( \frac{\nu^2}{\mu n K})$ for the small epoch regime $K \lesssim \kappa$. We can observe that the lower bound in \cref{thm:xhat} for the large epoch regime is tightened by a factor of $\kappa$. In fact, the bound can be compactly written as:
\begin{align*}
    \E \left[ F (\vx_n^K) - F^* \right] &= \Omega \left( \frac{\nu^2}{\mu n K} \cdot \min \left\{ 1, \frac{\kappa}{K} \right\} \right),
\end{align*}
which can be interpreted as a \textit{continuous} change from $\Omega \left( \frac{\nu^2}{\mu n K} \right)$ to $\Omega \left( \frac{\kappa \nu^2}{\mu n K^2} \right)$ as $K$ gradually increases past the threshold $K \ge c_2 \kappa$.

We can also compare our results with \citet{safran_shamir}, which provide a lower bound of rate $\Omega \left( \frac{\nu^2}{\mu n K} \cdot \min \left\{ 1, \frac{\kappa}{K} \left( \frac{1}{n} + \frac{\kappa}{K} \right) \right\} \right)$ under a stronger assumption that the objective and components are all \emph{quadratic}. The lower bound for the small $K \lesssim \kappa$ regime is identical to ours since for this case our lower bound also relies on quadratic functions. However, if $K$ grows past $\Omega (\kappa)$, then we can observe that the lower bound in \cref{thm:xhat} derived from non-quadratic functions is tighter by a factor of $\left( \frac{1}{n} + \frac{\kappa}{K} \right)$.

An upper bound for \sgdrr{} in the $\mu$-strongly convex case under the step-size condition $\eta = \gO (\frac{1}{Ln})$ is introduced in Theorem~2 of \citet{mish20}.
\begin{proposition}[Corollary of \citet{mish20}, Theorem~2]
    \label{prop:strcvxub}
    Suppose that $F$ and $f_1, \dots, f_n$ are all $L$-smooth, $f_1, \dots, f_n$ are convex, and $F$ is $\mu$-strongly convex. Also, let us define
    \begin{align}
        \sigma_*^2 := \frac{1}{n} \sum_{i=1}^{n} \| \nabla f_i (\vx^*) \|^2. \label{eq:sigmastar}
    \end{align}
    Then, for \sgdrr{} with constant step size
    \begin{align*}
        \eta = \min \left\{ \frac{2}{L n}, \frac{1}{\mu n K} \log \left( \frac{\mu^{3} n D^2 K^2}{L \sigma_*^2} \right) \right\},
    \end{align*}
    the final iterate $\vx_n^K$ satisfies
    \begin{align*}
        \E \left[ F (\vx_n^K) - F^* \right] = \tilde{\gO} \left( L D^2 e^{- \frac{K}{\kappa}} +  \frac{L^2 \sigma_*^2}{\mu^3 n K^2} \right).
    \end{align*}
\end{proposition}

Note that the above proposition uses $\sigma_*^2$ which only relies on the gradients at the optimal point $\vx^*$, while our lower bounds involve $\nu^2$ which bounds the gradients for all $\vx$. However, we can easily observe that Assumption \ref{ass:bg} with $\tau = 0$ and $\vx = \vx^*$ implies that $\| \nabla f_i (\vx^*) \|^2 \le \nu^2$ for all $i$, and hence $\sigma_*^2 \le \nu^2$. Therefore it is safe to compare this upper bound with our lower bounds by simply substituting the $\sigma_*^2$ terms with $\nu^2$. Note that the same applies to \cref{prop:mishcvx}.
    
Now, assuming $K \gtrsim \kappa$ so that the learning rate becomes
\begin{align*}
    \eta = \frac{1}{\mu n K} \log \left( \frac{\mu^{3} n D^2 K^2}{L \nu^2} \right) \le \frac{2}{L n},
\end{align*}
then we have $\E \left[ F (\vx_n^K) - F^* \right] = \tilde{\gO} \left( \frac{L^2 \nu^2}{\mu^3 n K^2} \right)$, and for this case we can observe that lower bound shown in \cref{thm:xhat} matches the upper bound in Proposition \ref{prop:strcvxub} up to a factor of $\kappa = \frac{L}{\mu}$ and some polylogarithmic factors.

\subsection{Lower Bound for Weighted Average Iterates} \label{subsec:32}

For small step sizes $\eta = \gO \left( \frac{1}{Ln} \right)$, we can extend \cref{thm:xhat} to \textit{arbitrary weighted average (end-of-epoch) iterates}. That is, \cref{thm:xavg} provides a lower bound which is applicable to all linear combinations of the following form,
\begin{align}
    \hat{\vx} &= \frac{\sum_{k=1}^{K+1} \alpha_k \vx_0^k}{\sum_{k=1}^{K+1} \alpha_k}, \label{eq:hatxavg}
\end{align}
for nonnegative weights $\alpha_k \ge 0$ for all $k = 1, \dots, K+1$. 

\begin{restatable}{theorem}{xavg}
    \label{thm:xavg}
    For any $n \ge 2$ and $\kappa \ge c_1$, there exists a $2$-dimensional function $F \in \gF (L, \mu, 0, \nu)$ and an initialization point $\vx_0$ such that for any constant step size $\eta \le \frac{1}{c_2 Ln}$, any weighted average iterate $\hat{\vx}$ of \sgdrr{} of the form as in \eqref{eq:hatxavg} satisfies
    \begin{align*}
        \E \left[ F (\hat{\vx}) - F^* \right] &= \begin{cases}
        \Omega \left( \frac{L \nu^2}{\mu^2 n K^2} \right), &\text{if} \ K \ge c_2 \kappa, \\
        \Omega \left( \frac{\nu^2}{\mu} \right), &\text{if} \ K < c_2 \kappa,
        \end{cases}
    \end{align*}
    for the same universal constants $c_1, c_2$ as in \cref{thm:xhat}.
\end{restatable}
The full proof of \cref{thm:xavg} can be found in \cref{app:xavg}. 
Note that, for weighted average iterates, we restrict ourselves to small step sizes $\eta = \gO (\frac{1}{Ln})$; while this could look restrictive, such a choice of step size is commonly used in existing upper bounds, and we will see shortly that our lower bound exactly matches an upper bound when $K \gtrsim \kappa$ (\cref{thm:tailub}). The tightness also extends to general convex cases, as seen in \cref{subsec:33}.

One might wonder why the lower bound becomes a \textit{constant} for small $K \lesssim \kappa$. This is because in the $\eta = \gO (\frac{1}{Ln})$ regime, $K < c_2 \kappa$ implies $\eta < \frac{1}{c_2 L n} \le \frac{1}{\mu n K}$, i.e., the step size is too small for SGD to reach the optimum in $K$ steps. For instance, $K$ epochs of SGD on $F(x) = f_i(x) = \frac{\mu}{2} x^2$ initialized at $x = x_0$ reaches the point $(1 - \eta \mu)^{nK} x_0 > (1 - \frac{1}{nK})^{nK} x_0 \ge \frac{x_0}{4}$. Hence the iterate cannot get past $\frac{x_0}{4}$, rendering it impossible to reach the optimal point $x^* = 0$. 

The difficulty of extending the $\eta = \Omega (\frac{1}{Ln})$ regime in \cref{thm:xhat} to arbitrary weighted average iterates originates from our proof strategy: for small enough $\eta$, we can show for our worst-case examples that all $\vx_0^k$'s (in expectation) stay on the positive side bounded away from zero, thereby proving that any weighted average also stays sufficiently far from zero. However, for larger $\eta$, the iterates may oscillate between positive and negative regions, making it possible for an average iterate to converge faster than individual $\vx_0^k$'s.

Note that our definition in \eqref{eq:hatxavg} \textit{includes} the final iterate, as the choice $\alpha_{k} = 0$ for $1 \le k \le K$ and $\alpha_{K+1} = 1$ yields $\hat{\vx} = \vx_0^{K+1} = \vx_n^K$. Different forms of algorithm outputs other than the final iterate also frequently appear in prior works, especially regarding upper bounds for \sgdrr. For instance, we may choose $\alpha_k = 1$ for all $2 \le k \le K+1$ and $\alpha_1 = 0$ to represent the \textit{average iterate} $\hat{\vx}_{\text{avg}} := \frac{1}{K} \sum_{k=1}^{K} \vx_n^k$ \citep{mish20}. We may also set $\alpha_k = 1$ for $\lceil \frac{K}{2} \rceil + 1 \le k \le K+1$ and $\alpha_k = 0$ otherwise to recover the \textit{tail average iterate} \citep{pmlr-v97-nagaraj19a}, defined as $\hat{\vx}_{\text{tail}} := \frac{1}{K - \lceil \frac{K}{2} \rceil + 1} \sum_{k=\lceil \frac{K}{2} \rceil}^{K} \vx_n^{k}$.

We further show that the lower bound  in \cref{thm:xavg} tightly matches the upper bound suggested in \cref{thm:tailub}.
\begin{restatable}{proposition}{tailub}
    \label{thm:tailub}
    Suppose that $F \in \gF (L, \mu, 0, \nu)$, and that we choose $\eta$ as
    \begin{align*}
        \eta &= \min \left\{ \frac{1}{\sqrt{2}Ln}, \frac{9}{\mu n K} \max \left\{ 1, \log \left( \frac{\mu^3 n D^2 K^2}{L \nu^2} \right) \right\} \right\}.
    \end{align*}
    Then, for \sgdrr{} with constant step size $\eta$ and $K \ge 5$, the \textit{tail average iterate} $\hat{\vx}_{\text{\emph{tail}}}$ satisfies:
    \begin{align*}
        \mathbb{E} \left[ F (\hat{\vx}_{\text{\emph{tail}}}) - F^* \right] &= \tilde{\gO} \left( \frac{L D^2}{K} e^{- \frac{1}{9 \sqrt{2}} \frac{K}{\kappa}} + \frac{L \nu^2}{\mu^2 n K^2} \right).
    \end{align*}
\end{restatable}
See \cref{app:sgdrrub} for a full proof of \cref{thm:tailub}.

Assuming $K \gtrsim \kappa$ so that the learning rate becomes
\begin{align*}
    \eta &= \frac{9}{\mu n K} \max \left\{ 1, \log \left( \frac{\mu^3 n D^2 K^2}{L \nu^2} \right) \right\} \le \frac{1}{\sqrt{2}Ln},
\end{align*}
then we have $\E \left[ F (\hat{\vx}_{\text{tail}}) - F^* \right] = \tilde{\gO} \left( \frac{L \nu^2}{\mu^2 n K^2} \right)$ (see Cases (c), (d) in the proof). Then we can observe that the lower bound shown in \cref{thm:xavg} exactly matches the upper bound, ignoring polylogarithmic terms.

By introducing the tail average $\hat{\vx}_{\text{tail}}$, we can obtain a rate of $\tilde{\gO} \left( \frac{L \nu^2}{\mu^2 n K^2} \right)$ which is tighter than the rate $\tilde{\gO} \left( \frac{L^2 \nu^2}{\mu^3 n K^2} \right)$ for the final iterate $\vx_n^K$ by a factor of $\kappa$. Whether we can achieve the same, stronger upper bound for the final iterate $\vx_n^K$ or not is still an open problem.

\subsection{Extension to Convex Objectives} \label{subsec:33}

One important implication of \Cref{thm:xavg} is that we can carefully choose a small value of $\mu$ to derive a lower bound that exactly matches the upper bound for \textit{convex} objectives. \cref{cor:cvxcase} extends \cref{thm:xavg} to the convex case. 
\begin{restatable}{corollary}{corcvxcase}
    \label{cor:cvxcase}
    For any $n \ge 2$, there exists a $2$-dimensional function $F \in \gF (L, 0, 0, \nu)$ such that if
    \begin{align}
        K &\ge c_3 \max \left\{ \frac{L^2 D^2 n}{\nu^2}, \frac{\nu}{\mu D n^{1/2}} \right\}, \label{eq:klowerbound}
    \end{align}
    then for any constant step size $\eta \le \frac{1}{c_2 L n}$, any weighted average iterate $\hat{\vx}$ of \sgdrr{} of the form as in \eqref{eq:hatxavg} satisfies
    \begin{align*}
        \E \left[ F(\hat{\vx}) - F^* \right] = \Omega \left( \frac{L^{1/3} \nu^{2/3} D^{4/3}}{n^{1/3} K^{2/3}} \right),
    \end{align*}
    for some universal constants $c_2$ and $c_3$.
\end{restatable}
We defer the proof of \cref{cor:cvxcase} to \cref{subsec:b2}.

A matching upper bound for \sgdrr{} for the convex case under the step-size condition $\eta = \gO (\frac{1}{Ln})$ is introduced in Theorem~3 of \citet{mish20}.
\begin{proposition}[\citet{mish20}, Theorem~3]
    \label{prop:mishcvx}
    Suppose that $F$ and $f_1, \dots, f_n$ are all $L$-smooth and $f_1, \dots, f_n$ are convex. Also, suppose that we define $\sigma_*^2$ as in (\ref{eq:sigmastar}). Then, for \sgdrr{} with constant step size
    \begin{align*}
    \eta = \min \left\{ \frac{1}{\sqrt{2} L n}, \left( \frac{D^2}{L \sigma_*^2 n^2 K} \right)^{1/3} \right\},
    \end{align*}
    the \textbf{average iterate} $\hat{\vx}_{\text{\emph{avg}}} := \frac{1}{K} \sum_{k=1}^{K} \vx_n^k$ satisfies
    \begin{align*}
        \mathbb{E} \left[ F(\hat{\vx}_{\text{\emph{avg}}}) - F^* \right] &= \gO \left( \frac{L D^2}{K} + \frac{L^{1/3} \sigma_*^{2/3} D^{4/3}}{n^{1/3} K^{2/3}} \right).
    \end{align*}
\end{proposition}
With the same logic as in \cref{prop:strcvxub}, we can compare the above results with our lower bounds by substituting $\sigma_*^2$ with $\nu^2$.

In Proposition \ref{prop:mishcvx}, if we have a large number of epochs with $K \ge \frac{2 \sqrt{2} L^2 D^2 n}{\nu^2}$, then $\eta = \left( \frac{D^2}{L \nu^2 n^2 K} \right)^{1/3} \le \frac{1}{\sqrt{2} L n}$ yields
\begin{align*}
    \mathbb{E} \left[ F(\hat{\vx}_{\text{avg}}) - F^* \right] &= \gO \left( \frac{L^{1/3} \nu^{2/3} D^{4/3}}{n^{1/3} K^{2/3}} \right).
\end{align*}
 For a large $K$ regime of $K = \Omega \left( \frac{L^2 D^2 n}{\nu^2} + \frac{\nu}{\mu D n^{1/2}} \right)$, we may choose $\alpha_k = 1$ for all $k = 2, \dots, K+1$ and $\alpha_1 = 0$ so that $\hat{\vx} = \hat{\vx}_{\text{avg}}$, and then observe that the lower bound in \cref{cor:cvxcase} exactly matches the upper bound in Proposition \ref{prop:mishcvx}. Note that the lower bound of $K$ in \eqref{eq:klowerbound} reduces to $\Omega \left( \frac{L^2 D^2 n}{\nu^2} \right)$ when $n = \Omega \left( \frac{\nu^2}{\mu^{2/3} L^{4/3} D^2} \right)$, which then matches the epoch requirement that arises in the upper bound.
\section{Arbitrary Permutation-based SGD} \label{sec:4}
So far, we have considered the case where permutations are randomly shuffled for each epoch.
In this section, we study the scenario when permutations can be chosen \emph{manually} rather than randomly.
We provide lower bounds that are applicable to any arbitrary permutation-based SGD.
Our lower bounds match the previously established upper bound in terms of $n$ and $K$, and can further match with respect to $\kappa$ when the objective is ill-conditioned.

\subsection{Lower Bound with Component Convexity} \label{subsec:41}
\cref{thm:grabLBa} establishes a lower bound on arbitrary weighted average (end-of-epoch) iterates applicable to any permutation-based SGD.
\begin{restatable}{theorem}{grabLBa}
    \label{thm:grabLBa}
    For any $n \ge 2$ and $\kappa \ge 4$, there exists a $4$-dimensional function $F \in \gF(L,\mu,0,\nu)$ and an initialization point $\vx_0$ such that for any permutation-based SGD with any constant step size $\eta$, any weighted average iterate $\hat{\vx}$ of the form as in \cref{eq:hatxavg} satisfies
    \begin{align*}
        F(\hat{\vx}) - F^* = \Omega \left( \frac{L \nu^2}{\mu^2 n^2 K^2} \right).
    \end{align*}
\end{restatable}
The main technical difficulty in proving \cref{thm:grabLBa} is that we must construct an objective that demonstrates a ``slow'' convergence rate for every permutation over $K$ epochs.
To achieve this, we design an objective that pushes $\vx_n^k$ toward a constant direction, regardless of the permutation. 
The constructed objective belongs to the class $\gF(L,\mu,0,\nu)$ and satisfies component convexity.
Here we note that our proof technique does \textit{not} require any assumptions about large epochs.
Furthermore, in contrast to the \sgdrr{} case (\cref{thm:xavg} and \cref{cor:cvxcase}), this lower bound covers the entire range of step sizes.
The full proof of \cref{thm:grabLBa} is written in \cref{sec:grabLBa}.

As mentioned in \cref{subsec:31}, applying $\alpha_{k} = 0$ for $1 \le k \le K$ and $\alpha_{K+1} = 1$ yields the lower bound for the last iterate. 
Our result significantly improves the previous lower bound and also matches the known upper bound of permutation-based SGD which will be discussed later in this section.

\paragraph{Comparison with the Previous Work.} 
To the best of our knowledge, the best-known lower bound that holds for any arbitrary permutation-based SGD is proved by \citet{rajput2022permutation} prior to our work. 
Specifically, the authors show that there exists a $(2n+1)$-dimensional function $F \in \gF \bigopen{2L, \frac{n-1}{n}L, 1, \nu}$ such that for any permutation-based SGD with any constant step size,
\begin{align}
    F(\vx_n^K) - F^* = \Omega \bigopen{\frac{\nu^2}{L n^3 K^2}}. \label{eq:rajput}
\end{align}
Thus, \cref{thm:grabLBa} improves the lower bound rate by a factor of $n$ and sharpens the dependency on $\kappa$.

Before we state the matching upper bound, we define an additional assumption and a function class.
\begin{assumption}[Herding bound]
    \label{ass:hb}
    There exists an algorithm that has the following property: Given $\vz_1, \dots, \vz_{n} \in \R^d$ satisfying $\bignorm{\vz_i} \le 1$ for $\forall i \in [n]$ and $\sum_{i=1}^{n} \vz_i = 0$, the algorithm outputs a permutation $\sigma : [n] \rightarrow [n]$ such that $\max_{k \in \{1, \dots, n\}} \bignorm{\sum_{i=1}^k \vz_{\sigma(i)}} \le H$.
\end{assumption}
We call an algorithm considered in Assumption~\ref{ass:hb} as the \textit{Herding algorithm}, used as a subroutine in Algorithm~\ref{alg:grab}.
\begin{definition}[Function class]
    \label{def:fplclass}
    We define the function class $\gF_{\emph{\text{PŁ}}}$ as follows.
    \begin{align*}
        \gF_{\text{\emph{PŁ}}} (L, \mu, \tau, \nu) &:= \{ F: f_i \ \text{are} \ L\text{-\textit{smooth}}, \\
        &\phantom{{} := {} \{ F: {}} F \ \text{satisfies} \ \mu\text{-\textit{PŁ condition}}, \\
        &\phantom{{} := {} \{ F: {}} F \ \text{and} \ f_i \ \text{satisfy Assumption \ref{ass:bg}} \}.
    \end{align*}
\end{definition}
Note that $\gF_{\text{PŁ}}$ is a relaxation of $\gF$ in Definition~\ref{def:fclass}. Compared to $\gF$, we relax $\mu$-strong convexity to $\mu$-PŁ, and we also do not assume convexity of component functions $f_i$.

We now state the following proposition, provided in Theorem 1 of \citet{lu2022grab}, which gives the convergence rate of \cref{alg:grab} for objectives belonging to $\gF_{\text{PŁ}} (L, \mu, 0, \nu)$.
\begin{restatable}[\citet{lu2022grab}, Theorem 1]{proposition}{grabUBa}
    \label{prop:grabUBa}
    Suppose that $F \in \gF_{\text{\emph{PŁ}}} (L, \mu, 0, \nu)$. Under Assumption \ref{ass:hb}, with constant step size $\eta$ as
    \begin{align*}
        \eta = \frac{2}{\mu n K} W_0 \bigopen{\frac{\bigopen{F(\vx_0^1)-F^*+\nu^2/L}\mu^3n^2K^2}{192H^2L^2\nu^2}},
    \end{align*}
    where $W_0$ denotes the Lambert W function,
    \cref{alg:grab} converges at the rate
    \begin{align*}
        F(\vx_n^K) - F^* = \tilde{\gO} \bigopen{\frac{H^2 L^2 \nu^2}{\mu^3 n^2 K^2}}
    \end{align*}
    for $K \gtrsim \kappa$.
\end{restatable}
Proposition \ref{prop:grabUBa} is a slightly different version compared to the original paper \citep{lu2022grab}; the differences are discussed in \cref{subsec:43}.
We emphasize that \cref{thm:grabLBa} provides a lower bound $\Omega \bigopen{\frac{L \nu^2}{\mu^2 n^2 K^2}}$ for arbitrary permutation-based SGD and Proposition \ref{prop:grabUBa} shows that there exists an algorithm that converges at the rate of $\tilde{\gO} \bigopen{\frac{H^2 L^2 \nu^2}{\mu^3 n^2 K^2}}$.
Note that the function class considered in the lower bound is a subset of the function class handled in the upper bound.
Thus, \cref{thm:grabLBa} matches the upper bound up to a factor of $\kappa$, if we ignore the term $H$ and some polylogarithmic terms.
Therefore, we can conclude that \cref{alg:grab} is optimal in terms of the convergence rate with respect to $n$ and $K$.
We defer the discussion of herding constant $H$ to \cref{subsec:43}.

\subsection{Lower Bound without Component Convexity} \label{subsec:42}
\cref{subsec:41} leads us to wonder if it is possible to tighten this $\kappa$ gap. 
Our next theorem drops the assumption of component convexity in the lower bound and shows that we can close the gap and perfectly match the upper bound, if the problem is sufficiently ill-conditioned and the number of epochs is large enough.
\begin{restatable}{theorem}{grabLBb}
    \label{thm:grabLBb}
    For any $n \ge 104$, $L$ and $\mu$ satisfying $\kappa \ge 8n$, and $K \ge \max \left \{\frac{\kappa^2}{n}, \kappa^{3/2}n^{1/2} \right \}$, 
    there exists a $4$-dimensional function $F \in \gF_{\text{\emph{PŁ}}} \bigopen{L, \mu, \frac{L}{\mu}, \nu}$ and an initialization point $\vx_0$ such that for any permutation-based SGD with any constant step size $\eta$, any weighted average iterate $\hat{\vx}$ of the form as in \cref{eq:hatxavg} satisfies
    \begin{align*}
        F(\hat{\vx}) - F^* = \Omega \left( \frac{L^2 \nu^2}{\mu^3 n^2 K^2} \right).
    \end{align*}
\end{restatable}
The proof is in \cref{sec:grabLBb}.
\cref{thm:grabLBb} provides a sharper lower bound than the previous result with respect to $\kappa$.
In our construction, some of the components $f_i$ are nonconvex but the constructed objective $F$ is actually strongly convex; however, for simplicity of exposition, we stated $F$ as a member of a larger class $\gF_{\text{PŁ}}$.
Here we discuss the effect of nonconvex components on the convergence rate.

\paragraph{Nonconvex components.}
Some of our component functions constructed in \cref{thm:grabLBb} are concave in particular directions, and this is the key to obtaining an additional $\kappa$ factor.
\citet{rajput2022permutation} also observe that the presence of nonconvex components can slow down convergence.
They prove that for a $1$-dimensional objective $F(x) = \frac{1}{n} \sum_{i=1}^n \frac{a_i}{2}x^2 - b_ix$, where all $a_i$'s are nonnegative, there exists a permutation that leads to exponential convergence, but also that this no longer holds if $a_i$'s are allowed to be negative.
It is an open problem whether the convergence rate of \cref{alg:grab} could be improved to match the lower bound in \cref{thm:grabLBa} with respect to $\kappa$ if we additionally assume component convexity.

\cref{thm:grabLBb} provides a sharper lower bound compared to \cref{thm:grabLBa} with respect to $\kappa$.
One should be aware, however, that the function classes considered in the upper bound (Proposition \ref{prop:grabUBa}) and the construction in \cref{thm:grabLBb} mismatch.
Therefore, Proposition \ref{prop:grabUBa} does \textit{not} guarantee the $\gO \bigopen{\frac{H^2 L^2 \nu^2}{\mu^3 n^2 K^2}}$ convergence rate for the function constructed in \cref{thm:grabLBb}.
However, we argue that this issue can be addressed by extending Proposition \ref{prop:grabUBa} to a wider function class, which is done in \propref{prop:grabUBb}.

\begin{restatable}[Extended version of \citet{lu2022grab}, Theorem 1]{proposition}{grabUBb}
    \label{prop:grabUBb}
    Suppose that $F \in \gF_{\emph{\text{PŁ}}} (L, \mu, \tau, \nu)$ and $n \ge H$. Under Assumption \ref{ass:hb}, with constant step size $\eta$ as
    \begin{align*}
        \eta = \frac{2}{\mu n K} W_0 \bigopen{\frac{\bigopen{F(\vx_0^1)-F^*+\nu^2/L}\mu^3n^2K^2}{192H^2L^2\nu^2}},
    \end{align*}
    where $W_0$ denotes the Lambert W function, \cref{alg:grab} converges at the rate
    \begin{align*}
        F(\vx_n^K) - F^* = \tilde{\gO} \bigopen{\frac{H^2 L^2 \nu^2}{\mu^3 n^2 K^2}}
    \end{align*}
    for $K \gtrsim \kappa (\tau + 1)$.
\end{restatable}
The proof of Proposition \ref{prop:grabUBb} is in \cref{sec:grabUB}.
We show that the function class considered in Proposition \ref{prop:grabUBa} can be extended to  $\gF_{\text{PŁ}} (L, \mu, \tau, \nu)$ by following the proof step in Theorem 1 in \citet{lu2022grab} with slight modifications.
The function class of the upper bound (Proposition \ref{prop:grabUBb}) now includes the construction of the lower bound (\cref{thm:grabLBb}).
Therefore, when the objective is sufficiently ill-conditioned and a sufficiently many epochs are performed, our lower bound perfectly aligns with the upper bound in all factors, assuring that \textsf{GraB} is indeed the optimal permutation-based SGD algorithm.

\subsection{Discussion of Existing Results} \label{subsec:43}
In this section, we take a deeper look at previous researches that address permutation-based SGD.
We mainly discuss the dimension dependency hidden in the upper bounds.

\paragraph{Herding Bound.}
\citet{bansal2017algorithmic} prove that there exists an efficient Herding algorithm that achieves Assumption \ref{ass:hb} with $H = \gO \bigopen{\sqrt{d \log n}}$.
Also, it is well known that $H = \Omega (\sqrt{d})$ \citep{behrend1954steinitz, barany2008power}. 
Thus, both Proposition \ref{prop:grabUBa} and Proposition \ref{prop:grabUBb} contain a dimension term in their convergence rates.
Meanwhile, our lower bound results are based on fixed dimensional functions, so we can ignore the term $H$ when we compare our lower bound results to the upper bound results.
We also note that the assumption $n \ge H$ made in Proposition \ref{prop:grabUBb} is quite mild if the dimension of $F$ is independent of $n$.

\paragraph{Comparison between Proposition \ref{prop:grabUBa} and \citet{lu2022grab}, Theorem 1.}
In the original statement of Theorem 1 in \citet{lu2022grab}, the authors use slightly different assumptions. 
Instead of smoothness with respect to the $L_2$ norm, they assume $L_{2, \infty}$-smoothness as follows:
\begin{align*}
    \bignorm{\nabla f_i (\vx) - \nabla f_i (\vy)}_2 \le L_{2, \infty} \bignorm{\vx - \vy}_{\infty}, \ \forall \vx, \vy \in \R^d.
\end{align*}
\citet{lu2022grab} also define the herding bound $H$ with respect to different choices of norms. Specifically, the authors consider $\max_{k \in \{1, \dots, n\}} \bignorm{\sum_{i=1}^k \vz_{\sigma(i)}}_{\infty} \le H_{\infty}$ with $\max_{i} \bignorm{\vz_i}_2 \le 1$, and explain that combining the results from \citet{harvey2014near} and \citet{alweiss2021discrepancy} gives $H_{\infty} = \tilde{\gO} \bigopen{1}$. With these assumptions, the authors obtain the convergence rate of \cref{alg:grab} as the following:
\begin{align}
    F(\vx_n^K) - F^* = \tilde{\gO} \bigopen{\frac{H_{\infty}^2 {L^2_{2, \infty}} \nu^2}{\mu^3 n^2 K^2}}. \label{eq:originalgrab}
\end{align}
However, we believe that \cref{eq:originalgrab} is not also free from dimension dependency, since the term $L_{2, \infty}$ is likely to contain the dimension dependency in general (e.g., $L_{2, \infty} = \sqrt{d}L$ holds when $F(\vx) = \frac{L}{2}\bignorm{\vx}^2$ for $\forall \vx \in \R^d$).
It is an open problem whether there exists a permutation-based SGD algorithm that gives a dimension-free upper bound while maintaining the same dependency on other factors.

\paragraph{Revisiting \citet{rajput2022permutation}.} We have discussed that the best-known upper bound of permutation-based SGD has dimension dependency.
Earlier, we mentioned that our lower bound in \cref{thm:grabLBa} improves upon previous results from Theorem~2 of \citet{rajput2022permutation} by a factor of $n$.
In fact, the construction of \citet{rajput2022permutation} is based on a $(2n+1)$-dimensional function, and applying the upper bounds for \cref{alg:grab} to this function results in a convergence rate of $\tilde{\gO} \bigopen{\frac{1}{n K^2}}$, due to the dimension dependency.
More precisely, for the function constructed in \citet{rajput2022permutation}, $H$ is proportional to $\sqrt{n}$ and $L$ is constant according to our $L_2$-norm-based notations, while we also have that $H_{\infty}$ is constant and $L_{2, \infty}$ is proportional to $\sqrt{n}$ following the notations in \citet{lu2022grab}. (Here we ignore log factors.)
Thus, in terms of $n$ dependency, we conclude that the actual gap between existing upper and lower bounds is $n^2$ rather than $n$, and that our results succeed in closing the gap completely.
\section{Conclusion} \label{sec:5}

We have shown convergence lower bounds for without-replacement SGD methods, focusing on matching the upper and lower bound in terms of the condition number $\kappa$. Our lower bounds for \sgdrr{} on weighted average iterates tightly match the corresponding upper bounds under both strong convexity and convexity assumptions. We also constructed lower bounds for permutation-based SGD with \textit{and} without individual convexity assumptions, which tightly match the upper bounds for \textsf{GraB} in fixed-dimension settings, therefore implying that \textsf{GraB} achieves the optimal rate of convergence.

An immediate direction for future work is to investigate whether one can find lower bounds for arbitrary weighted average iterates of \sgdrr{} when $\eta = \Omega \left( \frac{1}{Ln} \right)$. In the discussion following \cref{thm:xavg} (\cref{subsec:32}), we outlined difficulties that arise in proving such a result for larger learning rates $\eta = \Omega \left( \frac{1}{Ln} \right)$.

We finally note that the power of general permutation-based SGD is not yet well-understood for the regime when the number of epochs is less than the condition number.
\citet{safran_shamir} show that \sgdrr{} does not enjoy faster convergence than \textit{with-replacement} SGD in this regime, and it is still unclear whether the same restriction holds for permutation-based SGD as well.

\section*{Acknowledgements}
This paper was supported by Institute of Information \& communications Technology Planning \& Evaluation (IITP) grant (No.2019-0-00075, Artificial Intelligence Graduate School Program (KAIST)) funded by the Korea government (MSIT), two National Research Foundation of Korea (NRF) grants (No. NRF-2019R1A5A1028324, RS-2023-00211352) funded by the Korea government (MSIT), and a grant funded by Samsung Electronics Co., Ltd.


\bibliography{refs}
\bibliographystyle{icml2023}

\newpage
\appendix
\onecolumn

\allowdisplaybreaks
\section{\texorpdfstring{Comparison with Previous Results}{Comparison with Previous Results}} \label{app:table}

\cref{tab:summary2} shows a detailed comparison of existing convergence rates and our results for permutation-based SGD. Note that the function class categories are divided with respect to the lower bound results--- the selected upper bounds are the results with the best convergence rates among those of which the function class contains the constructed lower bounds. The upper bound results are colored white and the lower bound results are colored gray. 

Similarly as in \cref{tab:summary}, the parameters $L$, $\mu$, $\nu$, and $D$ are defined in \cref{sec:2}. Algorithm outputs $\hat{\vx}$, $\hat{\vx}_{\text{tail}}$, and $\hat{\vx}_{\text{avg}}$ are defined in \cref{sec:3}. Function classes $\gF$ and $\gF_{\text{PŁ}}$ are defined in Sections \ref{sec:2} and \ref{sec:4}, respectively. The herding bound $H$, which closely relates to the convergence rate of \cref{alg:grab}, is defined in \cref{sec:4}.

\begin{table*}[hb] \footnotesize
    \centering
    \vspace{-8pt}
    \caption{A detailed comparison of existing convergence rates and our results for permutation-based SGD.}
    \label{tab:summary2}
    \vspace{10pt}
    \begin{threeparttable}[t]
    \centering
    \begin{tabular}{|c|c|lcr|} 
        \hline
        \multicolumn{5}{|c|}{Random Reshuffling} \\
        \hline
        Function Class & Output & References & Convergence Rate & Assumptions \\
        \hline \hline
        \multirow{9.25}{*}{$\gF (L, \mu, 0, \nu)$} & \multirow{4.5}{*}{$\vx_n^K$} & \citet{mish20} & $\tilde{\gO} \left( \frac{L^2 \nu^2}{\mu^3 n K^2} \right)$ & $K \gtrsim \kappa$ \\
        & & \cellcolor{gray!20}\citet{yun22} & \cellcolor{gray!20}$\Omega \left( \frac{\nu^2}{\mu n K^2} \right)$ & \cellcolor{gray!20}$\kappa \ge c$, $K \gtrsim \kappa$ \\
        & & \cellcolor{gray!20}Ours, \cref{thm:xhat} & \cellcolor{gray!20}$\Omega \left( \frac{L \nu^2}{\mu^2 n K^2} \right)$ & \cellcolor{gray!20}$\kappa \ge c$, $K \gtrsim \kappa$ \\
        \cline{2-5}
        & \multirow{2.75}{*}{$\hat{\vx}_{\text{tail}}$} & \citet{pmlr-v97-nagaraj19a}\tnote{\dag} & $\tilde{\gO} \left( \frac{L^2 \nu^2}{\mu^3 n K^2} \right)$ & $K \gtrsim \kappa^2$ \\
        \cline{3-5}
        & & Ours, \cref{thm:tailub} & $\tilde{\gO} \left( \frac{L \nu^2}{\mu^2 n K^2} \right)$ & $K \gtrsim \kappa$ \\
        \cline{2-5}
        & $\hat{\vx}$ & \cellcolor{gray!20}Ours, \cref{thm:xavg}\tnote{\ddag} & \cellcolor{gray!20}$\Omega \left( \frac{L \nu^2}{\mu^2 n K^2} \right)$ & \cellcolor{gray!20}$\kappa \ge c$, $K \gtrsim \kappa$ \\
        \hline
        \multirow{2.75}{*}{$\gF (L, 0, 0, \nu)$} & $\hat{\vx}_{\text{avg}}$ & \citet{mish20} & $\gO \left( \frac{L^{1/3} \nu^{2/3} D^{4/3}}{n^{1/3} K^{2/3}} \right)$ & $K \gtrsim \frac{L^2 D^2 n}{\nu^2}$ \\
        \cline{2-5}
        & $\hat{\vx}$ & \cellcolor{gray!20}Ours, \cref{cor:cvxcase}\tnote{\ddag} & \cellcolor{gray!20}$\Omega \left( \frac{L^{1/3} \nu^{2/3} D^{4/3}}{n^{1/3} K^{2/3}} \right)$ & \cellcolor{gray!20}$K \gtrsim \max \{ \frac{L^2 D^2 n}{\nu^2}, \frac{\nu}{\mu D n^{1/2}} \}$ \\
        \hline
        \multicolumn{5}{c}{\vspace{-3pt}} \\
        \hline
        \multicolumn{5}{|c|}{Arbitrary Permutations} \\
        \hline
        Function Class & Output & References & Convergence Rate & Assumptions \\
        \hline \hline
        \multirow{2.75}{*}{$\gF (L, \mu, 0, \nu)$} & $\vx_n^K$ & \citet{lu2022grab} (\textsf{GraB}) & $\tilde{\gO} \left( \frac{H^2 L^2 \nu^2}{\mu^3 n^2 K^2} \right)$ & $K \gtrsim \kappa$ \\
        \cline{2-5}
        & $\hat{\vx}$ & \cellcolor{gray!20}Ours, \cref{thm:grabLBa} & \cellcolor{gray!20}$\Omega \left( \frac{L \nu^2}{\mu^2 n^2 K^2} \right)$ & \cellcolor{gray!20}- \\
        \hline
        \multirow{4.5}{*}{$\gF_{\text{PŁ}} (L, \mu, \tau, \nu)$} & $\vx_n^K$ & Ours, Proposition \ref{prop:grabUBb} (\textsf{GraB}) & $\tilde{\gO} \left( \frac{H^2 L^2 \nu^2}{\mu^3 n^2 K^2} \right)$ & $n \ge H$, $K \gtrsim \kappa (\tau + 1)$ \\
        \cline{2-5}
        & $\vx_n^K$ & \cellcolor{gray!20}\citet{rajput2022permutation}$^*$ & \cellcolor{gray!20}$\Omega \bigopen{\frac{\nu^2}{Ln^3K^2}}$ & \cellcolor{gray!20}$d=2n+1$ \\
        \cline{2-5}
        & $\hat{\vx}$ & \cellcolor{gray!20}Ours, \cref{thm:grabLBb} & \cellcolor{gray!20}$\Omega \left( \frac{L^2 \nu^2}{\mu^3 n^2 K^2} \right)$ & \cellcolor{gray!20}$\tau = \kappa \ge 8n$, $K \ge \max \{ \frac{\kappa^2}{n}, \kappa^{\frac{3}{2}}n^{\frac{1}{2}} \}$ \\
        \hline
    \end{tabular}
    \vspace{4pt}
    \begin{tablenotes}
        \item[\dag] \small{Assumes a stronger condition of $\| \nabla f_i (\vx) \| \le \nu$ for all $i$ and $\vx$}
        \item[\ddag] \small{Additionally assumes $\eta \le \frac{1}{c_2 Ln}$}
        \item[$*$] \small{The constructed objective is a member of $\gF \bigopen{2L, \frac{n-1}{n}L, 1, \nu}$.}
    \end{tablenotes}
    \vspace{0pt}
    \end{threeparttable}
    \vspace{-6pt}
\end{table*}


\section{\texorpdfstring{Proof of \cref{thm:xhat}}{Proof of Theorem 1}} \label{app:xhat}

Here we prove \cref{thm:xhat}, restated below for the sake of readability.

\xhat* 

\begin{proof}
    We prove the theorem statement for constants $c_1 = 2415$ and $c_2 = 161$.

    As the convergence behavior of SGD heavily depends on the step size $\eta$, we consider three step-size regimes and use different objective functions with slow convergence rates in each case. Then we aggregate the three functions to obtain the final lower bound, which will be the minimum among the lower bounds from each regime. Throughout the proof, we will assume $n$ is even. If $n$ is odd, then we can use a similar technique with Theorem~1 of \citet{safran_shamir}, which is to set $n-1$ nontrivial components satisfying the statement, add a single zero component function, and scale by $\frac{n-1}{n}$.

    To elaborate, we prove the following lower bounds for each of the following three regimes. Here we denote by $F_j^*$ the minimizer of $F_j$ for each $j = 1, 2, 3$. Note that the union of the three ranges completely covers the set of all positive learning rates, $\eta > 0$. 
    \begin{itemize}
        \item If $\eta \in \left( 0, \frac{1}{\mu n K} \right)$, there exists a 1-dimensional objective function $F_1 (x) \in \gF (L, \mu, 0, \nu)$ such that \textsf{SGD-RR} with initialization $x_0^1 = D_0$ (for any $D_0$) satisfies
        \begin{align*}
            \mathbb{E} \left[ F_1 (x_{n}^{K}) - F_1^* \right] &= \Omega \left( \mu D_0^2 \right).
        \end{align*}
        \item If $\eta \in \left[ \frac{1}{\mu n K}, \frac{1}{161 L n} \right]$, there exists a 1-dimensional objective function $F_2 (x) \in \gF (L, \mu, 0, \nu)$ such that \textsf{SGD-RR} with initialization $x_0^1 = \note{\frac{1}{27000} \cdot \frac{\nu}{\mu n^{1/2} K}}$ satisfies
        \begin{align*}
            \mathbb{E} \left[ F_2 (x_{n}^{K}) - F_2^* \right] = \Omega \left( \frac{L \nu^2}{\mu^2 n K^2} \right).
        \end{align*}
        \item If $\eta \ge \max \left\{ \frac{1}{\mu n K}, \frac{1}{161 L n} \right\}$, there exists a 1-dimensional objective function $F_3 (x) \in \gF (L, \mu, 0, \nu)$ such that \textsf{SGD-RR} with initialization $x_0^1 = 0$ satisfies
        \begin{align*}
            \mathbb{E} \left[ F_3 (x_{n}^{K}) - F_3^* \right] &= \Omega \left( \frac{\nu^2}{\mu n K} \right).
        \end{align*}
    \end{itemize}

    Now we define the $3$-dimensional function $F(x, y, z) = F_1(x) + F_2(y) + F_3(z)$, where $F_1$, $F_2$, and $F_3$ are chosen to satisfy the above lower bounds for $\nu$ replaced by $\frac{\nu}{\sqrt{3}}$. Note that scaling $\nu$ does not change the convergence rates above. We denote the components by $f_i(x, y, z) = f_{1, i} (x) + f_{2, i} (y) + f_{3, i} (z)$ for $i = 1, \dots, n$.
    
    If $H_1$, $H_2$, and $H_3$ are $L$-smooth and $\mu$-strongly convex, then $H(x, y, z) = H_1(x) + H_2(y) + H_3(z)$ satisfies
    \begin{align*}
        \mu \mI \preceq \min \{ \nabla^2 H_1(x), \nabla^2 H_2(y), \nabla^2 H_3(z) \} \preceq \nabla^2 H(x, y, z) \preceq \max \{ \nabla^2 H_1(x), \nabla^2 H_2(y), \nabla^2 H_3(z) \} \preceq L \mI,
    \end{align*}
    i.e., $H(\vx)$ must be $L$-smooth and $\mu$-strongly convex.
    
    Also, if $H_1$, $H_2$, and $H_3$ (each with $n$ components $h_{1, i}$, $h_{2, i}$, and $h_{3, i}$) have bounded gradients (Assumption~\ref{ass:bg}) for $\tau = 0$ and $\nu = \frac{\nu_0}{\sqrt{3}}$, then $H(x, y, z) = H_1(x) + H_2(y) + H_3(z)$ satisfies
    \begin{align*}
        & \| \nabla h_i (x, y, z) - \nabla H(x, y, z) \|^2 \\
        = \ & \| \nabla h_{1, i} (x) - \nabla H_1(x) \|^2 + \| \nabla h_{2, i} (y) - \nabla H_2(y) \|^2 + \| \nabla h_{3, i} (z) - \nabla H_3(z) \|^2 \\
        \le \ & \frac{\nu_0^2}{3} + \frac{\nu_0^2}{3} + \frac{\nu_0^2}{3} = \nu_0^2
    \end{align*}
    for all $i = 1, \dots, n$, i.e., $H(x, y, z)$ satisfies Assumption~\ref{ass:bg} for $\tau = 0$ and $\nu = \nu_0$.
    
    Since $F_1, F_2, F_3 \in \gF (L, \mu, 0, \frac{\nu}{\sqrt{3}})$ by construction, we have $F \in \gF (L, \mu, 0, \nu)$ from the above arguments.

    Now suppose that we set $D_0 = \frac{\nu}{\mu}$ and initialize at the point $\left( \frac{\nu}{\mu}, \note{\frac{1}{27000} \cdot \frac{\nu}{\mu n^{1/2} K}}, 0 \right)$.

    If $K \ge 161 \kappa$, then since $\frac{1}{\mu n K} \le \frac{1}{161 L n}$ we can use the lower bound for $F_2(y)$. The lower bound for this case becomes
    \begin{align*}
        \E \left[ F (x_n^K, y_n^K, z_n^K) - F^* \right] &= \Omega \left( \min \left\{\frac{\nu^2}{\mu}, \frac{L \nu^2}{\mu^2 n K^2}, \frac{\nu^2}{\mu n K} \right\} \right) = \Omega \left( \frac{L \nu^2}{\mu^2 n K^2} \right).
    \end{align*}
    If $K < 161 \kappa$, then since $\frac{1}{\mu n K} > \frac{1}{161 L n}$ the middle step-size regime does not exist, i.e., we \textit{cannot} use the lower bound for $F_2(y)$. Hence the lower bound for this case becomes 
    \begin{align*}
        \E \left[ F (x_n^K, y_n^K, z_n^K) - F^* \right] &= \Omega \left( \min \left\{\frac{\nu^2}{\mu}, \frac{\nu^2}{\mu n K} \right\} \right) = \Omega \left( \frac{\nu^2}{\mu n K} \right),
    \end{align*}
    which completes the proof.
\end{proof}

For the following subsections, we prove the lower bounds for $F_1$, $F_2$, and $F_3$ at the corresponding step size regimes. The proofs are similar to those of \citet{yun22}, corresponding to the case $M = B = 1$ with slight modifications.

\subsection{\texorpdfstring{Lower Bound for $\eta \in \left( 0, \frac{1}{\mu n K} \right)$}{Lower Bound for η ∈ (0, 1/μnK)}} \label{subsec:a1}

Here we show that there exists $F_1 (x) \in \gF (L, \mu, 0, \nu)$ such that \textsf{SGD-RR} with $x_0^1 = D_0$ satisfies
\begin{align*}
    \mathbb{E} \left[ F_1 (x_{n}^{K}) - F_1^* \right] &= \Omega \left( \mu D_0^2 \right).
\end{align*}

\begin{proof}
    We define $F_1 (x) \in \gF (\mu, \mu, 0, 0)$ by the following components: 
    \begin{align*}
        f_{i} (x) &= F_1(x) = \frac{\mu}{2} x^2.
    \end{align*}
    Note that $\gF (\mu, \mu, 0, 0) \subseteq \gF (L, \mu, 0, \nu)$ and $F_1^* = 0$ at $x^* = 0$ by definition. Also, note that the components have no stochasticity, and hence we can drop the expectation notation, $\E[\cdot]$. Then we can easily compute per-epoch updates as:
    \begin{align*}
        x_{0}^{k+1} &= (1 - \eta \mu)^{n} x_{0}^{k}, \quad \forall k = 1, \dots, K.
    \end{align*}
    Since $x_{0}^{1}= x_{0} = D_0$ and $\eta \le \frac{1}{\mu n K}$, for any $k = 1, \dots, K$ we have
    \begin{align}
        x_{0}^{k+1} &= (1 - \eta \mu)^{nk} \cdot D_0 \ge \left( 1 - \frac{1}{nK} \right)^{nK} \cdot D_0 \ge \frac{D_0}{4}, \label{eq:smalletaineq}
    \end{align}
    where in the last inequality we use $\left( 1 - \frac{1}{m} \right)^m \ge \frac{1}{4}$ for all $m \ge 2$.
    Hence, for the final iterate we have $x_n^K \ge \frac{D_0}{4}$ and therefore
    \begin{align*}
        F_1 (x_n^K) = \frac{\mu}{2} (x_n^K)^2 &\ge \frac{\mu}{2} \left( \frac{D_0}{4} \right)^2 = \frac{\mu D^2}{32},
    \end{align*}
    which concludes that $\mathbb{E} \left[ F_1 (x_n^K) - F_1^* \right] = \mathbb{E} \left[ F_1 (x_n^K) \right] = F_1 (x_n^K) = \Omega \left( \mu D_0^2 \right)$.
\end{proof}

\subsection{\texorpdfstring{Lower Bound for $\eta \in \left[ \frac{1}{\mu n K}, \frac{1}{161 L n} \right]$}{Lower Bound for η ∈ [1/μnK, 1/161Ln]}} \label{subsec:a2}

Here we show that there exists $F_2 (x) \in \gF (L, \mu, 0, \nu)$ such that \textsf{SGD-RR} with $x_0^1 = \note{\frac{1}{27000} \cdot \frac{\nu}{\mu n^{1/2} K}}$ satisfies
\begin{align*}
    \mathbb{E} \left[ F_2 (x_{n}^{K}) - F_2^* \right] = \Omega \left( \frac{L \nu^2}{\mu^2 n K^2} \right).
\end{align*}
    
\begin{proof}
    We define $F_2 (x) \in \gF (L, \mu, 0, \nu)$ by the following components:
    \begin{align*}
        f_{i} (x) &= \begin{cases}
            \left( L \mathbbm{1}_{x < 0} + \mu_0 \mathbbm{1}_{x \ge 0} \right) \frac{x^2}{2} + \nu x & \text{if} \ i \le n/2, \\
            \left( L \mathbbm{1}_{x < 0} + \mu_0 \mathbbm{1}_{x \ge 0} \right) \frac{x^2}{2} - \nu x & \text{otherwise,}
        \end{cases}
    \end{align*}
    where we assume $\mu_0 \le \frac{L}{2415}$ and later choose $\mu_0 = \frac{L}{2415}$. With this construction, the finite-sum objective becomes
    \begin{align*}
        F_2 (x) &= \frac{1}{n} \sum_{i=1}^{n} f_{i}(x) = \left( L \mathbbm{1}_{x < 0} + \mu_0 \mathbbm{1}_{x \ge 0} \right) \frac{x^2}{2}.
    \end{align*}
    Note that $F_2^* = 0$ at $x^* = 0$ by definition, and that $\mu_0$ is different from $\mu$. While $F_2 (x) \in \gF(L, \mu_0, 0, \nu)$ by construction, we can ensure that $\gF(L, \mu_0, 0, \nu) \subset \gF(L, \mu, 0, \nu)$ because the assumption $\kappa \ge 2415$ implies $\mu_0 = \frac{L}{2415} \ge \mu$.

    First, we focus on a single epoch, and hence we write $x_{i}$ instead of $x_{i}^{k}$, omitting the superscripts $k$ for a while.

    We use the following definition throughout the paper for notational simplicity.
    \begin{definition} \label{def:perms}
        Define $\gS_n$ as the set of all possible permutations of $\frac{n}{2}$ $+1$'s and $\frac{n}{2}$ $-1$'s, where $n$ is a positive, even integer. If \sgdrr \ samples a permutation $\sigma$ in a certain epoch, we define the corresponding $s \in \gS_n$ to satisfy $s_i = +1$ if $\sigma(i) \le \frac{n}{2}$ and $s_i = -1$ if $\sigma(i) > \frac{n}{2}$.
    \end{definition}
    Note that following \cref{def:perms}, we can express the iterations for $i = 1, \dots, n$ via $s$ as
    \begin{align*}
        x_{i} &= x_{i-1} - \eta \nabla f_{\sigma(i)} (x) = x_{i-1} - \eta (L \mathbbm{1}_{x_{i-1} < 0} + \mu_0 \mathbbm{1}_{x_{i-1} \ge 0}) x_{i-1} - \eta \nu s_{i}.
    \end{align*}
    Also, we can sum up the iterates to obtain
    \begin{align*}
        x_n &= x_{n-1} - \eta (L \mathbbm{1}_{x_{n-1} < 0} + \mu_0 \mathbbm{1}_{x_{n-1} \ge 0}) x_{n-1} - \eta \nu s_{n} \\
        &= x_{n-2} - \eta (L \mathbbm{1}_{x_{n-2} < 0} + \mu_0 \mathbbm{1}_{x_{n-2} \ge 0}) x_{n-2} - \eta \nu s_{n-1} - \eta (L \mathbbm{1}_{x_{n-1} < 0} + \mu_0 \mathbbm{1}_{x_{n-1} \ge 0}) x_{n-1} - \eta \nu s_{n} \\
        &\ \ \vdots \\
        &= x_{0}- \eta \sum_{i=0}^{n-1} (L \mathbbm{1}_{x_i < 0} + \mu_0 \mathbbm{1}_{x_i \ge 0}) x_i - \eta \nu \sum_{i=1}^{n} s_i \\
        &= x_{0}- \eta \sum_{i=0}^{n-1} (L \mathbbm{1}_{x_i < 0} + \mu_0 \mathbbm{1}_{x_i \ge 0}) x_i.
    \end{align*}
    Now we use the following three lemmas.
    \begin{restatable}{lemma}{lemone}
        \label{lem:yun1}
        For (fixed) $x_0 \ge 0$, $0 \le i \le \left\lfloor \frac{n}{2} \right\rfloor$, $\eta \le \frac{1}{161 L n}$, and $\frac{L}{\mu_0} \ge 2415$,
        \begin{align*}
            \E \left[ (L \mathbbm{1}_{x_i < 0} + \mu_0 \mathbbm{1}_{x_i \ge 0}) x_i \right] &\le \frac{2}{3} Lx_0 - \frac{\eta L \nu}{480} \sqrt{i}.
        \end{align*}
    \end{restatable}
    \begin{restatable}{lemma}{lemtwo}
        \label{lem:yun2}
        For (fixed) $x_0 \ge 0$, $0 \le i \le n-1$, and $\eta \le \frac{1}{161 L n}$,
        \begin{align*}
            \E \left[ (L \mathbbm{1}_{x_i < 0} + \mu_0 \mathbbm{1}_{x_i \ge 0}) x_i \right] &\le \left( 1 + \frac{161}{160} i \eta L \right) \mu_0 x_0 + \frac{161}{160} \eta \mu_0 \nu \sqrt{i}.
        \end{align*}
    \end{restatable}
    \begin{restatable}{lemma}{lemthree}
        \label{lem:yun3}
        If $\eta \le \frac{1}{161 L n}$, we have the followings.
        \begin{enumerate}
            \item For (fixed) $x_0 < 0$, we have
            \begin{align*}
                \mathbb{E} \left[ x_n \right] &\ge \left( 1 - \frac{160}{161} \eta L n \right) x_0.
            \end{align*}
            \item If we initialize at $x_{0}^{1} \ge 0$, then we always have $\mathbb{P} (x_{n}^{k} \ge 0) \ge \frac{1}{2}$ for future start-of-epoch iterates.
        \end{enumerate}
    \end{restatable}
    See \cref{subsec:a4} for the proofs of \cref{lem:yun1,lem:yun2,lem:yun3}.
    
    If an epoch starts at (a fixed value) $x_0 \ge 0$, then from Lemmas \ref{lem:yun1} and \ref{lem:yun2} we have
    \begin{align*}
        \E \left[ x_n - x_0 \right] &= - \eta \sum_{i=0}^{n-1} \mathbb{E} \left[ (L \mathbbm{1}_{x_i < 0} + \mu_0 \mathbbm{1}_{x_i \ge 0}) x_i \right] \\
        &= - \eta \sum_{i = 0}^{\lfloor \frac{n}{2} \rfloor} \mathbb{E} \left[ (L \mathbbm{1}_{x_i < 0} + \mu_0 \mathbbm{1}_{x_i \ge 0}) x_i \right] - \eta \sum_{i = \lfloor \frac{n}{2} \rfloor + 1}^{n-1} \mathbb{E} \left[ (L \mathbbm{1}_{x_i < 0} + \mu_0 \mathbbm{1}_{x_i \ge 0}) x_i \right] \\
        &\ge - \eta \sum_{i = 0}^{\lfloor \frac{n}{2} \rfloor} \left( \frac{2}{3} Lx_0 - \frac{\eta L \nu}{480} \sqrt{i} \right) - \eta \sum_{i = \lfloor \frac{n}{2} \rfloor + 1}^{n-1} \left( \left( 1 + \frac{161}{160} i \eta L \right) \mu_0 x_0 + \frac{161}{160} \eta \mu_0 \nu \sqrt{i} \right) \\
        &= - \eta \left( \sum_{i = 0}^{\lfloor \frac{n}{2} \rfloor} \frac{2}{3} L + \sum_{i = \lfloor \frac{n}{2} \rfloor + 1}^{n-1} \left( 1 + \frac{161}{160} i \eta L \right) \mu_0 \right) x_0 - \eta \left( - \sum_{i = 0}^{\lfloor \frac{n}{2} \rfloor} \frac{\eta L \nu}{480} \sqrt{i} + \sum_{i = \lfloor \frac{n}{2} \rfloor + 1}^{n-1} \frac{161}{160} \eta \mu_0 \nu \sqrt{i} \right).
    \end{align*}
    Now we can bound the coefficient of the $x_0$ term by the following inequality:
    \begin{align*}
        \sum_{i = 0}^{\lfloor \frac{n}{2} \rfloor} \frac{2}{3} L + \sum_{i = \lfloor \frac{n}{2} \rfloor + 1}^{n-1} \left( 1 + \frac{161}{160} i \eta L \right) \mu_0 & \le \left( \left\lfloor \frac{n}{2} \right\rfloor + 1 \right) \frac{2}{3} L + \left( n - \left\lfloor \frac{n}{2} \right\rfloor - 1 \right) \frac{L}{2415} \left( 1 + \frac{1}{160} \right) \\
        & \le \frac{2}{3} Ln + \frac{Ln}{2400} \le \frac{3}{4} L n,
    \end{align*}
    where we use $\mu_0 \le \frac{L}{2415}$ and $i \eta L \le \eta L n \le \frac{1}{161}$. Also, the constant term can be bounded as: 
    \begin{align*}
        - \sum_{i = 0}^{\lfloor \frac{n}{2} \rfloor} \frac{\eta L \nu}{480} \sqrt{i} + \sum_{i = \lfloor \frac{n}{2} \rfloor + 1}^{n-1} \frac{161}{160} \eta \mu_0 \nu \sqrt{i} 
        & \le - \frac{\eta L \nu}{480} \int_{0}^{\lfloor \frac{n}{2} \rfloor} \sqrt{t} dt + \frac{161}{160} \eta \mu_0 \nu \int_{\lfloor \frac{n}{2} \rfloor + 1}^{n} \sqrt{t} dt \\
        & \le - \frac{\eta L \nu}{480} \cdot \frac{2}{3} \left ( \left \lfloor \frac{n}{2} \right \rfloor \right )^{3/2} + \frac{161}{160} \eta \mu_0 \nu \cdot\frac{2}{3} \left( n^{3/2} - \left ( \frac{n}{2} \right )^{3/2} \right) \\
        & \le - \frac{\eta L \nu}{480} \cdot \frac{2}{3} \left ( \frac{n}{3} \right )^{3/2} + \frac{161}{160} \eta \mu_0 \nu \cdot \frac{2}{3} \cdot \frac{2\sqrt{2}-1}{2\sqrt{2}} n^{3/2}\\
        & \le - \frac{\eta L \nu}{480} \cdot \frac{2}{9 \sqrt{3}} n^{3/2} + \frac{161}{160} \eta \mu_0 \nu \cdot \frac{1}{2} n^{3/2} \\
        & \le - \eta L \nu n^{3/2} \left( \frac{2}{480 \cdot 9 \sqrt{3}} - \frac{161}{160 \cdot 2 \cdot 2415} \right) \le - \frac{\eta L \nu n^{3/2}}{18000},
    \end{align*}
    where we use $\mu_0 \le \frac{L}{2415}$, $\lfloor \frac{n}{2} \rfloor \ge \frac{n}{3}$ (for $n \ge 2$), and $\frac{2}{480 \cdot 9 \sqrt{3}} - \frac{161}{160 \cdot 2 \cdot 2415} > \frac{1}{18000}$. Hence we can conclude that
    \begin{align*}
        \E \left[ x_n - x_0 \right] &\ge - \eta \left( \frac{3}{4} L n x_0 - \frac{\eta L \nu n^{3/2}}{18000} \right)
    \end{align*}
    and therefore
    \begin{align*}
        \E \left[ x_n \right] &\ge \left( 1 - \frac{3}{4} \eta L n \right) x_0 + \frac{\eta^2 L \nu n^{3/2}}{18000}.
    \end{align*}
    If an epoch starts at (a fixed value) $x_0 < 0$, then from Lemma \ref{lem:yun3} we have
    \begin{align*}
        \E \left[ x_n \right] &\ge \left( 1 - \frac{160}{161} \eta L n \right) x_0 \ge \left( 1 - \frac{3}{4} \eta L n \right) x_0.
    \end{align*}
    From the second statement of Lemma \ref{lem:yun3}, we can observe that for all epochs we have $\mathbb{P} (x^k_0 \ge 0) \ge \frac{1}{2}$ because we initialize at $x_0^1 \ge \note{\frac{1}{27000} \cdot \frac{\nu}{\mu n^{1/2} K}} \ge 0$. Therefore, taking expectations over $x_0$, we can conclude that each epoch must satisfy
    \begin{align*}
        \E \left[ x_n \right] 
        &= \mathbb{P}(x_0 \geq 0) \E[x_n | x_0 \geq 0] + \mathbb{P}(x_0 < 0) \E[x_n \mid x_0 < 0]\\
        &\ge \mathbb{P}(x_0 \geq 0) \left (\left( 1 - \frac{3}{4} \eta L n \right) \E[x_0 \mid x_0 \geq 0] + \frac{\eta^2 L \nu n^{3/2}}{18000} \right ) + \mathbb{P}(x_0 < 0) \left ( \left( 1 - \frac{3}{4} \eta L n \right) \E[x_0 \mid x_0 < 0] \right)\\
        &\ge \left( 1 - \frac{3}{4} \eta L n \right) \E[x_0] + \frac{\eta^2 L \nu n^{3/2}}{36000}.
    \end{align*}
    Since the above holds for all $\mu_0 \le \frac{L}{2415}$, we may choose $\mu_0 = \frac{L}{2415}$, i.e., our function $F_2$ can be chosen as
    \begin{align*}
        F_2 (x) &= \left( L \mathbbm{1}_{x < 0} + \frac{L}{2415} \mathbbm{1}_{x \ge 0} \right) \frac{x^2}{2}.
    \end{align*}
    Note that since $\kappa \ge 2415$ is equivalent to $\frac{L}{2415} \ge \mu$, we have $F_2(x) \in \gF(L, \frac{L}{2415}, 0, \nu) \subseteq \gF(L, \mu, 0, \nu)$.

    From here we focus on unrolling the per-epoch inequalities for all $k$, and hence we put the superscripts $k$ back in our notation.

    If the starting point of an epoch satisfies $\mathbb{E} [x_{0}^{k}] \ge \note{\frac{1}{27000} \cdot \frac{\nu}{\mu n^{1/2} K}}$, then using $\eta \ge \frac{1}{\mu n K}$ we easily have
    \begin{align}
        \mathbb{E} \left[ x_{0}^{k+1} \right] = \mathbb{E} \left[ x_{n}^{k} \right] &\ge \left( 1 - \frac{3}{4} \eta L n \right) \mathbb{E} [x_{0}^{k}] + \frac{\eta^2 L \nu n^{3/2}}{36000} \notag \\
        &\ge \left( 1 - \frac{3}{4} \eta L n \right) \left( \note{\frac{1}{27000} \cdot \frac{\nu}{\mu n^{1/2} K}} \right) + \left( \frac{1}{\mu n K} \right) \frac{\eta L \nu n^{3/2}}{36000} \notag \\
        &\ge \frac{1}{27000} \cdot \frac{\nu}{\mu n^{1/2} K} - \frac{1}{36000} \cdot \frac{\eta L \nu n^{1/2}}{\mu K} + \frac{1}{36000} \cdot \frac{\eta L \nu n^{1/2}}{\mu K} = \frac{1}{27000} \cdot \frac{\nu}{\mu n^{1/2} K}. \label{eq:midetaineq}
    \end{align}
    Therefore, if we set $x_{0}^1 \ge \note{\frac{1}{27000} \cdot \frac{\nu}{\mu n^{1/2} K}}$, then the final iterate must also maintain $\mathbb{E} [x_{n}^{K}] \ge \note{\frac{1}{27000} \cdot \frac{\nu}{\mu n^{1/2} K}}$. By Jensen's inequality, we can finally conclude that
    \begin{align*}
        \mathbb{E} \left[ F_2 (x_{n}^{K}) - F_2^* \right] &= \mathbb{E} \left[ F_2 (x_{n}^{K}) \right] \\
        &\ge \frac{L}{2 \cdot 2415} \mathbb{E} \left[ (x_{n}^{K})^2 \right] \\
        &\ge \frac{L}{4830} \mathbb{E} \left[ x_{n}^{K} \right]^2 \\
        &\ge \frac{L}{4830} \cdot \left( \frac{1}{27000} \cdot \frac{\nu}{\mu n^{1/2} K} \right)^2 = \Omega \left( \frac{L \nu^2}{\mu^2 n K^2} \right).
    \end{align*}
\end{proof}

\subsection{\texorpdfstring{Lower Bound for $\eta \ge \max \left\{ \frac{1}{\mu n K}, \frac{1}{161 L n} \right\}$}{Lower Bound for η ≥ max\{1/μnK, 1/161Ln\}}} \label{subsec:a3}

Here we show that there exists $F_3 (x) \in \gF (L, \mu, 0, \nu)$ such that \textsf{SGD-RR} with $x_0^1 = 0$ satisfies
\begin{align*}
    \mathbb{E} \left[ F_3 (x_{n}^{K}) - F_3^* \right] &= \Omega \left( \frac{\nu^2}{\mu n K} \right).
\end{align*}

\begin{proof}
    We define $F_3 (x) \in \gF (L, L, 0, \nu)$ by the following components:
    \begin{align*}
        f_{i} (x) &= \begin{cases}
            \frac{Lx^2}{2} + \nu x & \text{if} \ i \le n/2, \\
            \frac{Lx^2}{2} - \nu x & \text{otherwise.}
        \end{cases}
    \end{align*}
    With this construction, the finite-sum objective becomes
    \begin{align*}
        F_3 (x) &= \frac{1}{n} \sum_{i=1}^{n} f_{i}(x) = \frac{L x^2}{2}.
    \end{align*}
    Note that $\gF (L, L, 0, \nu) \subseteq \gF (L, \mu, 0, \nu)$ and $F_3^* = 0$ at $x^* = 0$ by definition.
    
    First, we focus on a single epoch, and hence we write $x_{i}$ instead of $x_{i}^{k}$, omitting the superscripts $k$ for a while.

    Similarly as in \cref{subsec:a2}, we can follow \cref{def:perms} to express the iterations for $i = 1, \dots, n$ via $s$ as
    \begin{align*}
        x_{i} &= x_{i-1} - \eta \nabla f_{\sigma(i)} (x) = x_{i-1} - \eta L x_{i-1} - \eta \nu s_{i} = (1 - \eta L) x_{i-1} - \eta \nu s_i.
    \end{align*}
    Also, we can sum up the iterates to obtain
    \begin{align*}
        x_n &= (1 - \eta L) x_{n-1} - \eta \nu s_{n} \\
        &= (1 - \eta L) \left( (1 - \eta L) x_{n-2} - \eta \nu s_{n-1} \right) - \eta \nu s_{n} \\
        &\ \ \vdots \\
        &= (1 - \eta L)^n x_{0} - \eta \nu \sum_{i=1}^{n} (1 - \eta L)^{n-i} s_i.
    \end{align*}
    Now we can square both terms and take expectations over $s \in \gS_n$ to obtain
    \begin{align*}
        \E [x_n^2] &= \E \left[ \left( (1 - \eta L)^n x_{0} - \eta \nu \sum_{i=1}^{n} (1 - \eta L)^{n-i} s_{i} \right)^2 \right] \\
        &= (1 - \eta L)^{2n} x_0^2 - 2 (1 - \eta L)^n x_{0} \cdot \eta \nu \E \left[ \sum_{i=1}^{n} (1 - \eta L)^{n-i} s_{i} \right] + \eta^2 \nu^2 \E \left[ \left( \sum_{i=1}^{n} (1 - \eta L)^{n-i} s_{i} \right)^2 \right] \\
        &= (1 - \eta L)^{2n} x_0^2 + \eta^2 \nu^2 \E \left[ \left( \sum_{i=1}^{n} (1 - \eta L)^{n-i} s_{i} \right)^2 \right],
    \end{align*}
    where the middle term is eliminated since $\E[s_i] = 0$ for all $i$. By Lemma~1 of \citet{safran2020good}, we can bound
    \begin{align*}
        \E \left[ \left( \sum_{i=1}^{n} (1 - \eta L)^{n-i} s_{i} \right)^2 \right] \ge c \cdot \min \left\{ 1 + \frac{1}{\eta L}, \eta^2 L^2 n^3 \right\}
    \end{align*}
    for some universal constant $c > 0$. Since $\eta \ge \frac{1}{161 L n}$, we can further lower bound the RHS by $\frac{c'}{\eta L}$ for some universal constant $c' > 0$. Then we have
    \begin{align*}
        \E [x_n^2] &= (1 - \eta L)^{2n} x_0^2 + \eta^2 \nu^2 \E \left[ \left( \sum_{i=1}^{n} (1 - \eta L)^{n-i} s_i \right)^2 \right] \ge (1 - \eta L)^{2n} x_0^2 + c' \frac{\eta \nu^2}{L}.
    \end{align*}

    From here we focus on unrolling the per-epoch inequalities for all $k$, and hence we put the $k$'s back in our notations.
    
    Unrolling the inequalities, we obtain
    \begin{align*}
        \E [(x_n^K)^2] &\ge (1 - \eta L)^{2n} \E [(x_n^{K-1})^2] + c' \frac{\eta \nu^2}{L} \\
        &\ \ \vdots \\
        &\ge (1 - \eta L)^{2nK} (x_0^1)^2 + c' \frac{\eta \nu^2}{L} \sum_{k=0}^{K-1} (1 - \eta L)^{2nk} 
        \ge c' \frac{\eta \nu^2}{L},
    \end{align*}
    where we used $x^1_{0} = 0$. Finally, from $\eta \ge \frac{1}{\mu n K}$ we can conclude that
    \begin{align*}
        \E [F_3 (x_n^K) - F_3^*] = \E [F_3 (x_n^K)] = \frac{L}{2} \E [(x_n^K)^2] \ge \frac{c'}{2} \eta \nu^2 \ge \frac{c'}{2} \frac{\nu^2}{\mu n K}.
    \end{align*}
\end{proof}

\subsection{\texorpdfstring{Lemmas used in \cref{thm:xhat}}{Lemmas used in Theorem 1}} \label{subsec:a4}

In this subsection, we will prove the lemmas used in \cref{thm:xhat}.


\lemone* 

\begin{proof}
    For $i=0$, the statement is trivial since $x_0 \ge 0$ and $\frac{L}{\mu_0} \ge 2415$ implies
    \begin{align*}
        \E \left[ (L \mathbbm{1}_{x_0 < 0} + \mu_0 \mathbbm{1}_{x_0 \ge 0}) x_0 \right] &= \mu_0 x_0 \le \frac{1}{2415} Lx_0 \le \frac{2}{3} Lx_0.
    \end{align*}
    Hence we may assume that $1 \le i \le \left\lfloor \frac{n}{2} \right\rfloor$.
    
    Given $s = \{ s_i \}_{i=1}^{n} \in \gS_n$ (as in \cref{def:perms}), let us denote the partial sums as $\gE_i \triangleq \sum_{j=1}^{i} s_j$. We will use conditional expectations under $\gE_i > 0$ and $\gE_i \le 0$, and then aggregate the results to obtain the final inequality. 
    
    First observe that
    \begin{align*}
        \E \left[ \left. (L \mathbbm{1}_{x_i < 0} + \mu_0 \mathbbm{1}_{x_i \ge 0}) x_i \right| \gE_i > 0 \right] &\le L \E \left[ \left. x_i \right| \gE_i > 0 \right], \\
        \E \left[ \left. (L \mathbbm{1}_{x_i < 0} + \mu_0 \mathbbm{1}_{x_i \ge 0}) x_i \right| \gE_i \le 0 \right] &\le \mu_0 \E \left[ \left. x_i \right| \gE_i \le 0 \right],
    \end{align*}
    since $(L \mathbbm{1}_{x <0} + \mu_0 \mathbbm{1}_{x \ge 0}) \le Lx$ and $(L \mathbbm{1}_{x <0} + \mu_0 \mathbbm{1}_{x \ge 0}) \le \mu_0 x$ for all $x \in \R$. By the law of total expectations, we have
    \begin{align}
        \E \left[ (L \mathbbm{1}_{x_i < 0} + \mu_0 \mathbbm{1}_{x_i \ge 0}) x_i \right] &\le L \sP (\gE_i > 0) \E \left[ \left. x_i \right| \gE_i > 0 \right] + \mu_0 \sP (\gE_i \le 0) \E \left[ \left. x_i \right| \gE_i \le 0 \right]. \label{eq:ch}
    \end{align}
    First, we bound $\E \left[ \left. x_i \right| \gE_i > 0 \right]$ for the former term. We can show that
    \begin{align}
        \E \left[ \left. x_i \right| \gE_i > 0 \right] &= \E \left[ \left. x_0 - \eta \cdot \sum_{j=0}^{i-1} \left( \left( L \mathbbm{1}_{x_j < 0} + \mu_0 \mathbbm{1}_{x_j \ge 0} \right) x_j + \nu s_{j+1} \right) \right| \gE_i > 0 \right] \notag \\
        &= \E \left[ \left. x_0 - \eta \cdot \sum_{j=0}^{i-1} \left( L \mathbbm{1}_{x_j < 0} + \mu_0 \mathbbm{1}_{x_j \ge 0} \right) (x_0 + (x_j - x_0)) - \eta \nu \gE_i \right| \gE_i > 0 \right] \notag \\
        &= x_0 \E \left[ \left. 1 - \eta \cdot \sum_{j=0}^{i-1} \left( L \mathbbm{1}_{x_j < 0} + \mu_0 \mathbbm{1}_{x_j \ge 0} \right) \right| \gE_i > 0 \right] \notag \\
        &\phantom{=} - \eta \E \left[ \left. \sum_{j=0}^{i-1} \left( L \mathbbm{1}_{x_j < 0} + \mu_0 \mathbbm{1}_{x_j \ge 0} \right) (x_j - x_0) \right| \gE_i > 0 \right] - \eta \nu \E \left[ \left. \gE_i \right| \gE_i > 0 \right] \notag \\
        &\le x_0 \E \left[ \left. 1 - \eta \cdot \sum_{j=0}^{i-1} \left( L \mathbbm{1}_{x_j < 0} + \mu_0 \mathbbm{1}_{x_j \ge 0} \right) \right| \gE_i > 0 \right] \notag \\
        &\phantom{=} + \eta L \sum_{j=0}^{i-1} \E \left[ \left. | x_j - x_0 | \right| \gE_i > 0 \right] - \eta \nu \E \left[ \left. \gE_i \right| \gE_i > 0 \right]. \label{eq:ca}
    \end{align}
    
    Now we use the following lemmas.
    \begin{restatable}{lemma}{lemfour}
        \label{lem:four}
        If $n \ge 2$ is an even number and $0 \le i \le \frac{n}{2}$, then $\frac{\sqrt{i}}{10} \le \E \left[ \left| \gE_i \right| \right] \le \sqrt{i}.$
    \end{restatable}
    \begin{lemma}[\citet{yun22}, Lemma~14]
        \label{lem:eip}
        For all $0 \le i \le n$, we have $\sP(\gE_i > 0) = \sP(\gE_i < 0) \ge \frac{1}{6}$.
    \end{lemma}
    We will prove \cref{lem:four} later on.
    
    Observe that the probability distribution of each $\gE_i$ is symmetric by the definition of $\gS_n$. Therefore we have
    \begin{align*}
        \E \left[ \left| \gE_i \right| \right] &= P(\gE_i > 0) \E \left[ \left| \gE_i \right| | \gE_i > 0 \right] + P(\gE_i = 0) \E \left[ \left| \gE_i \right| | \gE_i = 0 \right] + P(\gE_i < 0) \E \left[ \left| \gE_i \right| | \gE_i < 0 \right] \\
        &= P(\gE_i > 0) \E \left[ \gE_i | \gE_i > 0 \right] + P(\gE_i < 0) \E \left[ - \gE_i | \gE_i < 0 \right] \\
        &= 2 P(\gE_i > 0) \E \left[ \gE_i | \gE_i > 0 \right].
    \end{align*}
    Using Lemmas \ref{lem:four} and \ref{lem:eip}, we can obtain
    \begin{align}
        \frac{\sqrt{i}}{20} \le \frac{\E \left[ \left| \gE_i \right| \right]}{2} \le \E \left[ \gE_i | \gE_i > 0 \right] = \frac{\E \left[ \left| \gE_i \right| \right]}{2 P(\gE_i > 0)} \le 3 \E \left[ \left| \gE_i \right| \right] \le 3 \sqrt{i}. \label{eq:cb}
    \end{align}
    
    We also use the following lemma, which we prove later on. This is a simple application of Lemmas \ref{lem:four} and \ref{lem:eip}.
    \begin{restatable}{lemma}{lemfive}
        \label{lem:yun5}
        Suppose that $x_0 \ge 0$, $0 \le i \le n$, and $\eta \le \frac{1}{161 L n}$. Then we have
        \begin{align*}
            \E \left[ | x_i - x_0 | \right] &\le \frac{161}{160} \left( \eta L i x_0 + \eta \nu \sqrt{i} \right).
        \end{align*}
    \end{restatable}
    
    Now we bound the three terms of \cref{eq:ca} one by one. The first term can be bounded simply as
    \begin{align}
        x_0 \E \left[ \left. 1 - \eta \cdot \sum_{j=0}^{i-1} \left( L \mathbbm{1}_{x_j < 0} + \mu_0 \mathbbm{1}_{x_j \ge 0} \right) \right| \gE_i > 0 \right] & \le (1 - \eta \mu_0 i) x_0. \label{eq:cc}
    \end{align}
    For the second term of \cref{eq:ca}, we use \cref{lem:eip} to obtain
    \begin{align*}
        \E \left[ \left. | x_i - x_0 | \right| \gE_i > 0 \right] \le \frac{\E \left[ | x_i - x_0 | \right]}{\sP(\gE_i > 0)} \le 6 \E \left[ | x_i - x_0 | \right],
    \end{align*}
    and then use \cref{lem:yun5} to obtain
    \begin{align}
        \eta L \sum_{j=0}^{i-1} \E \left[ \left. | x_j - x_0 | \right| \gE_i > 0 \right] &\le 6 \eta L \sum_{j=0}^{i-1} \E \left[ | x_j - x_0 | \right] \notag \\
        &\le 6 \eta L \cdot \frac{161}{160} \sum_{j=0}^{i-1} \left( \eta L j x_0 + \eta \nu \sqrt{j} \right) \notag \\
        &= \frac{483}{80} \left( \eta^2 L^2 x_0 \sum_{j=0}^{i-1} j + \eta^2 L \nu \sum_{j=0}^{i-1} \sqrt{j} \right) \notag \\
        &\le \frac{483}{80} \left( \eta^2 L^2 x_0 \cdot \frac{1}{2} i^2 + \eta^2 L \nu \cdot \frac{2}{3} i^{3/2} \right) \notag \\
        &\le \frac{483}{160} \eta^2 L^2 i^2 x_0 + \frac{161}{40} \eta^2 L \nu i^{3/2}. \label{eq:cd}
    \end{align}
    The last term of \cref{eq:ca} can be bounded using \cref{eq:cb} as
    \begin{align}
        - \eta \nu \E \left[ \left. \gE_i \right| \gE_i > 0 \right] \le - \eta \nu \frac{\sqrt{i}}{20}. \label{eq:ce}
    \end{align}
    From Equations (\ref{eq:cc})-(\ref{eq:ce}), we have
    \begin{align}
        \E \left[ \left. x_i \right| \gE_i > 0 \right] &\le (1 - \eta \mu_0 i) x_0 + \frac{483}{160} \eta^2 L^2 i^2 x_0 + \frac{161}{40} \eta^2 L \nu i^{3/2} - \eta \nu \frac{\sqrt{i}}{20} \notag \\
        &= \left( 1 - \eta \mu_0 i + \frac{483}{160} \eta^2 L^2 i^2 \right) x_0 - \left( \frac{1}{20} - \frac{161}{40} \eta L i \right) \eta \nu \sqrt{i} \notag \\
        &\le \left( 1 + \frac{3}{160 \cdot 161} \right) x_0 - \frac{\eta \nu \sqrt{i}}{40} , \label{eq:cf}
    \end{align}
    where the last inequality comes from $\eta L i \le \eta L n \le \frac{1}{161}$.

    Next, we bound $\E \left[ \left. x_i \right| \gE_i \le 0 \right]$ for the former term. We can show that
    \begin{align}
        \E \left[ \left. x_i \right| \gE_i \le 0 \right] &\le x_0 + \E \left[ \left. | x_i - x_0 | \right| \gE_i \le 0 \right] \notag \\
        &\le x_0 + \frac{\E \left[ | x_i - x_0 | \right]}{\sP(\gE_i \le 0)} \notag \\
        &\le x_0 + 6 \E \left[ | x_i - x_0 | \right] &(\because \text{\cref{lem:eip}}) \notag \\
        &\le x_0 + 6 \cdot \frac{161}{160} \left( \eta L i x_0 + \eta \nu \sqrt{i} \right) &(\because \text{\cref{lem:yun5}}) \notag \\
        &= \left( 1 + \frac{483}{80} \eta L i \right) x_0 + \frac{483}{80} \eta \nu \sqrt{i} \notag \\
        &= \left( 1 + \frac{3}{80} \right) x_0 + \frac{483}{80} \eta \nu \sqrt{i}, \label{eq:cg}
    \end{align}
    where the last inequality comes from $\eta L i \le \eta L n \le \frac{1}{161}$.
    
    Plugging in Equations (\ref{eq:cf}) and (\ref{eq:cg}) in (\ref{eq:ch}), we have
    \begin{align*}
        \E \left[ (L \mathbbm{1}_{x_i < 0} + \mu_0 \mathbbm{1}_{x_i \ge 0}) x_i \right] &\le L \sP (\gE_i > 0) \E \left[ \left. x_i \right| \gE_i > 0 \right] + \mu_0 \sP (\gE_i \le 0) \E \left[ \left. x_i \right| \gE_i \le 0 \right] \\
        &\le L \sP (\gE_i > 0) \left( \left( 1 + \frac{3}{160 \cdot 161} \right) x_0 - \frac{\eta \nu \sqrt{i}}{40} \right) \\
        &\phantom{{}\le{}} + \mu_0 \sP (\gE_i \le 0) \left( \left( 1 + \frac{3}{80} \right) x_0 + \frac{483}{80} \eta \nu \sqrt{i} \right) \\
        &= \left( L \sP (\gE_i > 0) \left( 1 + \frac{3}{160 \cdot 161} \right) + \mu_0 \sP (\gE_i \le 0) \left( 1 + \frac{3}{80} \right) \right) x_0 \\
        &\phantom{{}\le{}} - \left( L \sP (\gE_i > 0) \cdot \frac{1}{40} - \mu_0 \sP (\gE_i \le 0) \cdot \frac{483}{80} \right) \eta \nu \sqrt{i}.
    \end{align*}
    Since $\sP(\gE_i > 0) = \frac{1 - \sP(\gE_i = 0)}{2} \le \frac{1}{2}$ by symmetry and $\sP(\gE_i \le 0) = 1 - \sP(\gE_i > 0) \le \frac{5}{6}$ by \cref{lem:eip}, we have
    \begin{align*}
        L \sP (\gE_i > 0) \left( 1 + \frac{3}{160 \cdot 161} \right) + \mu_0 \sP (\gE_i \le 0) \left( 1 + \frac{3}{80} \right) &\le \left( \frac{1}{2} \left( 1 + \frac{3}{160 \cdot 161} \right) + \frac{5}{6} \cdot \frac{1}{2415} \cdot \frac{83}{80} \right) L \le \frac{2}{3}L,
    \end{align*}
    where we use $\eta L i \le \frac{1}{161}$, $\frac{L}{\mu} \ge 2415$, and $\frac{1}{2} \left( 1 + \frac{3}{160 \cdot 161} \right) + \frac{5}{6} \cdot \frac{1}{2415} \cdot \frac{83}{80} \le \frac{2}{3}$. Also, by \cref{lem:eip} we have
    \begin{align*}
        L \sP (\gE_i > 0) \cdot \frac{1}{40} - \mu_0 \sP (\gE_i \le 0) \cdot \frac{483}{80} &\ge \left( \frac{1}{6} \cdot \frac{1}{40} - \frac{5}{6} \cdot \frac{1}{2415} \cdot \frac{483}{80} \right) L = \frac{1}{480}L,
    \end{align*}
    where we use $\eta L i \le \frac{1}{161}$ and $\frac{L}{\mu_0} \ge 2415$. Therefore we have
    \begin{align*}
        \E \left[ (L \mathbbm{1}_{x_i < 0} + \mu_0 \mathbbm{1}_{x_i \ge 0}) x_i \right] &\le \frac{2}{3} Lx_0 - \frac{\eta L \nu}{480} \sqrt{i}.
    \end{align*}
\end{proof}

\lemtwo* 

\begin{proof}
    Since $\eta \le \frac{1}{161 L n}$, we can easily prove using \cref{lem:yun5} as follows.
    \begin{align*}
        \E \left[ (L \mathbbm{1}_{x_i < 0} + \mu_0 \mathbbm{1}_{x_i \ge 0}) x_i \right] &\le \mu \E \left[ \mu_0 x_i \right] \\
        &\le \mu_0 x_0 + \mu_0 \E \left[ | x_i - x_0 | \right] \\
        &\le \mu_0 x_0 + \mu_0 \frac{161}{160} \left( \eta L i x_0 + \eta \nu \sqrt{i} \right) \\
        &\le \left( 1 + \frac{161}{160} i \eta L \right) \mu_0 x_0 + \frac{161}{160} \eta \mu_0 \nu \sqrt{i}.
    \end{align*}
\end{proof}

\lemthree* 

\begin{proof}
    We divide the proof into three parts. In the first part, we compare with the case of using a quadratic function instead, sharing the same permutation. In the second part, we assume $x_0 < 0$ and use the first part to prove the first result of the statement. In the third part, we assume $x_0 \ge 0$ and use the first part to prove the second result of the statement. Note that the statement in the first part holds for both $x_0 \ge 0$ or $x_0 < 0$.

    \paragraph{\textit{Part 1.}} For comparison, we define and use the same function used in \cref{subsec:a3}:
    \begin{align*}
        h_i(x) &= \begin{cases}
            \frac{Lx^2}{2} + \nu x & \text{if} \ i \le n/2, \\
            \frac{Lx^2}{2} - \nu x & \text{otherwise}
        \end{cases}
    \end{align*}
    such that the finite-sum objective becomes
    \begin{align*}
        H (x) &= \frac{1}{n} \sum_{i=1}^{n} h_{i}(x) = \frac{L x^2}{2}.
    \end{align*}
    Now let us think of \textsf{SGD-RR} run on the two functions $F_2(x)$ and $H(x)$, where both of the algorithms start from the same point $x_0$ and both share the same random permutation for all epochs. Let $x_{i, F}$ and $x_{i, H}$ be the output of the $i$-th iterate for \textsf{SGD-RR} on $F_2(x)$ and $H(x)$, respectively. Now we use mathematical induction on $i$ to prove that $x_{i, F} \ge x_{i, H}$.

    \paragraph{Base case.} For $i = 0$, we have $x_{0, F} = x_{0, H} = x_0$. 

    \paragraph{Inductive Case.} Let us assume that the induction hypothesis $x_{i, F} \ge x_{i, H}$ is true, and show that $x_{i+1, F} \ge x_{i+1, H}$ by considering the following three cases. Note that $f_i(x)$'s are the components of $F_2(x)$, $s_i$'s are defined as in \cref{def:perms}, and  $\eta \le \frac{1}{161 L n}$ implies $1 - \eta \mu \ge 1 - \eta L \ge 1 - \frac{1}{161 n} \ge 0$.
    \begin{itemize}
        \item If $x_{i, F} \ge x_{i, H} \ge 0$, then we have
        \begin{align*}
            x_{i+1, F} - x_{i+1, H} &= x_{i, F} - x_{i, H} - \eta \left( \nabla f_{i} (x_{i, F}) - \nabla h_i (x_{i, H}) \right) \\
            &= x_{i, F} - x_{i, H} - \eta \left( \mu x_{i, F} + \nu s_i - L x_{i, H} - \nu s_i \right) \\
            &= (1 - \eta \mu) x_{i, F} - (1 - \eta L) x_{i, H} \ge 0,
        \end{align*}
        since $x_{i, F} \ge x_{i, H} \ge 0$ and $1 - \eta \mu \ge 1 - \eta L \ge 0$.
        \item If $x_{i, F} \ge 0 \ge x_{i, H}$, then we have
        \begin{align*}
            x_{i+1, F} - x_{i+1, H} &= x_{i, F} - x_{i, H} - \eta \left( \nabla f_{i} (x_{i, F}) - \nabla h_i (x_{i, H}) \right) \\
            &= x_{i, F} - x_{i, H} - \eta \left( \mu x_{i, F} + \nu s_i - L x_{i, H} - \nu s_i \right) \\
            &= (1 - \eta \mu) x_{i, F} - (1 - \eta L) x_{i, H} \ge 0,
        \end{align*}
        since $(1 - \eta \mu) x_{i, F} \ge 0$ and $(1 - \eta L) x_{i, H} \le 0$.
        \item If $0 \ge x_{i, F} \ge x_{i, H}$, then we have
        \begin{align*}
            x_{i+1, F} - x_{i+1, H} &= x_{i, F} - x_{i, H} - \eta \left( \nabla f_{i} (x_{i, F}) - \nabla h_i (x_{i, H}) \right) \\
            &= x_{i, F} - x_{i, H} - \eta \left( L x_{i, F} + \nu s_i - L x_{i, H} - \nu s_i \right) \\
            &= (1 - \eta L) x_{i, F} - (1 - \eta L) x_{i, H} \ge 0.
        \end{align*}
    \end{itemize}
    Hence by induction, we have $x_{i+1, F} \ge x_{i+1, H}$ for all $i$.

    From the above, we can observe that $\E [x_{n, F}] \ge \E [x_{n, H}] = (1 - \eta L)^n x_0$.
    
    \paragraph{\textit{Part 2.}} For the next step, let us assume $x_0 < 0$. Let us define
    \begin{align*}
        \varphi(z) = 1 - \frac{160}{161} nz - (1 - z)^n.
    \end{align*}
    Then for $z \in [0, 1 - (\frac{160}{161})^{\frac{1}{n-1}}]$, we have $\varphi'(z) = n ((1 - z)^{n-1} - \frac{160}{161}) \ge 0$ and hence $\varphi(z) \ge \varphi(0) = 0$.

    Also, we can observe that for $n \ge 2$:
    \begin{align*}
        \left( 1 - \frac{1}{161(n-1)} \right)^{n-1} \ge 1 - \frac{1}{161} \ \Rightarrow \ 1 - \left( \frac{160}{161} \right)^{\frac{1}{n-1}} \ge \frac{1}{161 (n-1)} \ge \frac{1}{161 n},
    \end{align*}
    which implies that $\eta L \le \frac{1}{161 n} \le 1 - \left( \frac{160}{161} \right)^{\frac{1}{n-1}}$. Hence we have $\varphi(\eta L) \ge 0$, or
    \begin{align*}
        (1 - \eta L)^n \le 1 - \frac{160}{161} \eta L n,
    \end{align*}
    and for $x_0 < 0$ we have
    \begin{align*}
        \E [x_{n, H}] = (1 - \eta L)^n x_0 \ge \left( 1 - \frac{160}{161} \eta L n \right) x_0.
    \end{align*}
    Applying \textit{\textbf{Part 1}}, we can conclude that $\E [x_{n, F}] \ge \E [x_{n, H}] \ge \left( 1 - \frac{160}{161} \eta L n \right) x_0$.

    \paragraph{\textit{Part 3.}} Now suppose that we initialize $x_0 \ge 0$. For $H(x)$ and some given permutation $s \in \gS_n$, we have
    \begin{align*}
        x_{n, H} &= (1 - \eta L)^n x_0 - \eta \nu \sum_{i=1}^{n} (1 - \eta L)^{n-i} s_i.
    \end{align*}
    Now let us think of pairs of permutations $s, s' \in \gS_n$ which satisfy $s_i = -s_i'$ for all $i$. By definition, the set $\gS_n$ can be exactly partitioned into $\frac{1}{2} \binom{n}{n/2}$ disjoint pairs. Let us temporarily denote the final iterates obtained by choosing the permutations $s$ and $s'$ by $x_{n, H}^{s}$ and $x_{n, H}^{s'}$, respectively. Then we can observe that
    \begin{align*}
        \frac{1}{2} (x_{n, H}^{s} + x_{n, H}^{s'}) &= (1 - \eta L)^n x_0 - \eta \nu \sum_{i=1}^{n} (1 - \eta L)^{n-i} \cdot \left( \frac{s_i + s_i'}{2} \right) = (1 - \eta L)^n x_0,
    \end{align*}
    which means that each pair of outcomes will be symmetric with respect to $(1 - \eta L)^n x_0$. Hence the whole probability distribution of $(1 - \eta L)^{-n} x_{n, H}$ will stay symmetric with respect to the initial point $x_0$.

    Considering outputs after multiple epochs, we can sequentially apply the same logic to prove that the distribution of $(1 - \eta L)^{-nk} x_{n, H}^{k}$ will always stay symmetric with respect to $x_0^1$ for all $k$. In other words, for each $k$, the distribution of outputs $x_{n, H}^k$ \textit{conditioned only on the first epoch $x_0^1$} will be symmetric with respect to $(1 - \eta L)^{nk} x_0 \ge 0$. This automatically implies that we must have $\sP(x_{n, H}^k \ge 0) \ge \sP(x_{n, H}^k \ge (1 - \eta L)^{nk} x_0) \ge \frac{1}{2}$ for any starting point $x_0^1 \ge 0$. Finally, since \textit{\textbf{Part~1}} ensures $x_{n, F}^k \ge x_{n, H}^k$, we can conclude that $\sP(x_{n, F}^k \ge 0) \ge \sP(x_{n, H}^k \ge 0) \ge \frac{1}{2}$.
\end{proof}

\lemfive* 

\begin{proof}
    From $x_{i+1} = x_i - \eta \left( (L \mathbbm{1}_{x_i < 0} + \mu_0 \mathbbm{1}_{x_i \ge 0}) x_i + \nu s_{i+1} \right)$, we have for all $i = 1, \dots, n$:
    \begin{align*}
        \E \left[ | x_i - x_0 | \right] &= \E \left[ \left| - \eta \cdot \sum_{j=0}^{i-1} \left( \left( L \mathbbm{1}_{x_j < 0} + \mu_0 \mathbbm{1}_{x_j \ge 0} \right) x_j + \nu s_{i+1} \right) \right| \right] \\
        &\le \eta \sum_{j=0}^{i-1} \E \left[ \left| \left( L \mathbbm{1}_{x_j < 0} + \mu_0 \mathbbm{1}_{x_j \ge 0} \right) x_j \right| \right] + \eta \nu \E \left[ \left| \sum_{j=1}^{i} s_j \right| \right] \\
        &\le \eta L \sum_{j=0}^{i-1} \E \left[ \left| x_j \right| \right] + \eta \nu \E \left[ \left| \gE_i \right| \right] \\
        &\le \eta L i x_0 + \eta L \sum_{j=0}^{i-1} \E \left[ \left| x_j - x_0\right| \right] + \eta \nu \sqrt{i}. &(\because \text{\cref{lem:four}})
    \end{align*} 
    Now let us think of a sequence $h(i)$ defined by $h(0) = 0$ and recursively as
    \begin{align*}
        h(i) = \eta L i x_0 + \eta L \sum_{j=0}^{i-1} h(j) + \eta \nu \sqrt{i}, \quad \text{for} \ i = 1, \dots, n.  
    \end{align*}
    Then obviously $h(i)$ monotonically increases since $h(i) - h(i-1) = \eta L x_0 + \eta L h(i-1) + \eta \nu (\sqrt{i} - \sqrt{i-1}) > 0$. \\
    We can plug in $h(j) \le h(i)$ for all $j = 0, \dots, i-1$ to obtain $h(i) \le \eta L i x_0 + \eta L i h(i) + \eta \nu \sqrt{i}$, and hence
    \begin{align*}
        h(i) &\le \frac{\eta L i x_0 + \eta \nu \sqrt{i}}{1 - \eta L i}.
    \end{align*}
    Also, by induction, we have $\E \left[ | x_i - x_0 | \right] \le h(i)$, since the sequence $\E \left[ | x_i - x_0 | \right]$ satisfies a recurrence of the same form but with an inequality. Hence, from $\eta L i \le \eta L n \le \frac{1}{161}$ we get
    \begin{align*}
        \E \left[ | x_i - x_0 | \right] &\le \frac{\eta L i x_0 + \eta \nu \sqrt{i}}{1 - \eta L i} \le \frac{1}{1 - \eta L n} \left( \eta L i x_0 + \eta \nu \sqrt{i} \right) \le \frac{161}{160} \left( \eta L i x_0 + \eta \nu \sqrt{i} \right).
    \end{align*}
\end{proof}


\lemfour*
\begin{proof}
    We assume $i \ge 1$ since the statement is vacuously true for $i = 0$.
    
    For the upper bound, we use $\E \left[ \left| \gE_i \right| \right] \le \sqrt{i}$ as in Lemma 12 of \citet{rajput20}.

    For the lower bound, we start from the following equation in Lemma 12 of \citet{rajput20}:
    \begin{align*}
        \E \left[ \left| \gE_{i+1} \right| \right] &= \left( 1 - \frac{1}{n - i} \right) \E \left[ \left| \gE_i \right| \right] + \sP(\gE_i = 0).
    \end{align*}
    We can explicitly compute for $i = 1, \dots, \frac{n}{2}$:
    \begin{align*}
        \sP(\gE_i = 0) &= \mathbbm{1}_{\{i \ \text{is even}\}} \cdot \frac{\binom{i}{\frac{i}{2}}\binom{n-i}{\frac{n-i}{2}}}{\binom{n}{\frac{n}{2}}},
    \end{align*}
    where $\gE_i = 0$ has nonzero probability if and only if $i$ is even. We also use the following lemma.
    \begin{restatable}{lemma}{lemmortici}
        \label{lem:mortici}
        For even, positive integers $n$, $i$ with $n \ge 4$ and $2 \le i \le \left\lfloor \frac{n}{2} \right\rfloor$, we have
        \begin{align*}
            \frac{\binom{i}{\frac{i}{2}}\binom{n-i}{\frac{n-i}{2}}}{\binom{n}{\frac{n}{2}}} \ge \frac{2}{5\sqrt{i}}.
        \end{align*}
    \end{restatable}
    This lemma yields $\sP(\gE_i = 0) \ge \frac{2}{5\sqrt{i}}$ for even $i$.
    We prove \cref{lem:mortici} at the very end of \cref{subsec:a4}.
    
    First, suppose that $i \ge 2$ is an \textit{even} integer. Then since $i \le \frac{n}{2}$, we have for $i \ge 4$:
    \begin{align}
        \E \left[ \left| \gE_{i} \right| \right] &= \left( 1 - \frac{1}{n - i + 1} \right) \E \left[ \left| \gE_{i-1} \right| \right] + \sP(\gE_{i-1} = 0) \notag \\
        &\ge \left( 1 - \frac{2}{n} \right) \E \left[ \left| \gE_{i-1} \right| \right] + \mathbbm{1}_{i-1 \ \text{is even}} \frac{2}{5\sqrt{i-1}} \notag \\
        &= \left( 1 - \frac{2}{n} \right) \E \left[ \left| \gE_{i-1} \right| \right] \notag \\
        &\ge \left( 1 - \frac{2}{n} \right) \left( \left( 1 - \frac{2}{n} \right) \E \left[ \left| \gE_{i-2} \right| \right] + \mathbbm{1}_{i-2 \ \text{is even}} \frac{2}{5\sqrt{i-2}} \right) \notag \\
        &= \left( 1 - \frac{2}{n} \right)^2 \E \left[ \left| \gE_{i-2} \right| \right] + \left( 1 - \frac{2}{n} \right) \frac{2}{5\sqrt{i-2}}.  \label{eq:iterateineq}
    \end{align}
    We can also explicitly compute the base case $i=2$ as
    \begin{align}
        \E \left[ \left| \gE_{2} \right| \right] &= 2 \cdot \frac{2 \cdot \binom{n-2}{\frac{n}{2}}}{\binom{n}{\frac{n}{2}}} = \frac{4 \cdot (n-2)! (\frac{n}{2})! (\frac{n}{2})!}{(\frac{n}{2})! (\frac{n}{2} - 2)! n!} = \frac{4 (\frac{n}{2}) (\frac{n}{2} - 1)}{n (n-1)} = \frac{n-2}{n-1} = 1 - \frac{1}{n-1} \ge 1 - \frac{2}{n} ,\label{eq:basemortici}
    \end{align}
    from the fact that $\gE_{2} = \pm 2$ each occurs $\binom{n-2}{\frac{n}{2}}$ times among a total of $\binom{n}{\frac{n}{2}}$ cases, and $\gE_2 = 0$ otherwise. Also, note that we automatically have $\E \left[ \left| \gE_{2} \right| \right] \ge 1 - \frac{2}{n} \ge \frac{\sqrt{2}}{10}$, which proves the given statement for $i = 2$.
    
    Now, unrolling the inequalities in (\ref{eq:iterateineq}), we have for $i \ge 4$:
    \note{
    \begin{align}
        \E \left[ \left| \gE_{i} \right| \right] &\ge \left( 1 - \frac{2}{n} \right)^2 \E \left[ \left| \gE_{i-2} \right| \right] + \left( 1 - \frac{2}{n} \right) \frac{2}{5\sqrt{i-2}} \notag \\
        &\ge \left( 1 - \frac{2}{n} \right)^2 \left( \left( 1 - \frac{2}{n} \right)^2 \E \left[ \left| \gE_{i-4} \right| \right] + \left( 1 - \frac{2}{n} \right) \frac{2}{5\sqrt{i-4}} \right) + \left( 1 - \frac{2}{n} \right) \frac{2}{5\sqrt{i-2}} \notag \\
        &= \left( 1 - \frac{2}{n} \right)^4 \E \left[ \left| \gE_{i-4} \right| \right] + \left( 1 - \frac{2}{n} \right)^3 \frac{2}{5\sqrt{i-4}} + \left( 1 - \frac{2}{n} \right) \frac{2}{5\sqrt{i-2}} \notag \\
        &\ \ \vdots \notag \\
        &\ge \left( 1 - \frac{2}{n} \right)^{i-2} \E \left[ \left| \gE_{2} \right| \right] + \left( 1 - \frac{2}{n} \right) \left( \sum_{p=0}^{ \frac{i}{2} - 2 } \left( 1 - \frac{2}{n} \right)^{2p} \frac{2}{5\sqrt{i -  2 - 2p}} \right) \notag \\
        &\ge \left( 1 - \frac{2}{n} \right)^{i-1} + \left( 1 - \frac{2}{n} \right) \left( \sum_{p=0}^{ \frac{i}{2} - 2 } \left( 1 - \frac{2}{n} \right)^{2p} \frac{2}{5\sqrt{i -  2 - 2p}} \right) &(\because \text{\cref{eq:basemortici}}) \notag \\
        &\ge \left( 1 - \frac{2}{n} \right)^{i-1} + \left( 1 - \frac{2}{n} \right) \left( \sum_{p=0}^{ \frac{i}{2} - 2 } \left( 1 - \frac{2}{n} \right)^{2p} \right) \frac{2}{5\sqrt{i-2}} \notag \\
        &\ge \left( 1 - \frac{2}{n} \right) \left( \sum_{p=0}^{ \frac{i}{2} - 1 } \left( 1 - \frac{2}{n} \right)^{2p} \right) \frac{2}{5\sqrt{i-2}} \notag \\
        &= \left( 1 - \frac{2}{n} \right) \cdot \frac{1 - \left( 1 - \frac{2}{n} \right)^{i}}{1 - \left( 1 - \frac{2}{n} \right)^2} \cdot \frac{2}{5\sqrt{i-2}} \notag \\
        &= \left( 1 - \frac{2}{n} \right) \cdot \frac{1}{\frac{4}{n} - \frac{4}{n^2}} \cdot \left( 1 - \left( 1 - \frac{2}{n} \right)^{i} \right) \cdot \frac{2}{5\sqrt{i-2}} \notag \\
        &\ge \left( 1 - \frac{2}{n} \right) \cdot \frac{1}{\frac{4}{n} - \frac{4}{n^2}} \cdot \left( 1 - \frac{1}{1 + \frac{2i}{n}} \right) \cdot \frac{2}{5\sqrt{i-2}} \label{eq:bernoulli} \\
        &= \frac{n (n-2)}{4 (n-1)} \cdot \frac{2i}{n + 2i} \cdot \frac{2}{5\sqrt{i-2}} \notag \\
        &= \left( \frac{n-2}{n-1} \cdot \frac{\sqrt{i}}{\sqrt{i-2}} \right) \left( \frac{n}{n + 2i} \cdot \frac{\sqrt{i}}{5} \right) \ge \frac{n}{n + 2i} \cdot \frac{\sqrt{i}}{5} \ge \frac{\sqrt{i}}{10}. \label{eq:final}
    \end{align}
    }
    In (\ref{eq:bernoulli}) we use $(1 - x)^r \le \frac{1}{1 + r x}$ for all $0 \le x \le 1$ and $r \ge 0$. In (\ref{eq:final}) we use the fact that $2i \le n$ and $\frac{n-2}{n-1} \cdot \frac{\sqrt{i}}{\sqrt{i-2}} \ge 1$, which is equivalent to
    \begin{align*}
        i (n-2)^2 \ge (i-2) (n-1)^2 \quad \Leftrightarrow \quad 2 (n-1)^2 \ge i(2n-3),
    \end{align*}
    which can be easily verified since $i (2n-3) \le \frac{n}{2} (2n-3) = n^2 - \frac{3}{2} n \le 2 (n-1)^2$ for all $n \ge 2$.
    Also, note that we have to deal with the last iterate separately since \cref{lem:mortici} applies only for $i \ge 2$. 

    Now suppose that $i \ge 1$ is an \textit{odd} integer. Then we have
    \begin{align*}
        \E \left[ \left| \gE_{i+1} \right| \right] &= \left( 1 - \frac{1}{n - i} \right) \E \left[ \left| \gE_{i} \right| \right] + \mathbbm{1}_{i \ \text{is even}} \frac{\binom{i}{\frac{i}{2}}\binom{n-i}{\frac{n-i}{2}}}{\binom{n}{\frac{n}{2}}} = \left( 1 - \frac{1}{n - i} \right) \E \left[ \left| \gE_{i} \right| \right] &(\because i \ \text{is odd})
    \end{align*}
    and since $i+1$ is even, we can use the previous result as
    \begin{align*}
        \E \left[ \left| \gE_{i} \right| \right] &= \frac{n-i}{n-i-1} \E \left[ \left| \gE_{i+1} \right| \right] \ge \frac{n-i}{n-i-1} \frac{\sqrt{i+1}}{10}.
    \end{align*}
    Finally, since $\frac{n-i}{n-i-1} \ge 1 \ge \frac{\sqrt{i}}{\sqrt{i+1}}$, we can conclude that
    \begin{align*}
        \E \left[ \left| \gE_{i} \right| \right] &\ge \frac{n-i}{n-i-1} \frac{\sqrt{i+1}}{10} \ge \frac{\sqrt{i}}{10}.
    \end{align*}
\end{proof}


\lemmortici*
    
\begin{proof}
    From Theorem 1 of \citet{mortici}, for all $n \ge 1$ we have the expression
    \begin{align*}
        n! &= \sqrt{\pi(2n + \alpha_n)} \cdot  \frac{n^n}{e^n}, \quad \text{for some value} \ 0.333  \le \alpha_n \le 0.354.
    \end{align*}
    Since $\frac{i}{2} \ge 1$, we have
    \begin{align*}
        \binom{i}{\frac{i}{2}} &= \frac{i!}{((\frac{i}{2})!)^2} = \frac{i^i}{e^i} \cdot \frac{e^i}{(\frac{i}{2})^i} \frac{\sqrt{\pi(2i + \alpha_{i})}}{\pi(i + \alpha_{i/2})} = 2^i \cdot \frac{\sqrt{(2i + \alpha_{i})}}{\sqrt{\pi}(i + \alpha_{i/2})}.
    \end{align*}
    Then we can compute
    \begin{align*}
        \frac{\binom{i}{\frac{i}{2}}\binom{n-i}{\frac{n-i}{2}}}{\binom{n}{\frac{n}{2}}} &= \frac{\sqrt{(2i + \alpha_{i})}}{\sqrt{\pi}(i + \alpha_{i/2})} \frac{\sqrt{(2(n-i) + \alpha_{n-i})}}{\sqrt{\pi}(n-i + \alpha_{(n-i)/2})} \frac{\sqrt{\pi}(n + \alpha_{n/2})}{\sqrt{(2n + \alpha_{n})}} \\
        &= \frac{1}{\sqrt{\pi}} \frac{\sqrt{(2i + \alpha_{i})}}{(i + \alpha_{i/2})} \frac{\sqrt{(2(n-i) + \alpha_{n-i})}}{(n-i + \alpha_{(n-i)/2})} \frac{(n + \alpha_{n/2})}{\sqrt{(2n + \alpha_{n})}} \\
        &\ge \frac{1}{\sqrt{\pi}} \frac{\sqrt{(2i + 0.33)}}{(i + 0.354)} \frac{\sqrt{(2(n-i) + 0.33)}}{(n-i + 0.354)} \frac{(n + 0.33)}{\sqrt{(2n + 0.354)}} \\
        &\ge \frac{1}{\sqrt{\pi}} \frac{\sqrt{2i}}{1.354 i} \frac{\sqrt{(2(n-i))}}{1.354(n-i)} \frac{n}{\sqrt{2.354n}} = \frac{2}{1.354^2 \sqrt{2.354 \pi}} \frac{\sqrt{n}}{\sqrt{i(n-i)}} \\
        &\ge \frac{2}{1.354^2 \sqrt{2.354 \pi}} \frac{1}{\sqrt{i}} \ge \frac{2}{5 \sqrt{i}},
    \end{align*}
    where $\frac{2}{1.354^2 \sqrt{2.354 \pi}} = 0.401157 \dots \ge 0.4 = \frac{2}{5}$.
\end{proof}
\section{\texorpdfstring{Proof of \cref{thm:xavg}}{Proof of Theorem 3}} \label{app:xavg}

Here we prove \cref{thm:xavg}, restated below for the sake of readability.

\xavg* 

\begin{proof}
    We prove the theorem statement for the same constants $c_1 = 2415$ and $c_2 = 161$ as in \cref{thm:xhat}.
    
    Similarly as in \cref{app:xhat}, we use different objective functions for two step-size regimes and aggregate the functions to obtain the final lower bound. Here we also assume $n$ is even, where we can easily extend to odd $n$'s by the same reasoning as in \cref{app:xhat}.

    Here we prove the following lower bounds for each regime. Here $F_j^*$ is the minimizer of $F_j$ for $j = 1, 2$. 
    \begin{itemize}
        \item If $\eta \in \left( 0, \frac{1}{\mu n K} \right)$, there exists a 1-dimensional objective function $F_1 (x) \in \gF (L, \mu, 0, \nu)$ such that \textsf{SGD-RR} with initialization $x_0^1 = D_0$ (for any $D_0$) satisfies
        \begin{align*}
            \mathbb{E} \left[ F_1 (\hat{x}) - F_1^* \right] &= \Omega \left( \mu D_0^2 \right).
        \end{align*}
        \item If $\eta \in \left[ \frac{1}{\mu n K}, \frac{1}{161 L n} \right]$, there exists a 1-dimensional objective function $F_2 (x) \in \gF (L, \mu, 0, \nu)$ such that \textsf{SGD-RR} with initialization $x_0^1 = \note{\frac{1}{27000} \cdot \frac{\nu}{\mu n^{1/2} K}}$ satisfies
        \begin{align*}
            \mathbb{E} \left[ F_2 (\hat{x}) - F_2^* \right] = \Omega \left( \frac{L \nu^2}{\mu^2 n K^2} \right).
        \end{align*}
    \end{itemize}

    Now we define the $2$-dimensional function $F(x, y) = F_1(x) + F_2(y)$, where $F_1$ and $F_2$ are chosen to satisfy the above lower bounds for $\nu$ replaced by $\frac{\nu}{\sqrt{2}}$. Following the analyses in \cref{app:xhat}, from $F_1, F_2 \in \gF (L, \mu, 0, \frac{\nu}{\sqrt{2}})$ (by construction) we have $F \in \gF (L, \mu, 0, \nu)$.

    Now suppose that we set $D_0 = \frac{\nu}{\mu}$ and initialize at the point $\left( \frac{\nu}{\mu}, \note{\frac{1}{27000} \cdot \frac{\nu}{\mu n^{1/2} K}} \right)$.

    If $K \ge 161 \kappa$, then since $\frac{1}{\mu n K} \le \frac{1}{161 L n}$ we can use the lower bound for $F_2(y)$. The lower bound for this case becomes
    \begin{align*}
        \E \left[ F (\hat{x}, \hat{y}) - F^* \right] &= \Omega \left( \min \left\{\frac{\nu^2}{\mu}, \frac{L \nu^2}{\mu^2 n K^2}\right\} \right) = \Omega \left( \frac{L \nu^2}{\mu^2 n K^2} \right).
    \end{align*}
    If $K < 161 \kappa$, then since $\frac{1}{\mu n K} > \frac{1}{161 L n}$ the latter step-size regime does not exist, i.e., we \textit{cannot} use the lower bound for $F_2(y)$, and the lower bound for this case becomes
    \begin{align*}
        \E \left[ F (\hat{x}, \hat{y}) - F^* \right] &= \Omega \left( \frac{\nu^2}{\mu} \right),
    \end{align*}
    which completes the proof.
\end{proof}

\subsection{\texorpdfstring{Lower Bound for $\eta \in \left( 0, \frac{1}{\mu n K} \right)$}{Lower Bound for η ∈ (0, 1/μnK)}}
Here we show that there exists $F_1 (x) \in \gF (L, \mu, 0, \nu)$ such that \textsf{SGD-RR} with $x_0^1 = D_0$ satisfies
\begin{align*}
    \mathbb{E} \left[ F_1 (\hat{x}) - F_1^* \right] &= \Omega \left( \mu D_0^2 \right).
\end{align*}

\begin{proof}
    We define the same $F_1 (x) \in \gF (\mu, \mu, 0, 0)$ as in \cref{subsec:a1} by the following components.
    \begin{align*}
        f_{i} (x) &= F_1(x) = \frac{\mu x^2}{2}
    \end{align*}
    Note that $\gF (\mu, \mu, 0, 0) \subseteq \gF (L, \mu, 0, \nu)$ and $F_1^* = 0$ at $x^* = 0$ by definition.
    
    We start from \cref{eq:smalletaineq} in \cref{app:xhat}, which gives
    \begin{align*}
        x_{0}^{k+1} = (1 - \eta \mu)^{nk} \cdot D_0 \ge \left( 1 - \frac{1}{nK} \right)^{nK} \cdot D_0 \ge \frac{D_0}{4}
    \end{align*}
    for all $k$. Then for any weighted average $\hat{x}$ we have
    \begin{align*}
        \hat{x} &= \frac{\sum_{k=1}^{K+1} \alpha_k x_0^k}{\sum_{k=1}^{K+1} \alpha_k} \ge \frac{\sum_{k=1}^{K+1} \alpha_k \frac{D_0}{4}}{\sum_{k=1}^{K+1} \alpha_k} = \frac{D_0}{4}
    \end{align*}
    and therefore
    \begin{align*}
        F_1 (\hat{x}) &\ge \frac{\mu}{2} \left( \frac{D_0}{4} \right)^2 = \frac{\mu D_0^2}{32},
    \end{align*}
    which concludes that $\mathbb{E} \left[ F_1 (\hat{x}) - F_1^* \right] = \mathbb{E} \left[ F_1 (\hat{x}) \right] = F_1 (\hat{x}) = \Omega \left( \mu D_0^2 \right)$.
\end{proof}

\subsection{\texorpdfstring{Lower Bound for $\eta \in \left[ \frac{1}{\mu n K}, \frac{1}{161 L n} \right]$}{Lower Bound for η ∈ [1/μnK, 1/161Ln]}}

Here we show that there exists $F_2 (x) \in \gF (L, \mu, 0, \nu)$ such that \textsf{SGD-RR} with $x_0^1 = \note{\frac{1}{27000} \cdot \frac{\nu}{\mu n^{1/2} K}}$ satisfies
\begin{align*}
    \mathbb{E} \left[ F_2 (\hat{x}) - F_2^* \right] = \Omega \left( \frac{L \nu^2}{\mu^2 n K^2} \right).
\end{align*}

\begin{proof}
    We define the $F_2 (x) \in \gF (L, \mu, 0, \nu)$ as in \cref{subsec:a2} by the following components:
    \begin{align*}
        f_{i} (x) &= \begin{cases}
            \left( L \mathbbm{1}_{x < 0} + \mu_0 \mathbbm{1}_{x \ge 0} \right) \frac{x^2}{2} + \nu x & \text{if} \ i \le n/2, \\
            \left( L \mathbbm{1}_{x < 0} + \mu_0 \mathbbm{1}_{x \ge 0} \right) \frac{x^2}{2} - \nu x & \text{otherwise,}
        \end{cases}
    \end{align*}
    where we assume $\mu_0 \le \frac{L}{2415}$ and later choose $\mu_0 = \frac{L}{2415}$. With this construction, the finite-sum objective becomes
    \begin{align*}
        F_2 (x) &= \frac{1}{n} \sum_{i=1}^{n} f_{i}(x) = \left( L \mathbbm{1}_{x < 0} + \mu_0 \mathbbm{1}_{x \ge 0} \right) \frac{x^2}{2}.
    \end{align*}
    Note that $F_2^* = 0$ at $x^* = 0$ by definition, and that $\mu_0$ is different from $\mu$. While $F_2 (x) \in \gF(L, \mu_0, 0, \nu)$ by construction, we can ensure that $\gF(L, \mu_0, 0, \nu) \subset \gF(L, \mu, 0, \nu)$ because the assumption $\kappa \ge 2415$ implies $\mu_0 = \frac{L}{2415} \ge \mu$.

    We start from \cref{eq:midetaineq} in \cref{app:xhat}, which gives
    \begin{align*}
        \mathbb{E} \left[ x_{0}^{k+1} \right] = \mathbb{E} \left[ x_{n}^{k} \right] &\ge \note{\frac{1}{27000} \cdot \frac{\nu}{\mu n^{1/2} K}}
    \end{align*}
    for all $k$. If we set $x_{0} \ge \note{\frac{1}{27000} \cdot \frac{\nu}{\mu n^{1/2} K}}$ then all end-of-epoch iterates must maintain $\mathbb{E} [x_{n}^{k}] \ge \note{\frac{1}{27000} \cdot \frac{\nu}{\mu n^{1/2} K}}$. This implies that for any weighted average $\hat{x}$ we have
    \begin{align*}
        \mathbb{E} [\hat{x}] &= \E \left[ \frac{\sum_{k=1}^{K+1} \alpha_k x_0^k}{\sum_{k=1}^{K+1} \alpha_k} \right] = \frac{\sum_{k=1}^{K+1} \alpha_k \E \left[ x_0^k \right]}{\sum_{k=1}^{K+1} \alpha_k} \ge \frac{\sum_{k=1}^{K+1} \alpha_k \left( \note{\frac{1}{27000} \cdot \frac{\nu}{\mu n^{1/2} K}} \right)}{\sum_{k=1}^{K+1} \alpha_k} = \note{\frac{1}{27000} \cdot \frac{\nu}{\mu n^{1/2} K}}.
    \end{align*}
    Finally, by Jensen's inequality, we have
    \begin{align*}
        \mathbb{E} \left[ F_2 ( \hat{x} ) - F_2^* \right] &= \mathbb{E} \left[ F_2 ( \hat{x} ) \right] \\
        &\ge \frac{L}{2 \cdot 2415} \mathbb{E} \left[ \hat{x}^2 \right] \\
        &\ge \frac{L}{4830} \mathbb{E} \left[ \hat{x} \right]^2 \\
        & \ge \frac{L}{4830} \cdot \left( \frac{1}{27000} \cdot \frac{\nu}{\mu n^{1/2} K} \right)^2 = \Omega \left( \frac{L \nu^2}{\mu^2 n K^2} \right).
    \end{align*}
\end{proof}

\subsection{\texorpdfstring{Proof of \cref{cor:cvxcase}}{Proof of Corollary 5}}
\label{subsec:b2}
Here we prove \cref{cor:cvxcase}, restated below for the sake of readability.
\corcvxcase*
\begin{proof}
    Suppose that $c_1$, $c_2$ are the same constants as in \cref{thm:xavg}, and let $c_3 = \max \{ c_1^{3/2}, c_2^3 \}$. We use results of \cref{thm:xavg} in \cref{app:xavg}, but we view the initialization $D_0$ for the first coordinate as a separate constant instead of setting to a certain value like $D_0 = \frac{\nu}{\mu}$.

    Let us choose $\mu = \frac{L^{1/3} \nu^{2/3}}{D^{2/3} n^{1/3} K^{2/3}}$. First, we can check that our choice of $\mu$ and $K \note{ \ge c_3 \frac{\nu}{\mu D n^{1/2}}} \ge c_3 \frac{\nu}{L D n^{1/2}}$ implies 
    \begin{align*}
        \kappa = \frac{L^{2/3} D^{2/3} n^{1/3} K^{2/3}}{\nu^{2/3}} \ge c_3^{2/3} \ge c_1,
    \end{align*}
    and $K \ge c_3 \frac{L^2 D^2 n}{\nu^2}$ implies 
    \begin{align*}
        K \ge c_3^{1/3} \cdot \frac{L^{2/3} D^{2/3} n^{1/3} K^{2/3}}{\nu^{2/3}} = c_3^{1/3} \kappa \ge c_2 \kappa.
    \end{align*}
    Since $\kappa \ge c_1$ and $K \ge c_2 \kappa$, by \cref{thm:xavg} there exists some $F_{\mu} \in \gF(L, \mu, 0, \nu)$ satisfying
    \begin{align*}
        \E \left[ F_{\mu} (\hat{\vx}) - F_{\mu}^* \right] = \Omega \left( \min \left\{\mu D_0^2, \frac{L \nu^2}{\mu^2 n K^2} \right\} \right).
    \end{align*}
    \note{for initialization $(D_0, \frac{1}{27000} \cdot \frac{\nu}{\mu n^{1/2} K})$. Note that we have $D^2 = D_0^2 + \frac{1}{27000^2} \cdot \frac{\nu^2}{\mu^2 n K^2}$. 
    Since $K \ge c_3 \frac{\nu}{\mu D n^{1/2}}$, or equivalently $D \ge c_3 \frac{\nu}{\mu n^{1/2} K}$, we have $D_0 = \Omega (\frac{\nu}{\mu n^{1/2} K})$ and therefore}
    \begin{align*}
        \note{\E \left[ F_{\mu} (\hat{\vx}) - F_{\mu}^* \right] = \Omega \left( \min \left\{\mu D^2, \frac{L \nu^2}{\mu^2 n K^2} \right\} \right).}
    \end{align*}
    Letting $F = F_{\mu}$ and plugging in $\mu = \frac{L^{1/3} \nu^{2/3}}{D^{2/3} n^{1/3} K^{2/3}}$, we can conclude that
    \begin{align*}
        \E \left[ F(\hat{\vx}) - F^* \right] = \Omega \left( \frac{L^{1/3} \nu^{2/3} D^{4/3}}{n^{1/3} K^{2/3}} \right).
    \end{align*}
    Finally we can check that $F \in \gF(L, \mu, 0, \nu) \subseteq \gF(L, 0, 0, \nu)$ for all $\mu > 0$, which completes the proof.
\end{proof}
\section{\texorpdfstring{Proof of \cref{thm:tailub}}{Proof of Proposition 4}} \label{app:sgdrrub}

Here we prove \cref{thm:tailub}, restated below for the sake of readability.
\tailub*

\begin{proof}
    We start from the following lemma from \citet{mish20}.
    \begin{lemma}[\citet{mish20}, Lemma~3]
        Assume that functions $f_1$, \dots, $f_n$ are convex and $F$, $f_1$, \dots, $f_n$ are $L$-smooth. Suppose that $\sigma_*^2$ is defined as in \cref{eq:sigmastar}. Then \textsf{SGD-RR} with step size $\eta \le \frac{1}{\sqrt{2} L n}$ satisfies
        \begin{align*}
            \E \left[ \| \vx_0^{k+1} - \vx^* \|^2 \right] &\le \E \left[ \| \vx_0^{k} - \vx^* \|^2 \right] - 2 \eta n \E \left[ F (\vx_0^{k+1}) - F(\vx^*)  \right] + \frac{1}{2} \eta^3 L \sigma_*^2 n^2.
        \end{align*}
    \end{lemma}
    Note that the assumptions in \cref{def:fclass} of $\gF (L, \mu, 0, \nu)$ includes all the required conditions above, \textit{plus} an additional condition that $F$ is $\mu$-strongly convex. Also, note that the $\sigma_*^2$ term from the original paper can be replaced with $\nu^2$, which is safe by the same reasoning as what we mentioned in \cref{prop:strcvxub}. Hence we can use the following inequality:
    \begin{align*}
        \E \left[ \| \vx_0^{k+1} - \vx^* \|^2 \right] &\le \E \left[ \| \vx_0^{k} - \vx^* \|^2 \right] - 2 \eta n \E \left[ F (\vx_0^{k+1}) - F(\vx^*)  \right] + \frac{1}{2} \eta^3 L \nu^2 n^2.
    \end{align*}
    From strong convexity, for all $k$ we have
    \begin{align}
        \E \left[ F (\vx_0^{k+1}) - F (\vx^*)  \right] &\ge \frac{\mu}{2} \E \left[ \| \vx_0^{k+1} - \vx^* \|^2 \right]. \label{eq:ubstrcvx}
    \end{align}
    Now we apply (\ref{eq:ubstrcvx}) to \textit{only exactly half} of the term involving $\E \left[ F (\vx_0^{k+1}) - F (\vx^*)  \right]$ to obtain
    \begin{align*}
        \left( 1 + \frac{\eta \mu n}{2} \right) \E \left[ \| \vx_0^{k+1} - \vx^* \|^2 \right] &\le \E \left[ \| \vx_0^{k} - \vx^* \|^2 \right] - \eta n \E \left[ F (\vx_0^{k+1}) -F (\vx^*)  \right] + \frac{1}{2} \eta^3 L \nu^2 n^2.
    \end{align*}
    Since $0 \le \eta \mu n \le \eta L n \le \frac{1}{\sqrt{2}} \le 1$ implies $\frac{1}{1 + \frac{\eta \mu n}{2}} \le 1 - \frac{\eta \mu n}{3}$ and $\frac{2}{3} \le \frac{1}{1 + \frac{\eta \mu n}{2}} \le 1$, we obtain
    \begin{align}
        \E \left[ \| \vx_0^{k+1} - \vx^* \|^2 \right] &\le \frac{1}{1 + \frac{\eta \mu n}{2}} \left( \E \left[ \| \vx_0^{k} - \vx^* \|^2 \right] - \eta n \E \left[ F (\vx_0^{k+1}) -F (\vx^*)  \right] + \frac{1}{2} \eta^3 L \nu^2 n^2 \right) \notag \\
        &\le \left( 1 - \frac{\eta \mu n}{3} \right) \E \left[ \| \vx_0^{k} - \vx^* \|^2 \right] - \frac{2}{3} \eta n \E \left[ F (\vx_0^{k+1}) -F (\vx^*)  \right] + \frac{1}{2} \eta^3 L \nu^2 n^2. \label{eq:apply}
    \end{align}
    
    We derive two different types of weaker inequalities from (\ref{eq:apply}), as:
    \begin{align}
        \E \left[ \| \vx_0^{k+1} - \vx^* \|^2 \right] &\le \left( 1 - \frac{\eta \mu n}{3} \right) \E \left[ \| \vx_0^{k} - \vx^* \|^2 \right] + \frac{1}{2} \eta^3 L \nu^2 n^2, \label{eq:firstone} \\
        \E \left[ \| \vx_0^{k+1} - \vx^* \|^2 \right] &\le \E \left[ \| \vx_0^{k} - \vx^* \|^2 \right] - \frac{2}{3} \eta n \E \left[ F (\vx_0^{k+1}) - F (\vx^*)  \right] + \frac{1}{2} \eta^3 L \nu^2 n^2. \label{eq:secondone}
    \end{align}
    From (\ref{eq:firstone}), we can unroll the inequality to obtain
    \begin{align}
        \E \left[ \| \vx_0^{k+1} - \vx^* \|^2 \right] &\le \left( 1 - \frac{\eta \mu n}{3} \right)^k \E \left[ \| \vx_0^{1} - \vx^* \|^2 \right] + \frac{1}{2} \eta^3 L \nu^2 n^2 \sum_{j=0}^{K-1} \left( 1 - \frac{\eta \mu n}{3} \right)^j \notag \\
        &\le \left( 1 - \frac{\eta \mu n}{3} \right)^k \E \left[ \| \vx_0^{1} - \vx^* \|^2 \right] + \frac{1}{2} \eta^3 L \nu^2  n^2\sum_{j=0}^{\infty} \left( 1 - \frac{\eta \mu n}{3} \right)^j \notag \\
        &= \left( 1 - \frac{\eta \mu n}{3} \right)^k \E \left[ \| \vx_0^{1} - \vx^* \|^2 \right] + \frac{1}{2} \eta^3 L \nu^2 n^2 \frac{1}{1 - \left( 1 - \frac{\eta \mu n}{3} \right)} \notag \\
        &= \left( 1 - \frac{\eta \mu n}{3} \right)^k D^2 + \frac{3}{2} \cdot \frac{\eta^2 L \nu^2 n}{\mu} &(D := \| \vx_0^{1} - \vx^* \|) \notag \\
        &\le e^{- \frac{1}{3} \eta \mu n k} D^2 + \frac{3}{2} \cdot \frac{\eta^2 L \nu^2 n}{\mu} \label{eq:halfx}
    \end{align}
    which holds for all $k$.
    
    From (\ref{eq:secondone}), we can rearrange terms as
    \begin{align*}
        \eta n \E \left[ F (\vx_0^{k+1}) - F (\vx^*)  \right] &\le \frac{3}{2} \E \left[ \| \vx_0^{k} - \vx^* \|^2 \right] - \frac{3}{2} \E \left[ \| \vx_0^{k+1} - \vx^* \|^2 \right] + \frac{3}{4} \eta^3 L \nu^2 n^2
    \end{align*}
    and average the inequality from $k = \lceil \frac{K}{2} \rceil$ to $K$ to obtain
    \begin{align}
        &\phantom{{}\le{}} \frac{\eta n}{K - \lceil \frac{K}{2} \rceil + 1} \sum_{k = \lceil \frac{K}{2} \rceil}^{K} \E \left[ F (\vx_0^{k+1}) - F (\vx^*)  \right] \notag \\
        &\le \frac{1}{K - \lceil \frac{K}{2} \rceil + 1} \left( \frac{3}{2} \E \left[ \| \vx_0^{\lceil \frac{K}{2} \rceil} - \vx^* \|^2 \right] - \frac{3}{2} \E \left[ \| \vx_0^{K+1} - \vx^* \|^2 \right] \right) + \frac{3}{4} \eta^3 L \nu^2 n^2 \notag \\
        &\le \frac{3/2}{K - \lceil \frac{K}{2} \rceil + 1} \E \left[ \| \vx_0^{\lceil \frac{K}{2} \rceil} - \vx^* \|^2 \right] + \frac{3}{4} \eta^3 L \nu^2 n^2. \label{eq:iphone}
    \end{align}
    Therefore we have
    \begin{align}
        \mathbb{E} \left[ F (\hat{\vx}_{\text{tail}}) - F^* \right] &= \mathbb{E} \left[ F \left( \frac{1}{K - \lceil \frac{K}{2} \rceil + 1} \sum_{k=\lceil \frac{K}{2} \rceil}^{K} \vx_n^{k} \right) - F^* \right] \notag \\
        &= \mathbb{E} \left[ F \left( \frac{1}{K - \lceil \frac{K}{2} \rceil + 1} \sum_{k=\lceil \frac{K}{2} \rceil}^{K} \vx_0^{k+1} \right) - F^* \right] \notag \\
        &\le \frac{1}{K - \lceil \frac{K}{2} \rceil + 1} \sum_{k=\lceil \frac{K}{2} \rceil}^{K} \E \left[ F (\vx_0^{k+1}) - F (\vx^*)  \right] &(\because \text{Jensen's inequality}) \notag \\
        &\le \frac{3/2}{K - \lceil \frac{K}{2} \rceil + 1} \cdot \frac{1}{\eta n} \E \left[ \| \vx_0^{\lceil \frac{K}{2} \rceil} - \vx^* \|^2 \right] + \frac{3}{4} \eta^2 L \nu^2 n &(\because \text{By (\ref{eq:iphone})}) \notag \\
        &\le \frac{3}{\eta n K} \E \left[ \| \vx_0^{\lceil \frac{K}{2} \rceil} - \vx^* \|^2 \right] + \frac{3}{4} \eta^2 L \nu^2 n &\left( \because \left\lceil \tfrac{K}{2} \right\rceil \le \tfrac{K}{2} + 1 \right) \notag \\
        &\le \frac{3}{\eta n K} \left( e^{- \frac{1}{3} \eta \mu n \left( \lceil \frac{K}{2} \rceil - 1 \right)} D^2 + \frac{3}{2} \cdot \frac{\eta^2 L \nu^2 n}{\mu} \right) + \frac{1}{2} \eta^2 L \nu^2 n &\left( \because \text{By (\ref{eq:halfx}), for} \ k = \left\lceil \tfrac{K}{2} \right\rceil \right) \notag \\
        &= \frac{3 D^2}{\eta n K} e^{- \frac{1}{3} \eta \mu n \left( \lceil \frac{K}{2} \rceil - 1 \right)} + \frac{9}{2} \cdot \frac{\eta L \nu^2}{\mu K} + \frac{1}{2} \eta^2 L \nu^2 n \notag \\
        &\le \frac{3D^2}{\eta n K} e^{- \frac{1}{9} \eta \mu n K} + \frac{9}{2} \cdot \frac{\eta L \nu^2}{\mu K} + \frac{1}{2} \eta^2 L \nu^2 n. \label{eq:ublin}
    \end{align}
    Note that in the last inequality, $K \ge 5$ implies $\lceil \frac{K}{2} \rceil - 1 \ge \frac{K}{3}$.

    Now we will divide into four possible cases according to how we choose $\eta$, and then derive that desired upper bound holds in each case from \cref{eq:ublin}. Note that we have $\max \left\{ 1, \log \left( \frac{\mu^3 n D^2 K^2}{L \nu^2} \right) \right\} = 1$ if and only if $K \le \frac{e^{1/2} L^{1/2} \nu}{\mu^{3/2} n^{1/2} D}$, which is again equivalent to $\mu D^2 \le \frac{e L \nu^2}{\mu^2 n K^2}$.

    \paragraph{Case (a)} Suppose that $\eta = \frac{1}{\sqrt{2}Ln} \le \frac{9}{\mu n K} \log \left( \frac{\mu^3 n D^2 K^2}{L \nu^2} \right)$, where $\max \left\{ 1, \log \left( \frac{\mu^3 n D^2 K^2}{L \nu^2} \right) \right\} = \log \left( \frac{\mu^3 n D^2 K^2}{L \nu^2} \right)$.

    From \cref{eq:ublin} we have
    \begin{align*}
        \mathbb{E} \left[ F (\hat{\vx}_{\text{tail}}) - F^* \right] &\le \frac{3D^2}{\eta n K} e^{- \frac{1}{9} \eta \mu n K} + \frac{9}{2} \cdot \frac{\eta L \nu^2}{\mu K} + \frac{1}{2} \eta^2 L \nu^2 n \\
        &\le \frac{3 \sqrt{2} L D^2}{K} e^{- \frac{1}{9 \sqrt{2}} \frac{K}{L / \mu}} + \frac{81}{2} \frac{L \nu^2}{\mu^2 n K^2} \log \left( \frac{\mu^3 n D^2 K^2}{L \nu^2} \right) + \frac{81}{2} \frac{L \nu^2}{\mu^2 n K^2} \log^2 \left( \frac{\mu^3 n D^2 K^2}{L \nu^2} \right) \\
        &= \tilde{\gO} \left( \frac{L D^2}{K} e^{- \frac{1}{9 \sqrt{2}} \frac{K}{L / \mu}} + \frac{L \nu^2}{\mu^2 n K^2} \right).
    \end{align*}

    \paragraph{Case (b)} Suppose that $\eta = \frac{1}{\sqrt{2}Ln} \le \frac{9}{\mu n K}$, where $\max \left\{ 1, \log \left( \frac{\mu^3 n D^2 K^2}{L \nu^2} \right) \right\} = 1$.

    From \cref{eq:ublin} we have
    \begin{align*}
        \mathbb{E} \left[ F (\hat{\vx}_{\text{tail}}) - F^* \right] &\le \frac{3D^2}{\eta n K} e^{- \frac{1}{9} \eta \mu n K} + \frac{9}{2} \cdot \frac{\eta L \nu^2}{\mu K} + \frac{1}{2} \eta^2 L \nu^2 n \\
        &\le \frac{3 \sqrt{2} L D^2}{K} e^{- \frac{1}{9 \sqrt{2}} \frac{K}{L / \mu}} + \frac{81}{2} \frac{L \nu^2}{\mu^2 n K^2} + \frac{81}{2} \frac{L \nu^2}{\mu^2 n K^2} \\
        &= \tilde{\gO} \left( \frac{L D^2}{K} e^{- \frac{1}{9 \sqrt{2}} \frac{K}{L / \mu}} + \frac{L \nu^2}{\mu^2 n K^2} \right).
    \end{align*}

    \paragraph{Case (c)} Suppose that $\eta = \frac{9}{\mu n K} \log \left( \frac{\mu^3 n D^2 K^2}{L \nu^2} \right) \le \frac{1}{\sqrt{2}Ln}$, where $\max \left\{ 1, \log \left( \frac{\mu^3 n D^2 K^2}{L \nu^2} \right) \right\} = \log \left( \frac{\mu^3 n D^2 K^2}{L \nu^2} \right)$.

    From \cref{eq:ublin} we have
    \begin{align*}
        \mathbb{E} \left[ F (\hat{\vx}_{\text{tail}}) - F^* \right] &\le \frac{3D^2}{\eta n K} e^{- \frac{1}{9} \eta \mu n K} + \frac{9}{2} \cdot \frac{\eta L \nu^2}{\mu K} + \frac{1}{2} \eta^2 L \nu^2 n \\
        &= \frac{\mu D^2}{3 \log \left( \frac{\mu^3 n D^2 K^2}{L \nu^2} \right)} \cdot \frac{L \nu^2}{\mu^3 n D^2 K^2} + \frac{81}{2} \frac{L \nu^2}{\mu^2 n K^2} \log \left( \frac{\mu^3 n D^2 K^2}{L \nu^2} \right) + \frac{81}{2} \frac{L \nu^2}{\mu^2 n K^2} \log^2 \left( \frac{\mu^3 n D^2 K^2}{L \nu^2} \right) \\
        &= \frac{L \nu^2}{\mu^2 n K^2} \left( \frac{1}{3 \log \left( \frac{\mu^3 n D^2 K^2}{L \nu^2} \right)} + \frac{81}{2} \log \left( \frac{\mu^3 n D^2 K^2}{L \nu^2} \right) + \frac{81}{2} \log^2 \left( \frac{\mu^3 n D^2 K^2}{L \nu^2} \right) \right) \\
        &= \tilde{\gO} \left( \frac{L \nu^2}{\mu^2 n K^2} \right).
    \end{align*}

    \paragraph{Case (d)} Suppose that $\eta = \frac{9}{\mu n K} \le \frac{1}{\sqrt{2}Ln}$, where $\max \left\{ 1, \log \left( \frac{\mu^3 n D^2 K^2}{L \nu^2} \right) \right\} = 1$.

    From \cref{eq:ublin} we have
    \begin{align*}
        \mathbb{E} \left[ F (\hat{\vx}_{\text{tail}}) - F^* \right] &\le \frac{3D^2}{\eta n K} e^{- \frac{1}{9} \eta \mu n K} + \frac{9}{2} \cdot \frac{\eta L \nu^2}{\mu K} + \frac{1}{2} \eta^2 L \nu^2 n \\
        &= \frac{\mu D^2}{3e} + \frac{81}{2} \frac{L \nu^2}{\mu^2 n K^2} + \frac{81}{2} \frac{L \nu^2}{\mu^2 n K^2} \\
        &\le \frac{L \nu^2}{\mu^2 n K^2} \cdot \left( \frac{1}{3e} + 81 \right) \\
        &= \tilde{\gO} \left( \frac{L \nu^2}{\mu^2 n K^2} \right).
    \end{align*}
    Therefore \cref{thm:tailub} holds for all cases, which completes the proof.
\end{proof}
\section{\texorpdfstring{Proof of \thmref{thm:grabLBa}}{Proof of Theorem 7}} \label{sec:grabLBa}
Here we prove \cref{thm:grabLBa}, restated below for the sake of readability.
\grabLBa*
\begin{proof}
Similarly as in \cref{app:xhat}, we define objective functions for three step-size regimes and aggregate the functions to obtain the final lower bound.
Here we also assume $n$ is even, where we can easily extend to odd $n$'s by the same reasoning as in \cref{app:xhat}.

We will prove the following lower bounds for each regime. Here $F_j^*$ is the minimizer of $F_j$ for $j = 1, 2, 3$. 

\begin{itemize}
    \item If $\eta \in \left( 0, \frac{1}{\mu n K} \right)$, there exists a $1$-dimensional objective function $F_1(x) \in \gF(L,\mu,0,\nu)$ such that any permutation-based SGD with initialization $x_0^1 = \frac{L^{1/2} \nu}{\mu^{3/2} n K}$ satisfies
    \begin{align*}
        F_1(\hat{x}) - F_1^* = \Omega \left( \frac{L \nu^2}{\mu^2 n^2 K^2} \right).
    \end{align*}
    \item If $\eta \in \left[ \frac{1}{\mu n K}, \frac{1}{L} \right)$, there exists a $2$-dimensional objective function $F_2(y, z) \in \gF \bigopen{L,\mu,0,\sqrt{2}\nu}$ such that any permutation-based SGD with initialization $(y_0^1, z_0^1) = \bigopen{\frac{\nu}{2L}, 0}$ satisfies
    \begin{align*}
        F_2(\hat{y}, \hat{z}) - F_2^* = \Omega \left( \frac{L \nu^2}{\mu^2 n^2 K^2} \right).
    \end{align*}
    \note{Note that $\hat{y}$ and $\hat{z}$ share same weights $\bigset{\alpha_k}_{k=1}^{K+1}$.}
    \item If $\eta \ge \frac{1}{L}$, there exists a $1$-dimensional objective function $F_3(w) \in \gF(2L,\mu,0,\nu)$ such that any permutation-based SGD with initialization $w_0^1 = \frac{\nu}{\mu n K}$ satisfies
    \begin{align*}
        F_3(\hat{w}) - F_3^* = \Omega \left( \frac{L \nu^2}{\mu^2 n^2 K^2} \right).
    \end{align*}
\end{itemize}

Now we define the 4-dimensional function $F(\vx) = F(x, y, z, w) = F_1(x) + F_2(y,z) + F_3(w)$, where $F_1$, $F_2$, and $F_3$ are chosen to satisfy the above lower bounds.

Following the analyses in \cref{app:xhat}, we have $F \in \gF (2L, \mu, 0, 2\nu)$, which allows us to directly apply the convergence rates in the lower bounds of $F_1$, $F_2$ and $F_3$ to the aggregated function $F$. 

When $K \le \frac{\kappa}{n}$, the second step-size regime becomes invalid. In this case, we define $F(x, y, z, w) = F_1(x) + F_1(y) + F_1(z) + F_3(w) \in \gF (2L, \mu, 0, 2 \nu)$. The final lower bound is then the minimum of the lower bounds obtained for the remaining two regimes, \note{which is $\Omega \bigopen{\frac{L \nu^2}{\mu^2 n^2 K^2}}$}.

Note that we assumed $\kappa \ge 4$ and our constructed function is $2L$-\textit{smooth} and $\mu$-\textit{strongly convex}. Thus, $\kappa \ge 4$ is equivalent to $\mu \le \frac{L}{2}$ throughout the proof.

Finally, rescaling $L$ and $\nu$ will give us the function $F \in \gF (L, \mu, 0, \nu)$ satisfying $F(\hat{\vx}) - F^* = \Omega \bigopen{ \frac{L \nu^2}{\mu^2 n^2 K^2} }$.
\end{proof}

For the following subsections, we prove the lower bounds for $F_1$, $F_2$, and $F_3$ at the corresponding step size regimes.

\subsection{\texorpdfstring{Lower Bound for $\eta \in \left( 0, \frac{1}{\mu n K} \right)$}{Lower Bound for η ∈ (0, 1/μnK)}} \label{subsec:grabLBaa}
Here we show that there exists $F_1 (x) \in \gF (L, \mu, 0, \nu)$ such that any permutation-based SGD with $x_0^1 = \frac{L^{1/2} \nu}{\mu^{3/2} n K}$ satisfies
\begin{align*}
    F_1(\hat{x}) - F_1^* = \Omega \left( \frac{L \nu^2}{\mu^2 n^2 K^2} \right).
\end{align*}

\begin{proof}
We define $F_1 (x) \in \gF (\mu, \mu, 0, 0)$ by the following components:
\begin{align*}
    f_i(x) = F(x) = \frac{\mu}{2}x^2.
\end{align*}
Note that $\gF (\mu, \mu, 0, 0) \subseteq \gF (L, \mu, 0, \nu)$ and $F_1^* = 0$ at $x^* = 0$ by definition.

In this regime, we will see that the step size is too small so that $\bigset{x_n^k}_{k=1}^{K}$ cannot even reach near the optimal point.
We start from $x_0^1 = \frac{L^{1/2} \nu}{\mu^{3/2} n K}$. Since the gradient of all component functions evaluated at point $x$ is fixed deterministically to $\mu x$, regardless of the permutation-based SGD algorithm we use, we have
\begin{align*}
    x_n^k = x_0^1 (1 - \eta \mu)^{nk}
    &\ge \frac{L^{1/2} \nu}{\mu^{3/2} n K}\left(1 - \frac{1}{nK} \right)^{nk} \ge \frac{L^{1/2} \nu}{\mu^{3/2} n K}\left(1 - \frac{1}{nK} \right)^{nK}\\
    &\overset{(a)} > \frac{L^{1/2} \nu}{\mu^{3/2} n K} \frac{1}{e} \left(1 - \frac{1}{nK} \right) \overset{(b)} \ge \frac{L^{1/2} \nu}{\mu^{3/2} n K} \frac{1}{2e},
\end{align*}
where (a) comes from \lmmref{lmm:thm1_2} and (b) comes from the assumption that $n \ge 2$. Therefore, we have $\hat{x} = \Omega \bigopen{\frac{L^{1/2} \nu}{\mu^{3/2} n K}}$ for any nonnegative weights $\bigset{\alpha_k}_{k=1}^{K+1}$. With this $\hat{x}$, we have
\begin{align*}
    F_1(\hat{x}) - F_1^* = \frac{\mu}{2} \hat{x}^2 = \Omega \bigopen{\frac{L \nu^2}{\mu^2 n^2 K^2}}.
\end{align*}
\end{proof}

\subsection{\texorpdfstring{Lower Bound for $\eta \in \left[ \frac{1}{\mu n K}, \frac{1}{L} \right)$}{Lower Bound for η ∈ [1/μnK, 1/L)}} \label{subsec:grabLBab}
Here we show that there exists $F_2 (y, z) \in \gF (L, \mu, 0, \sqrt{2}\nu)$ such that any permutation-based SGD with $(y_0^1, z_0^1) = \bigopen{\frac{\nu}{2L}, 0}$ satisfies
\begin{align*}
    F_2(\hat{y}, \hat{z}) - F_2^* = \Omega \left( \frac{L \nu^2}{\mu^2 n^2K^2} \right).
\end{align*}

\begin{proof}
Let us define the function $g_{+1}, g_{-1}$ as follows.
\begin{align*}
    g_{+1} (x) &= \left( L \mathbbm{1}_{x < 0} + \frac{L}{2} \mathbbm{1}_{x \ge 0} \right) \frac{x^2}{2} + \nu x, \\
    g_{-1} (x) &= \left( L \mathbbm{1}_{x < 0} + \frac{L}{2} \mathbbm{1}_{x \ge 0} \right) \frac{x^2}{2} - \nu x.
\end{align*}
Note that $g_{+1}$ and $g_{-1}$ are $\mu$-\textit{strongly convex} since $\mu \le \frac{L}{2}$.
We define $F_2 (x) \in \gF (L, \mu, 0, \sqrt{2}\nu)$ by the following components:
\begin{align*}
    f_i(y, z) =
    \begin{cases}
        g_{+1} (y) + g_{-1} (z) & \text{ if} \,\,\, i \le \frac{n}{2},\\
        g_{-1} (y) + g_{+1} (z) & \text{ otherwise.}
    \end{cases}
\end{align*}
With this construction, the finite-sum objective becomes
\begin{align*}
    F_2 (y, z) &= \frac{1}{n} \sum_{i=1}^{n} f_{i}(y, z) = \left( L \mathbbm{1}_{y < 0} + \frac{L}{2} \mathbbm{1}_{y \ge 0} \right) \frac{y^2}{2} + \left( L \mathbbm{1}_{z < 0} + \frac{L}{2} \mathbbm{1}_{z \ge 0} \right) \frac{z^2}{2}.
\end{align*}
Note that $F_2^* = 0$ at $(y^*, z^*) = (0, 0)$ by definition.

We start at $(y_0^1, z_0^1) = \bigopen{\frac{\nu}{2L}, 0}$.
We now use the following lemma to find the lower bound of $\bigset{y_n^k + z_n^k}_{k=1}^{K}$ that holds for every permutation.

\begin{restatable}{lemma}{grabUBalmma} \label{lmm:grabUBalmma}
    Consider the optimization process whose setting is given as \ref{subsec:grabLBab} with $\eta < \frac{1}{L}$.
    For any $0 \le t \le \frac{n}{2} - 1$ and any $k \in \{ 1, \cdots, K\}$, if $y_{2t}^k + z_{2t}^k \ge 0$ holds, then
    \begin{align*}
        y_{2t+2}^k + z_{2t+2}^k \ge \bigopen{1 - \frac{\eta L}{2}}\bigopen{1 - \eta L}\bigopen{y_{2t}^k + z_{2t}^k} + \frac{\eta^2 L \nu}{2}
    \end{align*}
    holds regardless of which functions are used at the $(2t+1)$-th and the $(2t+2)$-th iterations of the $k$-th epoch.
    Consequently, if $y_0^k + z_0^k \ge \frac{\eta \nu}{3 - \eta L}$, $y_n^{k} + z_n^{k} \ge \frac{\eta \nu}{3 - \eta L}$ holds regardless of the permutation $\sigma_k$.
\end{restatable}
The proof of the lemma is in \cref{subsec:grabLBad}. In our setting, $y_0^1 + z_0^1 = \frac{\nu}{2L} \ge \frac{\eta \nu}{3 - \eta L}$ since $\eta < \frac{1}{L}$. Thus, we have $y_n^k + z_n^k \ge \frac{\eta \nu}{3 - \eta L}$ for every $k \in [K]$.

For $\hat{y} + \hat{z}$, we get
\begin{align*}
    \hat{y} + \hat{z}
    &= \frac{\sum_{k=1}^{K+1} \alpha_k \bigopen{y_0^k + z_0^k}}{\sum_{k=1}^{K+1} \alpha_k} \ge \frac{\eta \nu}{3 - \eta L} \cdot \frac{\sum_{k=1}^{K+1} \alpha_k}{\sum_{k=1}^{K+1} \alpha_k}\\
    &= \frac{\eta \nu}{3 - \eta L} > \frac{\eta \nu}{3} = \Omega \bigopen{\frac{\nu}{\mu n K}},
\end{align*}
and using the inequality $2(a^2 + b^2) \ge (a + b)^2$, 
\begin{align*}
    F_2(\hat{y}, \hat{z}) - F_2^*
    &= \left( L \mathbbm{1}_{y<0} + \frac{L}{2} \mathbbm{1}_{y\ge0} \right) \frac{\hat{y}^2}{2} + \left( L \mathbbm{1}_{z<0} + \frac{L}{2} \mathbbm{1}_{z\ge0} \right) \frac{\hat{z}^2}{2}\\
    &\ge \frac{L}{4} \bigopen{\hat{y}^2 + \hat{z}^2 }\\
    &\ge \frac{L}{8} \left(\hat{y} + \hat{z}\right)^2\\
    &= \Omega \left(\frac{L \nu^2}{\mu^2 n^2K^2} \right).
\end{align*}
\end{proof}

\subsection{\texorpdfstring{Lower Bound for $\eta > \frac{1}{L}$}{Lower Bound for η > 1/L}} \label{subsec:grabLBac}
Here we show that there exists $F_3 (w) \in \gF (2L, \mu, 0, \nu)$ such that any permutation-based SGD with $w_0^1 = \frac{\nu}{\mu n K}$ satisfies
\begin{align*}
    F_3(\hat{w}) - F_3^* = \Omega \left( \frac{L \nu^2}{\mu^2 n^2 K^2} \right).
\end{align*}

\begin{proof}
We define $F_3 (w) \in \gF (2L, 2L, 0, 0)$ by the following components:
\begin{align*}
    f_i(w) = F_3(w) = Lw^2.
\end{align*}
Note that $\gF (2L, 2L, 0, 0) \subseteq \gF (2L, \mu, 0, \nu)$ and $F_3^* = 0$ at $w^* = 0$ by definition.

In this regime, we will see that the step size is so large that $\bigset{w_n^k}_{k=1}^{K}$ diverges.
We start from $w_0^1 = \frac{\nu}{\mu n K}$. Since the gradient of all component functions at point $w$ is fixed deterministically to $2 L w$, we have for every $k \in [K]$,
\begin{align*}
    w_n^k = \bigopen{1 - 2 \eta L}^{nk} w_0^1 \ge 1^{nk} \frac{\nu}{\mu n K} &= \Omega \bigopen{\frac{\nu}{\mu n K}},
\end{align*}
where we used the fact that $n$ is even in the second step.
Thus, regardless of the permutation-based SGD we use, we have $\hat{w} = \Omega \bigopen{\frac{\nu}{\mu n K}}$ and $F_3(\hat{w}) - F_3^* = L \hat{w}^2 = \Omega \bigopen{\frac{L \nu^2}{\mu^2 n^2 K^2}}$.
\end{proof}

\subsection{\texorpdfstring{Lemmas used in \thmref{thm:grabLBa}}{Lemmas used in Theorem 7}} \label{subsec:grabLBad}
In this subsection, we will prove the lemmas used in \cref{thm:grabLBa}.
\grabUBalmma*
\begin{proof}
Without loss of generality, assume $y_{2t}^k \ge z_{2t}^k$. 
Since we assumed $y_{2t}^k + z_{2t}^k$ is nonnegative, $y_{2t}^k \ge 0$ holds. 
Depending on which function is used at $(2t+1)$-th iteration of $k$-th epoch, we consider following two cases:
\begin{enumerate}[label=(\alph*)]
\item $y_{2t+1}^k = y_{2t}^k - \eta \nabla g_{-1}(y_{2t}^k)$ and $z_{2t+1}^k = z_{2t}^k - \eta \nabla g_{+1}(z_{2t}^k)$, 
\item $y_{2t+1}^k = y_{2t}^k - \eta \nabla g_{+1}(y_{2t}^k)$ and $z_{2t+1}^k = z_{2t}^k - \eta \nabla g_{-1}(z_{2t}^k)$.
\end{enumerate}
Note that $y_{2t+2}^k + z_{2t+2}^k$ is independent of which function is used at $(2t+2)$-th iteration of the $k$-th epoch, because $y_{2t+2}^k = \bigopen{1 - \eta \bigopen{L \mathbbm{1}_{y_{2t+1}^k < 0} + \frac{L}{2} \mathbbm{1}_{y_{2t+1}^k \ge 0} }}y_{2t+1}^k \pm \eta \nu$ and $z_{2t+2}^k = \bigopen{1 - \eta \bigopen{L \mathbbm{1}_{z_{2t+1}^k < 0} + \frac{L}{2} \mathbbm{1}_{z_{2t+1}^k \ge 0} }}z_{2t+1}^k \mp \eta \nu$ so that summation of $y$ and $z$ results in the canceling of $\eta \nu$ terms.
\paragraph{Case (a)}
For Case (a), $y_{2t+1}^k = (1 - \frac{\eta L}{2}) y_{2t}^k + \eta \nu > 0$ holds since $y_{2t}^k \ge 0$, but the signs of $z_{2t}^k$ and $z_{2t+1}^k$ are undetermined.
Thereby, we split the cases by the signs of $z_{2t}^k$ and $z_{2t+1}^k$.

First, assume $z_{2t}^k < 0$ and $z_{2t+1}^k < 0$. In this setting, $y_{2t}^k \ge 0$, $y_{2t+1}^k \ge 0$, $z_{2t}^k < 0$, $z_{2t+1}^k < 0$. Then,
\begin{align*}
y_{2t+2}^k + z_{2t+2}^k
&=\left(1 - \frac{\eta L}{2}\right) y_{2t+1}^k + (1 - \eta L) z_{2t+1}^k\\
&=\left(1 - \frac{\eta L}{2}\right) \left( \left(1 - \frac{\eta L}{2}\right) y_{2t}^k + \eta \nu \right) + (1 - \eta L) \left( (1 - \eta L) z_{2t}^k - \eta \nu \right)\\
&=\left(1 - \frac{\eta L}{2}\right)^2 y_{2t}^k + (1 - \eta L)^2 z_{2t}^k + \frac{\eta^2 L \nu}{2}\\
&=\left(1 - \frac{\eta L}{2}\right)(1 - \eta L)(y_{2t}^k + z_{2t}^k) + \frac{\eta^2 L \nu}{2}
+ \frac{\eta L}{2} \left( \left(1 - \frac{\eta L}{2}\right)y_{2t}^k - (1 - \eta L)z_{2t}^k \right)\\
&\ge \left(1 - \frac{\eta L}{2}\right)(1 - \eta L)(y_{2t}^k + z_{2t}^k) + \frac{\eta^2 L \nu}{2},
\end{align*}
where the last inequality holds because $y_{2t}^k \ge 0$ and $z_{2t}^k < 0$.

Next, assume $z_{2t}^k \ge 0$ and $z_{2t+1}^k < 0$. In this setting, $y_{2t}^k \ge 0$, $y_{2t+1}^k \ge 0$, $z_{2t}^k \ge 0$, $z_{2t+1}^k < 0$. Similarly,
\begin{align*}
y_{2t+2}^k + z_{2t+2}^k
&=\left(1 - \frac{\eta L}{2}\right) y_{2t+1}^k + (1 - \eta L) z_{2t+1}^k\\
&=\left(1 - \frac{\eta L}{2}\right) \left( \left(1 - \frac{\eta L}{2}\right) y_{2t}^k + \eta \nu \right) + (1 - \eta L) \left( \left(1 - \frac{\eta L}{2}\right)z_{2t}^k - \eta \nu \right)\\
&=\left(1 - \frac{\eta L}{2}\right)^2 y_{2t}^k + (1 - \eta L)\left(1 - \frac{\eta L}{2}\right) z_{2t}^k + \frac{\eta^2 L \nu}{2}\\
&=\left(1 - \frac{\eta L}{2}\right)(1 - \eta L)(y_{2t}^k + z_{2t}^k) + \frac{\eta^2 L \nu}{2} + \frac{\eta L}{2} \left(1 - \frac{\eta L}{2}\right) y_{2t}^k\\
&\ge \left(1 - \frac{\eta L}{2}\right)(1 - \eta L)(y_{2t}^k + z_{2t}^k) + \frac{\eta^2 L \nu}{2}.
\end{align*}

Finally, assume both $z_{2t}^k \ge 0$ and $z_{2t+1}^k \ge 0$. In this setting, $y_{2t}^k \ge 0$, $y_{2t+1}^k \ge 0$, $z_{2t}^k \ge 0$, $z_{2t+1}^k \ge 0$. Since $0 \le z_{2t+1}^k = \left(1 - \frac{\eta L}{2}\right)z_{2t}^k - \eta \nu$, $z_{2t}^k \ge \frac{\eta \nu}{1 - \eta L / 2}$ holds. Then,
\begin{align*}
y_{2t+2}^k + z_{2t+2}^k
&=\left(1 - \frac{\eta L}{2}\right) y_{2t+1}^k + \left(1 - \frac{\eta L}{2}\right) z_{2t+1}^k\\
&=\left(1 - \frac{\eta L}{2}\right) \left( \left(1 - \frac{\eta L}{2}\right) y_{2t}^k + \eta \nu \right) + \left(1 - \frac{\eta L}{2}\right) \left( \left(1 - \frac{\eta L}{2}\right) z_{2t}^k - \eta \nu \right)\\
&=\left(1 - \frac{\eta L}{2}\right)^2 y_{2t}^k + \left(1 - \frac{\eta L}{2}\right)^2 z_{2t}^k\\
&=\left(1 - \frac{\eta L}{2}\right)(1 - \eta L)(y_{2t}^k + z_{2t}^k) + \frac{\eta^2 L \nu}{2} 
+ \frac{\eta L}{2} \left(1 - \frac{\eta L}{2}\right) (y_{2t}^k + z_{2t}^k) - \frac{\eta^2 L \nu}{2}\\
&\ge \left(1 - \frac{\eta L}{2}\right)(1 - \eta L)(y_{2t}^k + z_{2t}^k) + \frac{\eta^2 L \nu}{2}.
\end{align*}
In the last inequality, we used the fact that $y_{2t}^k \ge 0$ and $z_{2t}^k \ge \frac{\eta \nu}{1 - \eta L /2}$.

We don't have to consider the case when $z_{2t}^k < 0$ and $z_{2t+1}^k \ge 0$ hold, because $z_{2t+1}^k = (1 - \eta L)z_{2t}^k - \eta \nu < 0$ if $z_{2t}^k < 0$. 
Thus, we have proven the first inequality of the lemma for Case (a).

\paragraph{Case (b)}
For Case (b), we consider three cases depending on the signs of $y_{2t+1}^k$ and $z_{2t+1}^k$.

First, assume $y_{2t+1}^k \ge 0$ and $z_{2t+1}^k < 0$ hold. 
In this case, if $z_{2t}^k \ge 0$, then $z_{2t+1}^k = \left(1 - \frac{\eta L}{2}\right)z_{2t}^k + \eta \nu > 0$; therefore, $z_{2t}^k < 0$ should hold. So in this setting, we have $y_{2t}^k \ge 0$, $y_{2t+1}^k \ge 0$, $z_{2t}^k < 0$, and $z_{2t+1}^k < 0$. Using this fact,
\begin{align*}
y_{2t+2}^k + z_{2t+2}^k
&=\left(1 - \frac{\eta L}{2}\right) y_{2t+1}^k + (1 - \eta L) z_{2t+1}^k\\
&=\left(1 - \frac{\eta L}{2}\right) \left( \left(1 - \frac{\eta L}{2}\right) y_{2t}^k - \eta \nu \right) + (1 - \eta L) \left( (1 - \eta L) z_{2t}^k + \eta \nu \right)\\
&=\left(1 - \frac{\eta L}{2}\right)^2 y_{2t}^k + (1 - \eta L)^2 z_{2t}^k - \frac{\eta^2 L \nu}{2}\\
&=\left(1 - \frac{\eta L}{2}\right)(1 - \eta L)(y_{2t}^k + z_{2t}^k) + \frac{\eta^2 L \nu}{2} + \frac{\eta L}{2} \left( \left(1 - \frac{\eta L}{2}\right)y_{2t}^k - (1 - \eta L)z_{2t}^k - 2\eta \nu\right)\\
&=\left(1 - \frac{\eta L}{2}\right)(1 - \eta L)(y_{2t}^k + z_{2t}^k) + \frac{\eta^2 L \nu}{2} + \frac{\eta L}{2} \left( y_{2t+1}^k - z_{2t+1}^k \right)\\
&\ge \left(1 - \frac{\eta L}{2}\right)(1 - \eta L)(y_{2t}^k + z_{2t}^k) + \frac{\eta^2 L \nu}{2}.
\end{align*}

Second, assume $y_{2t+1}^k < 0$ and $z_{2t+1}^k \ge 0$. If $z_{2t}^k < 0$, the setting becomes $y_{2t}^k \ge 0$, $y_{2t+1}^k < 0$, $z_{2t}^k < 0$, $z_{2t+1}^k \ge 0$ so that
\begin{align*}
y_{2t+2}^k + z_{2t+2}^k
&=(1 - \eta L) y_{2t+1}^k + \left(1 - \frac{\eta L}{2}\right) z_{2t+1}^k\\
&=(1 - \eta L) \left( \left(1 - \frac{\eta L}{2}\right) y_{2t}^k - \eta \nu \right) + \left(1 - \frac{\eta L}{2}\right) \left( (1 - \eta L) z_{2t}^k + \eta \nu \right)\\
&= (1 - \eta L)\left(1 - \frac{\eta L}{2}\right)(y_{2t}^k + z_{2t}^k) + \frac{\eta^2 L \nu}{2}.
\end{align*}
If $z_{2t}^k \ge 0$, the setting becomes $y_{2t}^k \ge 0$, $y_{2t+1}^k < 0$, $z_{2t}^k \ge 0$, $z_{2t+1}^k \ge 0$ so that
\begin{align*}
y_{2t+2}^k + z_{2t+2}^k
&=(1 - \eta L) y_{2t+1}^k + \left(1 - \frac{\eta L}{2}\right) z_{2t+1}^k\\
&=(1 - \eta L) \left( \left(1 - \frac{\eta L}{2}\right) y_{2t}^k - \eta \nu \right) + \left(1 - \frac{\eta L}{2}\right) \left( \left(1 - \frac{\eta L}{2}\right) z_{2t}^k + \eta \nu \right)\\
&=(1 - \eta L)\left(1 - \frac{\eta L}{2}\right)(y_{2t}^k + z_{2t}^k) + \frac{\eta^2 L \nu}{2} + \frac{\eta L}{2}\left(1 - \frac{\eta L}{2}\right)z_{2t}^k\\
&\ge (1 - \eta L)\left(1 - \frac{\eta L}{2}\right)(y_{2t}^k + z_{2t}^k) + \frac{\eta^2 L \nu}{2}.
\end{align*}

Lastly, assume $y_{2t+1}^k \ge 0$ and $z_{2t+1}^k \ge 0$. If $z_{2t}^k < 0$, the setting becomes $y_{2t}^k \ge 0$, $y_{2t+1}^k \ge 0$, $z_{2t}^k < 0$, $z_{2t+1}^k \ge 0$ so that
\begin{align*}
y_{2t+2}^k + z_{2t+2}^k
&=\left(1 - \frac{\eta L}{2}\right) y_{2t+1}^k + \left(1 - \frac{\eta L}{2}\right) z_{2t+1}^k\\
&=\left(1 - \frac{\eta L}{2}\right) \left( \left(1 - \frac{\eta L}{2}\right) y_{2t}^k - \eta \nu \right) + \left(1 - \frac{\eta L}{2}\right) \left( (1 - \eta L) z_{2t}^k + \eta \nu \right)\\
&= (1 - \eta L)\left(1 - \frac{\eta L}{2}\right)(y_{2t}^k + z_{2t}^k) + \frac{\eta^2 L \nu}{2} + \frac{\eta L}{2}\left(\left(1 - \frac{\eta L}{2}\right)y_{2t}^k - \eta \nu\right)\\
&= (1 - \eta L)\left(1 - \frac{\eta L}{2}\right)(y_{2t}^k + z_{2t}^k) + \frac{\eta^2 L \nu}{2} + \frac{\eta L}{2}y_{2t+1}^k\\
&\ge (1 - \eta L)\left(1 - \frac{\eta L}{2}\right)(y_{2t}^k + z_{2t}^k) + \frac{\eta^2 L \nu}{2}.
\end{align*}
If $z_{2t}^k \ge 0$, the setting becomes $y_{2t}^k \ge 0$, $y_{2t+1}^k \ge 0$, $z_{2t}^k \ge 0$, $z_{2t+1}^k \ge 0$ so that
\begin{align*}
y_{2t+2}^k + z_{2t+2}^k
&=\left(1 - \frac{\eta L}{2}\right) y_{2t+1}^k + \left(1 - \frac{\eta L}{2}\right) z_{2t+1}^k\\
&=\left(1 - \frac{\eta L}{2}\right) \left( \left(1 - \frac{\eta L}{2}\right) y_{2t}^k - \eta \nu \right) + \left(1 - \frac{\eta L}{2}\right) \left( \left(1 - \frac{\eta L}{2}\right) z_{2t}^k + \eta \nu \right)\\
&= (1 - \eta L)\left(1 - \frac{\eta L}{2}\right)(y_{2t}^k + z_{2t}^k) + \frac{\eta^2 L \nu}{2}+ \frac{\eta L}{2} \left(\left(1 - \frac{\eta L}{2}\right)(y_{2t}^k + z_{2t}^k) - \eta \nu \right)\\
&= (1 - \eta L)\left(1 - \frac{\eta L}{2}\right)(y_{2t}^k + z_{2t}^k) + \frac{\eta^2 L \nu}{2}+ \frac{\eta L}{2} \left(y_{2t+1}^k + \left(1 - \frac{\eta L}{2}\right)z_{2t}^k \right)\\
&\ge (1 - \eta L)\left(1 - \frac{\eta L}{2}\right)(y_{2t}^k + z_{2t}^k) + \frac{\eta^2 L \nu}{2}.
\end{align*}

We do not have to consider the case when $y_{2t+1}^k$ and $z_{2t+1}^k$ are both less than $0$, because $y_{2t+1}^k + z_{2t+1}^k \geq 0$ always holds. This can be shown by case analysis on the sign of $z_{2t}^k$: if $z_{2t}^k \ge 0$, then
\begin{align*}
y_{2t+1}^k + z_{2t+1}^k
&= \left(1 - \frac{\eta L}{2}\right)y_{2t}^k - \eta \nu + \left(1 - \frac{\eta L}{2}\right)z_{2t}^k + \eta \nu\\
&= \left(1 - \frac{\eta L}{2}\right)(y_{2t}^k + z_{2t}^k) \ge 0,
\end{align*}
and if $z_{2t}^k < 0$, then
\begin{align*}
y_{2t+1}^k + z_{2t+1}^k
&= \left(1 - \frac{\eta L}{2}\right)y_{2t}^k - \eta \nu + (1 - \eta L)z_{2t}^k + \eta \nu\\
&= (1 - \eta L)(y_{2t}^k + z_{2t}^k) + \frac{\eta L}{2} y_{2t}^k \ge 0.
\end{align*}
Therefore, we have proven the first inequality of the lemma for Case (b).

Putting the results of Case (a) and (b) together, we have
\begin{align*}
    y_{2t+2}^k + z_{2t+2}^k \ge \bigopen{1 - \frac{\eta L}{2}}\bigopen{1 - \eta L}\bigopen{y_{2t}^k + z_{2t}^k} + \frac{\eta^2 L \nu}{2}
\end{align*}
for any $0 \le t \le \frac{n}{2} - 1$ and any $k \in \{1, \cdots, K\}$, proving the first part of the lemma.

It now remains to prove the second part, namely that 
\begin{align*}
    y_n^k + z_n^k \ge \frac{\eta \nu}{3 - \eta L}
\end{align*}
holds if $y_0^k + z_0^k \ge \frac{\eta \nu}{3 - \eta L}$. From the first part of the lemma, we can see that the updates over a single epoch can be bounded as
\begin{align*}
    y_n^k + z_n^k
    &\ge \bigopen{1 - \frac{\eta L}{2}}\bigopen{1 - \eta L}\bigopen{y_{n-2}^k + z_{n-2}^k} + \frac{\eta^2 L \nu}{2}\\
    &\quad \vdots\\
    &\ge \bigopen{1 - \frac{\eta L}{2}}^{\frac{n}{2}} \bigopen{1 - \eta L}^{\frac{n}{2}} \bigopen{y_0^k + z_0^k} + \frac{\eta^2 L \nu}{2}
        \cdot \sum_{i=0}^{\frac{n}{2}-1} \bigopen{1 - \frac{\eta L}{2}}^{i} \bigopen{1 - \eta L}^{i}\\
    &= \bigopen{1 - \frac{\eta L}{2}}^{\frac{n}{2}} \bigopen{1 - \eta L}^{\frac{n}{2}} \bigopen{y_0^k + z_0^k} + \frac{\eta^2 L \nu}{2} \cdot \frac{1 - \left(1 - \frac{\eta L}{2}\right)^{\frac{n}{2}}(1 - \eta L)^{\frac{n}{2}}}{1 - \left(1 - \frac{\eta L}{2}\right)(1 - \eta L)}\\
    &= \bigopen{1 - \frac{\eta L}{2}}^{\frac{n}{2}} \bigopen{1 - \eta L}^{\frac{n}{2}} \bigopen{y_0^k + z_0^k} + \frac{\eta \nu}{3 - \eta L} \bigopen{1 - \left(1 - \frac{\eta L}{2}\right)^{\frac{n}{2}}(1 - \eta L)^{\frac{n}{2}}}\\
    &= \frac{\eta \nu}{3 - \eta L} + \bigopen{1 - \frac{\eta L}{2}}^{\frac{n}{2}} \bigopen{1 - \eta L}^{\frac{n}{2}} \bigopen{y_0^k + z_0^k - \frac{\eta \nu}{3 - \eta L}}\\
    &\ge \frac{\eta \nu}{3 - \eta L} .
\end{align*}
This ends the proof of the lemma.
\end{proof}

\begin{lemma} \label{lmm:thm1_2}
    For any $t \ge 2$, the following inequality holds:
    \begin{align*}
        \bigopen{1 - \frac{1}{t}}^t > \frac{1}{e} \bigopen{1 - \frac{1}{t}}.
    \end{align*}
\end{lemma}
\begin{proof}
    \begin{align*}
        \bigopen{1 + \frac{1}{t-1}}^{t-1} < e &\iff \bigopen{\frac{t}{t-1}}^{t-1} < e\\
        &\iff \bigopen{\frac{t-1}{t}}^{t-1} > \frac{1}{e} \iff \bigopen{1 - \frac{1}{t}}^t > \frac{1}{e} \bigopen{1 - \frac{1}{t}}.
    \end{align*}
\end{proof}
\section{\texorpdfstring{Proof of \thmref{thm:grabLBb}}{Proof of Theorem 9}} \label{sec:grabLBb}
Here we prove \cref{thm:grabLBb}, restated below for the sake of readability.
\grabLBb*
\begin{proof}
Similarly as in \cref{app:xhat}, we define objective functions for four step-size regimes and aggregate the functions to obtain the final lower bound. Here we also assume $n$ is even, where we can easily extend to odd $n$'s by the same reasoning as in \cref{app:xhat}.

We will prove the following lower bounds for each regime. Here $F_j^*$ is the minimizer of $F_j$ for $j = 1, 2, 3, 4$. 

\begin{itemize}
    \item If $\eta \in \left( 0, \frac{1}{2 \mu n K} \right)$, there exists a $1$-dimensional objective function $F_1(x) \in \gF_{\text{PŁ}}(L,\mu,0,\nu)$ such that any permutation-based SGD with initialization $x_0^1 = \frac{L \nu}{\mu^2 n K}$ satisfies
    \begin{align*}
        F_1(\hat{x}) - F_1^* = \Omega \left( \frac{L^2 \nu^2}{\mu^3 n^2 K^2} \right).
    \end{align*}
    \item If $\eta \in \left[\frac{1}{2 \mu n K}, \frac{2}{nL}\right]$, there exists a $1$-dimensional objective function $F_2(y) \in \gF_{\text{PŁ}}\bigopen{L,\mu,\frac{L}{\mu},\nu}$ such that any permutation-based SGD with initialization $y_0^1 = \frac{\nu}{60L}$ satisfies
    \begin{align*}
        F_2(\hat{y}) - F_2^* = \Omega \left( \frac{L^2 \nu^2}{\mu^3 n^2 K^2} \right).
    \end{align*}
    \item If $\eta \in \left[\frac{2}{n L}, \frac{1}{L}\right]$, there exists a $1$-dimensional objective function $F_3(z) \in \gF_{\text{PŁ}}\bigopen{L,\mu,\frac{L}{\mu},\nu}$ such that any permutation-based SGD with initialization $z_0^1 = \frac{3\nu}{8nL}$ satisfies
    \begin{align*}
        F_3(\hat{z}) - F_3^* = \Omega \left( \frac{L^2 \nu^2}{\mu^3 n^2 K^2} \right).
    \end{align*}
    \item If $\eta > \frac{1}{L}$, there exists a $1$-dimensional objective function $F_4(w) \in \gF_{\text{PŁ}}(2L,\mu,0,\nu)$ such that any permutation-based SGD with initialization $w_0^1 = \frac{L^{1/2} \nu}{\mu^{3/2} n K}$ satisfies
    \begin{align*}
        F_4(\hat{w}) - F_4^* = \Omega \left( \frac{L^2 \nu^2}{\mu^3 n^2 K^2} \right).
    \end{align*}
\end{itemize}

Now we define the 4-dimensional function $F(\vx) = F(x, y, z, w) = F_1(x) + F_2(y) + F_3(z) + F_4(w)$, where $F_1$, $F_2$, $F_3$, and $F_4$ are chosen to satisfy the above lower bounds.

\note{In fact, our constructed $F_1$, $F_2$, $F_3$, and $F_4$ are strongly convex; however, for simplicity of exposition, we stated as a member of a larger class $\gF_{\text{PŁ}}$.}
Following the analyses in \cref{app:xhat}, $F$ is a $2L$-\textit{smooth} and $\mu$-\textit{strongly convex} function.

Also, if four functions $H_1$, $H_2$, $H_3$, and $H_4$ (each with $n$ components $h_{1, i}$, $h_{2, i}$, $h_{3, i}$, and $h_{4, i}$) satisfies Assumption \ref{ass:bg} for $\tau = \tau_0$ and $\nu = \nu_0$, then $H(\vx) = H_1(x) + H_2(y) + H_3(z) + H_4(w)$ satisfies
\begin{align*}
    \bignorm{\nabla h_i (\vx) - \nabla H(\vx)}^2
    &= \bignorm{\nabla h_{1, i} (x) - \nabla H_1(x)}^2 + \bignorm{\nabla h_{2, i} (y) - \nabla H_2(y)}^2\\
    &\,\,\,\,\,\,+ \bignorm{\nabla h_{3, i} (z) - \nabla H_3(z)}^2 + \bignorm{\nabla h_{4, i} (w) - \nabla H_4(w)}^2\\
    &\le \bigopen{\tau_0 \bignorm{H_1(x)} + \nu_0}^2 + \bigopen{\tau_0 \bignorm{H_2(y)} + \nu_0}^2 + \bigopen{\tau_0 \bignorm{H_3(z)} + \nu_0}^2 + \bigopen{\tau_0 \bignorm{H_4(w)} + \nu_0}^2\\
    &\le (2\tau_0^2 \bignorm{H_1(x)}^2 + 2\nu_0^2) + (2\tau_0^2 \bignorm{H_2(y)}^2 + 2\nu_0^2)\\
    &\,\,\,\,\,\,+ (2\tau_0^2 \bignorm{H_3(z)}^2 + 2\nu_0^2) + (2\tau_0^2 \bignorm{H_4(w)}^2 + 2\nu_0^2)\\
    &\le 2\tau_0^2 \bigopen{\bignorm{H_1(x)}^2 +\bignorm{H_2(y)}^2 + \bignorm{H_3(z)}^2 + \bignorm{H_4(w)}^2} + 8\nu_0^2\\
    &= 2\tau_0^2 \bignorm{H(\vx)}^2 + 8\nu_0^2 < \bigopen{2\tau_0 \bignorm{H(\vx)} + 3\nu_0}^2.
\end{align*}
for all $i = 1, \dots, n$, i.e., $H(\vx)$ satisfies Assumption \ref{ass:bg} for $\tau = 2\tau_0$ and $\nu = 3\nu_0$. Combining these results, we obtain $F \in \gF_{\text{PŁ}} \bigopen{2L,\mu,\frac{2L}{\mu},3\nu}$, which allows us to directly apply the convergence rates in the lower bounds of $F_1$, $F_2$, $F_3$, and $F_4$ to the aggregated function $F$.

Note that we assumed $\kappa \ge 8n$ and our constructed function is $2L$-\textit{smooth} and $\mu$-\textit{strongly convex}.
Thus, $\kappa \ge 8n$ is equivalent to $ \frac{L}{\mu} \ge 4n$ throughout the proof.
\note{Also, combining $K \ge \max \left \{\frac{\kappa^2}{n}, \kappa^{3/2}n^{1/2} \right \}$ and $\kappa \ge 8n$, we have $K \ge \frac{\kappa^2}{n} \ge \kappa$ and thus $\frac{1}{2\mu n K} \le \frac{2}{nL}$ holds so all step size regimes are valid.}

Finally, rescaling $L$ and $\nu$ will give us the function $F \in \gF_{\text{PŁ}} \bigopen{L, \mu, \frac{L}{\mu}, \nu}$ satisfying $F(\hat{\vx}) - F^* = \Omega \bigopen{ \frac{L^2 \nu^2}{\mu^3 n^2 K^2} }$.
\end{proof}

For the following subsections, we prove the lower bounds for $F_1$, $F_2$, $F_3$, and $F_4$ at the corresponding step size regime.

\subsection{\texorpdfstring{Lower Bound for $\eta \in \left( 0, \frac{1}{2 \mu n K} \right)$}{Lower Bound for η ∈ (0, 1/2μnK)}} \label{subsec:grabLBba}
Here we show that there exists $F_1 (x) \in \gF_{\text{PŁ}} (L, \mu, 0, \nu)$ such that any permutation-based SGD with $x_0^1 = \frac{L \nu}{\mu^2 n K}$ satisfies
\begin{align*}
    F_1(\hat{x}) - F_1^* = \Omega \left( \frac{L^2 \nu^2}{\mu^3 n^2 K^2} \right).
\end{align*}

\begin{proof}
We define $F_1 (x) \in \gF_{\text{PŁ}} (\mu, \mu, 0, 0)$ by the following components:
\begin{align*}
    f_i(x) = F_1(x) = \frac{\mu}{2}x^2.
\end{align*}
Note that $\gF_{\text{PŁ}} (\mu, \mu, 0, 0) \subseteq \gF_{\text{PŁ}} (L, \mu, 0, \nu)$ and $F_1^* = 0$ at $x^* = 0$ by definition.

In this regime, we will see that the step size is too small so that $\bigset{x_n^k}_{k=1}^{K}$ cannot even reach near the optimal point.
We start from $x_0^1 = \frac{L \nu}{\mu^2 n K}$. Since the gradient of all component functions evaluated at point $x$ is fixed deterministically to $\mu x$, regardless of the permutation-based SGD algorithm we use, we have
\begin{align*}
    x_n^k = x_0^1 (1 - \eta \mu)^{nk}
    &\ge \frac{L \nu}{\mu^2 n K}\left(1 - \frac{1}{2nK} \right)^{nk} \ge \frac{L \nu}{\mu^2 n K}\left(1 - \frac{1}{2nK} \right)^{nK} \\
    &> \frac{L \nu}{\mu^2 n K}\left(1 - \frac{1}{nK} \right)^{nK} \overset{(a)} > \frac{L \nu}{\mu^2 n K} \frac{1}{e} \left(1 - \frac{1}{nK} \right) \overset{(b)} \ge \frac{L \nu}{\mu^2 n K} \frac{1}{2e}.
\end{align*}
where (a) comes from \lmmref{lmm:thm1_2} and (b) comes from assumption that $n \ge 2$. Therefore, we have $\hat{x} = \Omega \bigopen{\frac{L \nu}{\mu^2 n K}}$ for any nonnegative weights $\bigset{\alpha_k}_{k=1}^{K+1}$. With this $\hat{x}$, we have
\begin{align*}
    F_1(\hat{x}) - F_1^* = \frac{\mu}{2} \hat{x}^2 = \Omega \bigopen{\frac{L^2 \nu^2}{\mu^3 n^2 K^2}}.
\end{align*}
\end{proof}

\subsection{\texorpdfstring{Lower Bound for $\eta \in \left[\frac{1}{2 \mu n K}, \frac{2}{nL}\right]$}{Lower Bound for η ∈ [1/2μnK, 2/nL]}} \label{subsec:grabLBbb}
Here we show that there exists $F_2 (y) \in \gF_{\text{PŁ}} \bigopen{L, \mu, \frac{L}{\mu}, \nu}$ such that any permutation-based SGD with $y_0^1 = \frac{\nu}{60L}$ satisfies
\begin{align*}
    F_2(\hat{y}) - F_2^* = \Omega \left( \frac{L^2 \nu^2}{\mu^3 n^2K^2} \right).
\end{align*}

\begin{proof}
We define $F_2(y)$ by the following components:
\begin{align*}
    f_i(y) =
    \begin{cases}
        g_1(y) & \text{ if} \,\,\, i \le \frac{n}{2},\\
        g_2(y) & \text{ otherwise},
    \end{cases}
\end{align*}
where
\begin{align*}
    &g_1(y) = \frac{L}{2}y^2 - \nu y,\\
    &g_2(y) = -\frac{L}{2} \bigopen{1 - \frac{2\mu}{L}}y^2 + \nu y.
\end{align*}
With this construction, the finite-sum objective becomes
\begin{align*}
    F_2(y) = \frac{1}{n} \sum_{i=1}^n f_i(y) = \frac{\mu}{2} y^2.
\end{align*}
Note that all the components are $L$-\textit{smooth} and $F$ is $\mu$-\textit{strongly convex}. Moreover,
\begin{align*}
    \bignorm{\nabla f_1(y) - \nabla F_2(y)}
    &= \bignorm{(Ly - \nu) - \mu y} \le \bignorm{(L-\mu)y} + \nu\\
    &\le L \bignorm{y} + \nu = \frac{L}{\mu} \bignorm{\nabla F_2(y)} + \nu,\\
    \bignorm{\nabla f_{\frac{n}{2}+1}(y) - \nabla F_2(y)}
    &= \bignorm{\bigopen{-L\bigopen{1 - \frac{2\mu}{L}}y + \nu} - \mu y} \le \bignorm{(L - \mu)y} + \nu\\
    &\le L \bignorm{y} + \nu = \frac{L}{\mu} \bignorm{\nabla F_2(y)} + \nu,
\end{align*}
and thereby $F_2 \in \gF_{\text{PŁ}} \bigopen{L, \mu, \frac{L}{\mu}, \nu}$. Also, we can easily verify $F_2^* = 0$ at $y^* = 0$.

To simplify notation, we will write $\nabla f_i(y) = a_i y - b_i$ temporarily. Then, $a_i \in \bigset{L, -L\bigopen{1 - \frac{2\mu}{L}}}$, $b_i \in \bigset{\nu, -\nu}$ holds and we can write $y_n^k$ as
\begin{align}
    y_1^k = y_0^k - \eta \nabla f_{\sigma_k(1)} \bigopen{y_0^k} &= \bigopen{1 - \eta a_{\sigma_k(1)}}y_0^k + \eta b_{\sigma_k(1)}, \nonumber\\
    y_2^k = y_1^k - \eta \nabla f_{\sigma_k(2)} \bigopen{y_1^k} &= \bigopen{1 - \eta a_{\sigma_k(2)}}y_1^k + \eta b_{\sigma_k(2)} \nonumber \\
    &= \bigopen{1 - \eta a_{\sigma_k(2)}}\bigopen{1 - \eta a_{\sigma_k(1)}}y_0^k + \eta b_{\sigma_k(2)} + \eta b_{\sigma_k(1)} \bigopen{1 - \eta a_{\sigma_k(2)}}, \nonumber\\
    \vdots \nonumber\\
    y_n^k = y_{n-1}^k - \eta \nabla f_{\sigma_k(n)} \bigopen{y_{n-1}^k} &= \prod_{i=1}^n \bigopen{1 - \eta a_{\sigma_k(i)}} y_0^k + \eta \sum_{i=1}^n b_{\sigma_k(i)} \prod_{j=i+1}^{n} \bigopen{1 - \eta a_{\sigma_k(j)}}. \label{eq:thm2_2_1}
\end{align}
Define $S := \prod_{i=1}^n \bigopen{1 - \eta a_{\sigma_k(i)}} = \prod_{i=1}^n \bigopen{1 - \eta a_{i}}$ and $A_{\sigma} := \eta \sum_{i=1}^n b_{\sigma(i)} \prod_{j=i+1}^{n} \bigopen{1 - \eta a_{\sigma(j)}}$. Then, we can write \cref{eq:thm2_2_1} as $y_n^k = S y_0^k + A_{\sigma}$. Note that $S$ is independent of the choice of $\sigma_k$ and $A_{\sigma}$ is the term that we can control using permutation-based SGD.

We now consider which permutation $\sigma$ minimizes $A_{\sigma}$. Choose an arbitrary $\sigma$ and assume there exists $t \in \{1, \cdots, n-1 \}$ such that $f_{\sigma(t)} = g_2 $ and $f_{\sigma(t+1)} = g_1$. Then, define another permutation $\sigma^\prime$ by $\sigma^\prime (t) = \sigma (t+1)$, $\sigma^\prime (t+1) = \sigma (t)$ and $\sigma^\prime (i)= \sigma (i)$ for $i \in \{1, \cdots, n \} \setminus \{t, t+1\}$.

Let $y_{\sigma}$ and $y_{\sigma^\prime}$ as the value of $y_n^k$ generated by $\sigma$ and $\sigma^\prime$ starting from the same $y_0^k$, respectively. Since $b_{\sigma(i)} \prod_{j=i+1}^{n} \bigopen{1 - \eta a_{\sigma(j)}} = b_{\sigma^\prime(i)} \prod_{j=i+1}^{n} \bigopen{1 - \eta a_{\sigma^\prime(j)}}$ for $i \in \{1, \cdots, n \} \setminus \{t, t+1\}$, we have
\begin{align}
    y_{\sigma} - y_{\sigma^\prime}
    &= \prod_{i=1}^n \bigopen{1 - \eta a_{\sigma(i)}} y_0^k + \eta \sum_{i=1}^n b_{\sigma(i)} \prod_{j=i+1}^{n} \bigopen{1 - \eta a_{\sigma(j)}} \nonumber \\
    &\,\,\,\,\,\, - \prod_{i=1}^n \bigopen{1 - \eta a_{\sigma^\prime(i)}} y_0^k + \eta \sum_{i=1}^n b_{\sigma^\prime(i)} \prod_{j=i+1}^{n} \bigopen{1 - \eta a_{\sigma^\prime(j)}} \nonumber \\
    &= \eta \biggl( b_{\sigma(t)} \prod_{j=t+1}^{n} \bigopen{1 - \eta a_{\sigma(j)}} + b_{\sigma(t+1)} \prod_{j=t+2}^{n} \bigopen{1 - \eta a_{\sigma(j)}} \nonumber \\
    &\,\,\,\,\,\,\,\,\,\,\,\quad - b_{\sigma^\prime(t)} \prod_{j=t+1}^{n} \bigopen{1 - \eta a_{\sigma^\prime(j)}} - b_{\sigma^\prime(t+1)} \prod_{j=t+2}^{n} \bigopen{1 - \eta a_{\sigma^\prime(j)}} \biggl) \nonumber \\
    &= \bigopen{\eta \prod_{j=t+2}^{n} \bigopen{1 - \eta a_{\sigma(j)}}} \cdot \bigopen{ -\nu \bigopen{1 - \eta L} + \nu - \nu \bigopen{1 + \eta L \bigopen{1 - \frac{2\mu}{L}}} - (-\nu)} \nonumber \\
    &= \bigopen{\eta \prod_{j=t+2}^{n} \bigopen{1 - \eta a_{\sigma(j)}}} \cdot 2 \eta \mu \nu > 0. \label{eq:thm2_2_2}
\end{align}
Thereby, we can conclude that the permutation $\sigma$ that minimizes $A_{\sigma}$ should satisfy $\sigma(i) \le n/2$ for $i \le n/2$ and $\sigma(i) > n/2$ for $i > n/2$, i.e., $f_{\sigma(i)} = g_1$ for $i \le n/2$ and $f_{\sigma(i)} = g_2$ for $i > n/2$. Let $\sigma^*$ denote such $\sigma$.

With this permutation $\sigma^*$, $A_{\sigma^*}$ becomes
\begin{align*}
    A_{\sigma^*} &=\eta \nu \cdot \left( 1 + \eta L \left(1 - \frac{2 \mu}{L} \right) \right)^\frac{n}{2}\, \sum_{i=0}^{\frac{n}{2}-1} (1 - \eta L)^i - \eta \nu \cdot \sum_{i=0}^{\frac{n}{2}-1} \left( 1 + \eta L \left(1 - \frac{2 \mu}{L} \right) \right)^i.
\end{align*}
Here, we introduce $\beta := 1 - \frac{2\mu}{L}$ and $m := \frac{n}{2}$ to simplify notation a bit. Note that $\beta \ge 1 - \frac{1}{m}$ holds since we assumed $\frac{L}{\mu} > n$.
Then $A_{\sigma^*}$ can be rearranged as 
\begin{align*}
    A_{\sigma^*} &=\eta \nu \cdot \left( 1 + \eta L \beta \right)^m \frac{1-(1-\eta L)^m}{\eta L} - \eta \nu \cdot \frac{(1+\eta L \beta)^m - 1}{\eta L \beta}\\
    &= \frac{\nu}{L \beta} \cdot \left( (1 + \eta L \beta)^m (\beta - 1) - \beta(1 + \eta L \beta)^m (1 - \eta L)^m + 1 \right).
\end{align*}
Using \lmmref{lmm:thm2_1} (substituting $\eta L$ to $x$), we have $\frac{\nu}{L \beta} \cdot \left( (1 + \eta L \beta)^m (\beta - 1) - \beta(1 + \eta L \beta)^m (1 - \eta L)^m + 1 \right) \ge \frac{\eta^2 m L \nu}{30}$. 

We now show a lower bound for $S$. 
\begin{align*}
    S
    &= \prod_{i=1}^n \bigopen{1 - \eta a_{i}} \\
    &= \bigopen{1 - \eta L}^m \bigopen{1 + \eta L \beta}^m \\
    &= (1 - \eta L (1 - \beta) - \eta^2 L^2 \beta)^m\\
    &= \left(1 - \eta L \cdot \frac{2\mu}{L} - \eta^2 L^2 \left(1 - \frac{2\mu}{L} \right) \right)^m\\
    &>(1 - 2 \eta \mu - \eta^2 L^2)^m\\
    &\ge \begin{cases}
    (1 - 4 \eta \mu)^m > 1 - 4 \eta m \mu, &\text{(if $\frac{1}{2 \mu n K} \le \eta < \frac{2\mu}{L^2}$)}\\
    (1 - 2 \eta^2 L^2)^m > 1 - 2 \eta^2 m L^2. &\text{(if $\frac{2\mu}{L^2} \le \eta \le \frac{2}{nL}$)}
    \end{cases}
\end{align*}
We start at $y_0^1 = \frac{\nu}{60L}$.
Being aware of $\kappa = \frac{2L}{\mu}$ in the construction, we first verify that
\begin{align*}
    \frac{\nu}{60L}
    = \frac{\nu}{60L} \cdot \frac{K}{K} \ge \frac{\nu}{60L} \cdot \frac{\kappa^2}{nK}
    = \frac{4L\nu}{60 \mu^2} \cdot \frac{1}{nK} > \frac{L \nu}{240 \mu^2 n K}.
\end{align*}

For the case when $\frac{1}{2 \mu n K} \le \eta < \frac{2\mu}{L^2}$,
\begin{align*}
    y_n^1
    &= S y_0^1 + A_{\sigma_1}\\
    &\ge (1 - 4 \eta m \mu) \frac{L \nu}{240 \mu^2 n K} + \frac{\eta^2 m L \nu}{30}\\
    &= \frac{L \nu}{240 \mu^2 n K} - \frac{\eta L \nu}{120 \mu K} + \frac{\eta^2 n L \nu}{60} \,\,\,\,\,\,\,\,\,\,\,\, (\because n=2m)\\
    &= \frac{L \nu}{240 \mu^2 n K} - \frac{\eta n L \nu}{60} \bigopen{\frac{1}{2 \mu n K} - \eta}\\
    &\ge \frac{L \nu}{240 \mu^2 n K}.
\end{align*}
Applying this process in a chain, we then gain $y_n^k \ge \frac{L \nu}{240 \mu^2 n K}$ for all $k \in \{1, \cdots, K\}$.
Therefore, regardless of the choice of $\{\alpha_k\}_{k=1}^{K+1}$, $\hat{y} = \Omega \bigopen{\frac{L \nu}{\mu^2 n K}}$ holds and $F_2(\hat{y}) - F_2^* = \frac{\mu}{2} \hat{y}^2 = \Omega \bigopen{\frac{L^2 \nu^2}{\mu^3 n^2 K^2}}$.

For the case when $\frac{2\mu}{L^2} \le \eta \le \frac{2}{nL}$, we have
\begin{align*}
    y_n^1
    &= S y_0^1 + A_{\sigma_1}\\
    &\ge (1 - 2 \eta^2 m L^2) \frac{\nu}{60L} + \frac{\eta^2 m L \nu}{30}\\
    &= \frac{\nu}{60L}.
\end{align*}
Applying this process in a chain, we then gain $y_n^k \ge \frac{\nu}{60L}$ for all $k \in \{1, \cdots, K\}$.
Therefore, regardless of the choice of $\{\alpha_k\}_{k=1}^{K+1}$, $\hat{y} = \Omega \bigopen{\frac{\nu}{L}}$ holds and $F_2(\hat{y}) - F_2^* = \frac{\mu}{2} \hat{y}^2 = \Omega\left(\frac{\mu \nu^2}{L^2}\right) = \Omega\left(\frac{L^2 \nu^2}{\mu^3 n^2 K^2} \right)$, 
where we used $K \ge \frac{\kappa^2}{n}$ in the last step.
\end{proof}

\subsection{\texorpdfstring{Lower Bound for $\eta \in \left[\frac{2}{n L}, \frac{1}{L}\right]$}{Lower Bound for η ∈ [2/nL, 1/L]}} \label{subsec:grabLBbc}
Here we show that there exists $F_3 (z) \in \gF_{\text{PŁ}} \bigopen{L, \mu, \frac{L}{\mu}, \nu}$ such that any permutation-based SGD with $z_0^1 = \frac{3\nu}{8nL}$ satisfies
\begin{align*}
    F_3(\hat{z}) - F_3^* = \Omega \left( \frac{L^2 \nu^2}{\mu^3 n^2K^2} \right).
\end{align*}

\begin{proof}
We define $F_3 (z)$ by the following components:
\begin{align*}
    f_i(z) = \begin{cases}
    \frac{L}{2}z^2 - \nu z & \text{if } i = 1,\\
    -\frac{L}{4(n-1)}z^2 + \frac{\nu}{n-1}z & \text{otherwise.}
    \end{cases}
\end{align*}
With this construction, the finite-sum objective becomes
\begin{align*}
    F_3(z) = \frac{1}{n} \sum_{i=1}^n f_i(z) = \frac{L}{4n} z^2.
\end{align*}
Note that all the components are $L$-\textit{smooth} and $F$ is $\mu$-\textit{strongly convex} since we assumed $\frac{L}{4n} \ge \mu$. Moreover,
\begin{align*}
    \bignorm{\nabla f_1(z) - \nabla F_3(z)} 
    &= \bignorm{\bigopen{Lz - \nu} - \frac{Lz}{2n}} \le \bignorm{\bigopen{1 - \frac{1}{2n}}Lz} + \nu\\
    &\le \bignorm{Lz} + \nu = 2n \bignorm{\nabla F_3(z)} + \nu \le \frac{L}{\mu} \bignorm{\nabla F_3(z)} + \nu,\\
    \bignorm{\nabla f_2(z) - \nabla F_3(z)}
    &= \bignorm{\bigopen{-\frac{Lz}{2(n-1)} + \frac{\nu}{n-1}} - \frac{Lz}{2n}}\\
    &< \bignorm{Lz} + \nu = 2n \bignorm{\nabla F_3(z)} + \nu \le \frac{L}{\mu} \bignorm{\nabla F_3(z)} + \nu,
\end{align*}
and thereby $F_3 \in \gF_{\text{PŁ}} \bigopen{L, \mu, \frac{L}{\mu}, \nu}$.
Also, we can easily verify $F_3^* = 0$ at $z^* = 0$.

Similarly as in \eqref{eq:thm2_2_1}, we temporarily write $\nabla f_i(y) = a_i y - b_i$ where $a_i \in \bigset{L, -\frac{L}{2(n-1)}}$, $b_i \in \bigset{\nu, -\frac{\nu}{n-1}}$ holds.
We then write $y_n^k$ as $S y_0^k + A_{\sigma_k}$, where $S := \prod_{i=1}^n \bigopen{1 - \eta a_{\sigma_k(i)}} = \prod_{i=1}^n \bigopen{1 - \eta a_{i}}$ is independent of the choice of $\sigma_k$, and $A_{\sigma} := \eta \sum_{i=1}^n b_{\sigma(i)} \prod_{j=i+1}^{n} \bigopen{1 - \eta a_{\sigma(j)}}$ is the term that we can control using permutation-based SGD.

 We will first find what permutation $\sigma$ leads to the smallest $A_{\sigma}$.
 Choose arbitrary $\sigma$ and assume that $\sigma (1) \neq 1$. Define $t := \sigma^{-1} (1)$. 
 We then define another permutation $\sigma^\prime$ by $\sigma^\prime (t-1) = 1$, $\sigma^\prime (t) = \sigma (t-1)$ and $\sigma^\prime (i) = \sigma (i)$ for $i \in \{1, \cdots, n\} \setminus \{t-1, t\}$.

Let $z_{\sigma}$ and $z_{\sigma^\prime}$ as the value of $z_n^k$ generated by $\sigma$ and $\sigma^\prime$ starting from the same $z_0^k$, respectively.
We will show that $z_{\sigma_1} > z_{\sigma_1^\prime}$.
In a similar manner as \eqref{eq:thm2_2_2},
\begin{align*}
    z_{\sigma} - z_{\sigma^\prime}
    &= S z_0^k + A_{\sigma} - S z_0^k - A_{\sigma^\prime}\\
    &= \bigopen{\eta \prod_{j=t+1}^n \bigopen{1 + \frac{\eta L}{2(n-1)}}} \cdot \bigopen{\bigopen{-\frac{\nu}{n-1} \bigopen{1 - \eta L} + \nu} - \bigopen{\nu \bigopen{1 + \frac{\eta L}{2(n-1)}} - \frac{\nu}{n-1}}}\\
    &= \bigopen{\eta \prod_{j=t+1}^n \bigopen{1 + \frac{\eta L}{2(n-1)}}} \cdot \bigopen{\frac{\eta L \nu}{2(n-1)}} > 0
\end{align*}
holds.
Thus, we can conclude that the permutation $\sigma$ satisfying $\sigma(1) = 1$ is the permutation that minimizes $A_{\sigma}$. 
Let $\sigma^*$ denote such $\sigma$. 

With this permutation $\sigma^*$, $A_{\sigma^*}$ becomes
\begin{align}
    A_{\sigma^*} &= \eta \nu \left( 1 + \frac{\eta L}{2(n-1)} \right)^{n-1} - \frac{\eta \nu}{n-1} \sum_{i=0}^{n-2} \left( 1 + \frac{\eta L}{2(n-1)} \right)^i \nonumber \\
    &= \eta \nu \left( 1 + \frac{\eta L}{2(n-1)} \right)^{n-1} - \frac{\eta \nu}{n-1} \cdot \frac{\left(1 + \eta L/(2(n-1))\right)^{n-1} - 1}{\eta L /(2(n-1))} \nonumber \\
    &= \frac{2\nu}{L} - \left( 1 + \frac{\eta L}{2(n-1)} \right)^{n-1} \left( \frac{2\nu}{L} - \eta \nu \right). \label{eq:thm2_3_1}
\end{align}
Note that $\eta L \le 1$, so $\frac{2\nu}{L} - \eta \nu$ is nonnegative. Using \lmmref{lmm:thm2_5}, we have
\begin{align}
    \bigopen{1 + \frac{\eta L}{2(n-1)}}^{n-1} < e^{\frac{\eta L}{2}} < 1 + \frac{\eta L}{2} + \frac{5 \eta^2 L^2}{32}. \label{eq:thm2_3_2}
\end{align}
Substituting \eqref{eq:thm2_3_2} to \eqref{eq:thm2_3_1} results
\begin{align*}
    A_{\sigma^*}
    &> \frac{2\nu}{L} - \left(1 + \frac{\eta L}{2} + \frac{5\eta^2 L^2}{32}\right) \left( \frac{2\nu}{L} - \eta \nu \right)\\
    &= \frac{3\eta^2 L \nu}{16} + \frac{5 \eta^3 L^2 \nu}{32} \\
    &> \frac{3\eta^2 L \nu}{16}.
\end{align*}

We start at $z_0^1 = \frac{3\nu}{8nL}$. Using $S=\left(1 - \eta L\right)\left(1 + \frac{\eta L}{2(n-1)}\right)^{n-1} > 1 - \eta L$, we have
\begin{align*}
    z_n^1
    &= S z_0^1 + A_{\sigma^*}\\
    &\ge (1 - \eta L) z_0^1 + \frac{3\eta^2 L \nu}{16}\\
    &= (1 - \eta L)\frac{3\nu}{8nL} + \frac{3\eta^2 L \nu}{16}\\
    &= \frac{3\nu}{8nL} - \frac{3\eta \nu}{8n} + \frac{3\eta^2 L \nu}{16}\\
    &= \frac{3\nu}{8nL} + \frac{3 \eta L \nu}{16} \bigopen{\eta - \frac{2}{nL}}\\
    &\ge \frac{3\nu}{8nL}.
\end{align*}
Applying this process in a chain, we then gain $z_n^k \ge \frac{3\nu}{8nL}$ for all $k \in \{1, \cdots, K\}$.
Therefore, regardless of the choice of $\{\alpha_k\}_{k=1}^{K+1}$, $\hat{z} = \Omega \bigopen{\frac{\nu}{nL}}$ holds and $F_3(\hat{z}) - F_3^* = \frac{L}{4n} \hat{z}^2 = \Omega\left(\frac{\nu^2}{n^3L}\right) = \Omega\left(\frac{L^2 \nu^2}{\mu^3 n^2 K^2} \right)$, 
where we used $K \ge \kappa^{3/2}n^{1/2}$ in the last step.
\end{proof}

\subsection{\texorpdfstring{Lower Bound for $\eta > \frac{1}{L}$}{Lower Bound for η > 1/L}} \label{subsec:grabLBbd}
Here we show that there exists $F_4 (w) \in \gF_{\text{PŁ}} (2L, \mu, 0, \nu)$ such that any permutation-based SGD with $w_0^1 = \frac{L^{1/2} \nu}{\mu^{3/2} n K}$ satisfies
\begin{align*}
    F_4(\hat{w}) - F_4^* = \Omega \left( \frac{L^2 \nu^2}{\mu^3 n^2 K^2} \right).
\end{align*}

\begin{proof}
We define $F_4 (w) \in \gF_{\text{PŁ}} (2L, 2L, 0, 0)$ by the following components:
\begin{align*}
    f_i(w) = Lw^2.
\end{align*}
Note that $\gF_{\text{PŁ}}(2L, 2L, 0, 0) \subseteq \gF_{\text{PŁ}} (2L, \mu, 0, \nu)$ and $F_4^* = 0$ at $w^* = 0$ by definition.

In this regime, we will see that the step size is too large so that $\bigset{w_n^k}_{k=1}^{K}$ diverges.
We start from $w_0^1 = \frac{L^{1/2} \nu}{\mu^{3/2} n K}$. Since the gradient of all component functions evaluated at point $w$ is fixed deterministically to $2 L w$, we have for every $k \in [K]$,
\begin{align*}
    w_n^k = \bigopen{1 - 2 \eta L}^{nk} w_0^1 \ge 1^{nk} \frac{L^{1/2} \nu}{\mu^{3/2} n K} &= \Omega \bigopen{\frac{L^{1/2} \nu}{\mu^{3/2} n K}},
\end{align*}
where we used the fact that $n$ is even in the second step.
Thus, regardless of the permutation-based SGD algorithm we use, we have $\hat{w} = \Omega \bigopen{\frac{L^{1/2} \nu}{\mu^{3/2} n K}}$ and $F_4(\hat{w}) - F_4^* = L \hat{w}^2 = \Omega \bigopen{\frac{L^2 \nu^2}{\mu^3 n^2 K^2}}$.
\end{proof}

\subsection{\texorpdfstring{Lemmas used in \thmref{thm:grabLBb}}{Lemmas used in Theorem 9}} \label{subsec:grabLBbe}
In this subsection, we will prove the lemmas used in \cref{thm:grabLBb}.
\begin{lemma} \label{lmm:thm2_1}
    For any even $n \ge 104$, any $0 < x \le \frac{2}{n}$ and any $1 - \frac{2}{n} \le \beta < 1$, let $m=\frac{n}{2}$. Then, the following inequality holds:
    \begin{align}
        (1 + \beta x)^m (\beta - 1) - \beta(1 + \beta x)^m (1 - x)^m + 1\ge \frac{m x^2}{30}. \label{eq:lmm21}
    \end{align}
\end{lemma}

\begin{proof}
To prove the lemma, we focus on the coefficients of $x^k$ for $0 \le k \le 2m$.

Define $a_k$ as the absolute value of $x^k$'s coefficient in $(1 + \beta x)^m (\beta - 1)$.
Using the fact that $1-\frac{1}{m} \le \beta < 1$ , $a_k = \left\lvert \binom{m}{k} \beta^k (\beta - 1) \right\rvert \le \frac{m^k}{k!} \cdot \frac{1}{m} = \frac{m^{k-1}}{k!}$. Note that for $k \ge m+1$, $a_k$ is $0$.

Let $b_k$ be $x^k$'s coefficient in $\beta(1 + \beta x)^m (1 - x)^m$.
While the sequence of coefficients $\{b_k\}$ have alternating signs, we can define a \textit{positive} sequence $c_k$ which upper bounds the sequence $|b_k|$.
Since $(1 + \beta x)^m (1 - x)^m = (1 - (1 - \beta)x - \beta x^2)^m$,
\begin{align*}
    |b_k| &= \beta \cdot \bigabs{\sum_{t=\max \bigset{0, k-m}}^{\lfloor \frac{k}{2} \rfloor} (-\beta)^t (-(1-\beta))^{k-2t} \frac{m!}{t!(k-2t)!(m-k+t)!}}\\
    &\le 1 \cdot \sum_{t=\max \bigset{0, k-m}}^{\lfloor \frac{k}{2} \rfloor} \beta^t (1-\beta)^{k-2t} \frac{m!}{t!(k-2t)!(m-k+t)!}\\
    & \triangleq c_k
\end{align*}
Then $x^k$'s coefficient in LHS of \cref{eq:lmm21} is lower bounded by $-(a_k + c_k)$.
For even $k < m$, we have
\begin{align}
    \frac{c_{k+1}}{c_{k}}
    &\le (1 - \beta) \max_{t \le \lfloor \frac{k}{2} \rfloor} \frac{m!/(t!(k+1-2t)!(m-k-1+t)!)}{m!/(t!(k-2t)!(m-k+t)!)} \nonumber \\
    &\le \frac{1}{m} \max_{t \le \lfloor \frac{k}{2} \rfloor} \frac{m-k+t}{k+1-2t} \nonumber \\
    &\le \frac{1}{m} \max_{t \le \lfloor \frac{k}{2} \rfloor} (m-k+t) \nonumber \\
    &\le \frac{1}{m} \cdot m \nonumber \\
    &= 1. \label{eq:lmm_thm2_1_1}
\end{align}
For odd $k < m$, we have
\begin{align}
    c_{k+1}
    &= \beta^{\frac{k+1}{2}} \frac{m!}{(\frac{k+1}{2})!(m-\frac{k+1}{2})!} + \sum_{t=0}^{\frac{k-1}{2}} \beta^t (1-\beta)^{k+1-2t} \frac{m!}{t!(k+1-2t)!(m-k-1+t)!} \nonumber\\
    &< 1^{\frac{k+1}{2}}\cdot \frac{m^{\frac{k+1}{2}}}{(\frac{k+1}{2})!} + c_k \cdot (1 - \beta) \max_{t \le \lfloor \frac{k}{2} \rfloor} \frac{m!/(t!(k+1-2t)!(m-k-1+t)!)}{m!/(t!(k-2t)!(m-k+t)!)} \nonumber\\
    &\le \frac{m^{\frac{k+1}{2}}}{(\frac{k+1}{2})!} + c_k. \label{eq:lmm_thm2_1_2}
\end{align}
For $k \ge m$, we have
\begin{align}
    \frac{c_{k+1}}{c_k}
    &\le (1 - \beta) \max_{t \le \lfloor \frac{k}{2} \rfloor} \frac{m!/(t!(k+1-2t)!(m-k-1+t)!)}{m!/(t!(k-2t)!(m-k+t)!)} \nonumber \\
    &\le \frac{1}{m} \max_{t \le \lfloor \frac{k}{2} \rfloor} \frac{m-k+t}{k+1-2t} \nonumber \\
    &\le \frac{1}{m} \max_{t \le \lfloor \frac{k}{2} \rfloor} (m-k+t) \nonumber \\
    &\le \frac{1}{m} \cdot \left(m-\frac{k}{2} \right) \nonumber \\
    &\le \frac{1}{m} \cdot \frac{m}{2} \nonumber \\
    &= \frac{1}{2}. \label{eq:lmm_thm2_1_3}
\end{align}
Using (\ref{eq:lmm_thm2_1_1}), (\ref{eq:lmm_thm2_1_2}) and (\ref{eq:lmm_thm2_1_3}), we will show $c_k \le \frac{m^{k-1}}{k!}$ for $4 \le k \le 2m$.

Note that $c_1 = (1-\beta)\cdot \frac{m!}{(m-1)!} \le 1$. Also, we can easily prove $\sum_{i=0}^{p} \frac{m^i}{i!} \le \frac{m^p}{(p-1)!}$ for $\forall m \ge 3$, $\forall 2 \le p \le m-1$ using mathematical induction. Therefore, for $k \le m$,
\begin{align*}
    \begin{cases}
    c_k \le \sum_{i=0}^{\frac{k-1}{2}} \frac{m^i}{i!} \le \frac{m^{\frac{k-1}{2}}}{(\frac{k-1}{2} - 1)!} & \text{if $k$ is odd,}\\
    c_k \le \sum_{i=0}^{\frac{k}{2}} \frac{m^i}{i!} \le \frac{m^{\frac{k}{2}}}{(\frac{k}{2} - 1)!} & \text{if $k$ is even,}
    \end{cases}
\end{align*}
and applying \lmmref{lmm:thm2_2} and \lmmref{lmm:thm2_3}, we finally get $c_k \le \frac{m^{k-1}}{k!}$.

For $k > m$,
\begin{align*}
    c_k
    &\le c_m \cdot \left(\frac{1}{2}\right)^{k-m}\\
    &< \frac{m^{m-1}}{m!} \cdot \frac{m}{m+1} \cdot \frac{m}{m+2} \cdot \cdots \cdot \frac{m}{k}\\
    &= \frac{m^{k-1}}{k!}.
\end{align*}
Thus, we have proven $c_k \le \frac{m^{k-1}}{k!}$ for $4 \le k \le 2m$. Since we also have $a_k \le \frac{m^{k-1}}{k!}$, we can conclude that $a_k + c_k \le \frac{2 \cdot m^{k-1}}{k!}$, i.e., the absolute value of the $x^k$'s coefficient of LHS of our statement is upper bounded by $\frac{2 \cdot m^{k-1}}{k!}$ when $4 \le k \le 2m$.

We now consider the coefficient of $x^k$ when $k < 4$.
For $k = 0$, the coefficient is
\begin{align*}
    (\beta - 1) - \beta \cdot 1 + 1 = 0.
\end{align*}
For $k = 1$, the coefficient is
\begin{align*}
    \beta m (\beta - 1) - \beta (\beta m - m) = 0.
\end{align*}
For $k = 2$, the coefficient is
\begin{align*}
    \beta^2 \cdot \binom{m}{2} \cdot (\beta - 1) - \beta \cdot \left( (1- \beta)^2 \cdot \binom{m}{2} - \beta \cdot m \right) = \beta^2 \cdot \frac{m(m+1)}{2} - \beta \cdot \frac{m(m-1)}{2}.
\end{align*}
For fixed $m$, RHS is a quadratic with respect to $\beta$, and it is minimized when $\beta$ is $1 - \frac{1}{m}$. Hence the above equation can be lower bounded by
\begin{align}
    &\bigopen{1 - \frac{1}{m}}^2 \cdot \frac{m(m+1)}{2} - \bigopen{1 - \frac{1}{m}} \cdot \frac{m(m-1)}{2} \nonumber \\
    &=\frac{m}{2} - 1 + \frac{1}{2m} \nonumber \\
    &\ge \frac{2m}{5}. \label{eq:lmm_thm2_1_4} &(m \ge 10).
\end{align}
For $k=3$, the coefficient is
\begin{align}
    &\beta^3 \cdot \binom{m}{3} \cdot (\beta - 1) - \beta \cdot \left( -\binom{m}{3} + m \cdot \beta \cdot \binom{m}{2} - \binom{m}{2} \cdot \beta^2 \cdot m + \binom{m}{3} \cdot \beta^3 \right) \nonumber\\
    &=\frac{\beta}{6}\cdot m(m-1)(m-2) - \frac{\beta^2}{2}\cdot m^2(m-1) + \frac{\beta^3}{3} \cdot (m+1)m(m-1) \label{eq:lmm_thm2_1_5}.
\end{align}
For fixed $m$, (\ref{eq:lmm_thm2_1_5}) is a cubic function with respect to $\beta$. Differentiating this function, we get
\begin{align*}
    &\beta^2 (m+1)m(m-1) - \beta m^2(m-1) + \frac{m(m-1)(m-2)}{6}\\
    &=m(m-1) \cdot \bigopen{ \beta^2 (m+1) - \beta m + \frac{m-2}{6} }\\
    &=m(m-1) \cdot \bigopen{ \beta m (\beta - 1) + \beta^2 + \frac{m-2}{6}}\\
    &\ge m(m-1) \cdot \bigopen{ -\beta + \beta^2 + \frac{m-2}{6}} &(\because \beta - 1 \ge -\frac{1}{m})\\
    &= m(m-1) \cdot \bigopen{\beta (\beta - 1) + \frac{m-2}{6}}\\
    &\ge m(m-1) \cdot \bigopen{-\frac{1}{m} + \frac{m-2}{6}} &(\because \beta \le 1 \,\, \& \,\, \beta - 1 \ge -\frac{1}{m})\\
    &> 0. &(\because m \ge 4)
\end{align*}
Thereby, (\ref{eq:lmm_thm2_1_5}) is minimized when $\beta$ is $1 - \frac{1}{m}$, and substituting such $\beta$ to (\ref{eq:lmm_thm2_1_5}) results
\begin{align}
    &\frac{1 - \frac{1}{m}}{6}\cdot m(m-1)(m-2) - \frac{\bigopen{1 - \frac{1}{m}}^2}{2}\cdot m^2(m-1) + \frac{\bigopen{1 - \frac{1}{m}}^3}{3} \cdot (m+1)m(m-1) \nonumber\\
    &= -\frac{m^2}{6} - \frac{1}{m} + \frac{1}{3m^2} + \frac{5}{6} \nonumber \\
    &\ge -\frac{m^2}{6} \label{eq:lmm_thm2_1_6}.
\end{align}
Remind that $x^k$'s coefficient in LHS of \cref{eq:lmm21} is lower bounded by $-(a_k + c_k)$.
Summing up (\ref{eq:lmm_thm2_1_4}), (\ref{eq:lmm_thm2_1_6}), and the fact that $a_k + c_k \le \frac{2 \cdot m^{k-1}}{k!}$ for $k \ge 4$, we obtain
\begin{align*}
    &(1 + \beta x)^m (\beta - 1) - \beta(1 + \beta x)^m (1 - x)^m + 1\\
    &\ge \frac{2m}{5} x^2 - \frac{m^2}{6} x^3 - \sum_{k=4}^{2m} x^k \cdot \frac{2 \cdot m^{k-1}}{k!}\\
    &> \frac{2}{5} mx^2 - \frac{mx}{6} mx^2 - \sum_{k=4}^{\infty} x^k \cdot \frac{2 \cdot m^{k-1}}{k!}\\
    &\ge \frac{2}{5} mx^2 - \frac{1}{6} mx^2 - mx^2 \cdot \frac{2}{m^2x^2} \cdot \sum_{k=4}^{\infty} \frac{(mx)^k}{k!} &(\because mx \le 1).
\end{align*}
For the last term, $\frac{1}{m^2x^2} \cdot \sum_{k=4}^{\infty} \frac{(mx)^k}{k!}$ is an increasing function of $mx$ so it is maximized when $mx$ is 1. 
Thereby we can further extend the above inequality as:
\begin{align*}
    &\frac{2}{5} mx^2 - \frac{1}{6} mx^2 - mx^2 \cdot 2\left(e - 1 - \frac{1}{1!} - \frac{1}{2!} - \frac{1}{3!} \right) \\
    &\ge \frac{2}{5} mx^2 - \frac{1}{6} mx^2 - \frac{1}{5}mx^2\\
    &= \frac{1}{30} mx^2.
\end{align*}
\end{proof}

\begin{lemma} \label{lmm:thm2_2}
    For $m \ge 52$ and even $4 \le k \le m$, $\frac{m^{k-1}}{k!} > \frac{m^{\frac{k}{2}}}{\left(\frac{k}{2} - 1 \right)!}$ holds.
\end{lemma}
\begin{proof}

We first consider the case when $k \ge 14$. 
Since $m \ge k$, it is sufficient to show $m^{\frac{k}{2} - 2} > \frac{(k-1)!}{\left(\frac{k}{2} - 1 \right)!}$.
Taking log on both sides, this inequality becomes
\begin{align*}
    \bigopen{\frac{k}{2} - 2} \log m > \sum_{i=\frac{k}{2}}^{k-1} \log i.
\end{align*}
Using $\sum_{i=\frac{k}{2}}^{k-1} \log i < \int_{\frac{k}{2}}^{k} \log x \, dx$, we will instead prove following inequality when $k \ge 16$:
\begin{align*}
    \log m > \frac{\int_{\frac{k}{2}}^{k} \log x \, dx}{\frac{k}{2} - 2} = \frac{k \log k - \frac{k}{2} \log \frac{k}{2} - \frac{k}{2}}{\frac{k}{2} - 2}.
\end{align*}
Define $f(X) := \frac{2X \log(2X) - X \log X - X}{X-2} = \frac{X \log X + 2X \log 2 - X}{X-2}$. Then,
\begin{align*}\frac{}{}
    f'(X)
    &= \left( \frac{X \log X + 2X \log 2 - X}{X-2} \right)'\\
    &= \frac{X - 2 \log X - 4 \log 2}{(X-2)^2}.
\end{align*}
We can numerically check that $f'(\frac{k}{2}) > 0$ holds for $k \ge 14$.
Therefore, for fixed $m$, $\argmax_{k \ge 14} f\left( \frac{k}{2} \right) = 2\lfloor \frac{m}{2} \rfloor$.

We now have to prove $\log m > f\left( \lfloor \frac{m}{2} \rfloor \right)$.
Let $s = \lfloor \frac{m}{2} \rfloor$. Then $f\left( \lfloor \frac{m}{2} \rfloor \right)$ becomes
\begin{align*}
    f\left( s \right) = \frac{s \log s + 2s \log2 - s}{s-2}.
\end{align*}
Combining $\log m \ge \log (2s)$ and
\begin{align*}
    &\log (2s) \ge \frac{s \log s + 2s \log 2 - s}{s-2}\\
    &\Longleftrightarrow (s-2) \log (2s) \ge s \log s + 2s \log 2 - s\\
    &\Longleftrightarrow s \ge (s+2) \log 2 + 2 \log s\\
    &\Longleftarrow s \ge 26 \Longleftrightarrow m \ge 52,
\end{align*}
we have proven the statement.

Now, we are left to prove the lemma for $k < 14$. Exchanging $m^{\frac{k}{2}}$ and $k!$ in the statement of the lemma, we have
\begin{align}
    m^{\frac{k}{2}-1} > \frac{k!}{\bigopen{\frac{k}{2}-1}!}. \label{eq:lmm_thm2_2_1}
\end{align}
We can numerically check that
\begin{itemize}
    \item for $k = 4$, $m \ge 25$ is sufficient,
    \item for $k = 6$, $m \ge 19$ is sufficient,
    \item for $k = 8$, $m \ge 19$ is sufficient,
    \item for $k = 10$, $m \ge 20$ is sufficient,
    \item for $k = 12$, $m \ge 21$ is sufficient,
\end{itemize}
for (\ref{eq:lmm_thm2_2_1}) to hold. This ends the proof of the lemma.
\end{proof}

\begin{lemma} \label{lmm:thm2_3}
    For $m \ge 52$ and odd $4 \le k \le m$,  $\frac{m^{k-1}}{k!} > \frac{m^{\frac{k-1}{2}}}{(\frac{k-1}{2} - 1)!}$
\end{lemma}
\begin{proof} 
\begin{align*}
    \frac{m^{k-1}}{k!} = \frac{m}{k} \cdot \frac{m^{k-2}}{(k-1)!} > \frac{m}{k} \cdot \frac{m^{\frac{k-1}{2}}}{\left( \frac{k-1}{2} - 1 \right)!} \ge \frac{m^{\frac{k-1}{2}}}{\left( \frac{k-1}{2} - 1 \right)!},
\end{align*}
where we used \lmmref{lmm:thm2_2} in the first inequality. This ends the proof.
\end{proof}

\begin{lemma} \label{lmm:thm2_5}
    For $x \le 1$, the following inequality holds:
    \begin{align*}
        e^{\frac{x}{2}} < 1 + \frac{x}{2} + \frac{5x^2}{32}.
    \end{align*}
\end{lemma}
\begin{proof}
Using Taylor expansion,
\begin{align*}
    e^{\frac{x}{2}}
    &= 1 + \frac{x}{2} + \frac{x^2}{8} + \sum_{i=3}^{\infty} \frac{1}{i!} \cdot \frac{x^i}{2^i}\\
    &= 1 + \frac{x}{2} + \frac{x^2}{8} + x^2 \sum_{i=3}^{\infty} \frac{1}{i!} \cdot \frac{x^{i-2}}{2^i}\\
    &\le 1 + \frac{x}{2} + \frac{x^2}{8} + x^2 \sum_{i=3}^{\infty} \frac{1}{i!} \cdot \frac{1}{2^i}\\
    &= 1 + \frac{x}{2} + \frac{x^2}{8} + x^2 \bigopen{e^{\frac{1}{2}} - 1 - \frac{1}{1! \cdot 2} - \frac{1}{2! \cdot 2^2}}\\
    &\le 1 + \frac{x}{2} + \frac{x^2}{8} + \frac{x^2}{32}\\
    &= 1 + \frac{x}{2} + \frac{5x^2}{32}.
\end{align*}
\end{proof}
\section{\texorpdfstring{Proof of \propref{prop:grabUBb}}{Proof of Proposition 10}} \label{sec:grabUB}
Here we prove \cref{prop:grabUBb}, restated below for the sake of readability.
\grabUBb*
\begin{proof} [Proof of \propref{prop:grabUBb}]
While \citet{lu2022grab} gained convergence rate for $F \in \mathcal{F}_{\text{PŁ}} (L,\mu,0,\nu)$, 
we found out that their result can easily be extended to $F \in \mathcal{F}_{\text{PŁ}} (L,\mu,\tau,\nu)$ with a slight adjustment.
We basically follow up the proof step in Theorem 1 of \citet{lu2022grab}. We first state 2 lemmas that will help us prove the proposition.
\begin{restatable}[Extended version of \citet{lu2022grab}, Lemma 2] {lemma}{grabUBlmma}
    \label{lmm:grabUBlmm1}
    Applying \textit{offline GraB} to a function $F \in \mathcal{F}_{\emph{\text{PŁ}}} (L,\mu,\tau,\nu)$ with $\eta n L < 1$ results
    \begin{align*}
        F(\vx_n^{K}) - F^* \le \rho^K (F(\vx_0^1) - F^*) + \frac{\eta n L^2}{2} \sum_{k=1}^K \rho^{K-k} \Delta_k^2 - \frac{\eta n}{4} \sum_{k=1}^{K} \rho^{K-k} \bignorm{\nabla F(\vx_0^k)}^2,
    \end{align*}
    where $\rho = 1 - \frac{\eta n \mu}{2}$ and $\Delta_k = \max_{m=1, \cdots , n}  \bignorm{\vx_m^k - \vx_0^k}$ for all $k \in [K]$.
\end{restatable}
\begin{restatable}[Extended version of \citet{lu2022grab}, Lemma 3]{lemma}{grabUBlmmb}
    \label{lmm:grabUBlmm2}
    Applying \textit{offline GraB} to a function $F \in \mathcal{F}_{\emph{\text{PŁ}}} (L,\mu,\tau,\nu)$ with $\eta n L \le \frac{1}{2}$ results
    \begin{align*}
    &\Delta_1 \le 2 \eta n \nu + 2 \eta n (\tau+1) \cdot \bignorm{\nabla F(\vx_0^1)}, \,\,\, \text{and} \\
    &\Delta_k \le 2\eta H \nu + (2\eta H \tau + 2\eta n) \cdot \left \Vert \nabla F \left( \vx_0^k \right) \right \Vert + \left( 4\eta HL(\tau+1) + 8\eta nL \right) \cdot \Delta_{k-1}
\end{align*}
for $k \in [K]\setminus\{1\}$.
\end{restatable}
We defer the proofs of the lemmas to \cref{sec:grabUB_lmm}.
We start by finding the upper bound of $\sum_{k=1}^K \rho^{K-k} \Delta_k^2$. From \lmmref{lmm:grabUBlmm2}, we have
\begin{align*}
    \Delta_k \le 2\eta H \nu + (2\eta H \tau + 2\eta n) \cdot \bignorm{\nabla F \left( \vx_0^k \right)} + \bigopen{ 4\eta HL(\tau+1) + 8\eta nL } \cdot \Delta_{k-1}
\end{align*}
for $k \in [K] \setminus \{1\}$. Taking square on both sides and applying the inequality $3 \bigopen{a^2 + b^2 + c^2} \ge \bigopen{a + b + c}^2$, we get
\begin{align*}
    \Delta_k^2 \le 3 \eta^2 \left(4 HL(\tau+1) + 8nL \right)^2 \Delta_{k-1}^2 + 12 \eta^2 H^2 \nu^2 + 12 \eta^2 (H\tau+n)^2 \left \Vert \nabla F \left( \vx_0^k \right) \right \Vert^2.
\end{align*}
Similarly, for $k=1$, we have
\begin{align*}
    \Delta_1^2 \le 8 \eta^2 n^2 (\tau+1)^2 \left \Vert \nabla F(\vx_0^1) \right \Vert^2 + 8 \eta^2 n^2 \nu^2.
\end{align*}
Hence,
\begin{align*}
    &\sum_{k=1}^K \rho^{K-k} \Delta_k^2\\
    &= \sum_{k=2}^{K} \rho^{K-k} \Delta_k^2 + \rho^{K-1} \Delta_1^2\\
    &\le \sum_{k=2}^{K} \rho^{K-k} \left( 3 \eta^2 \left(4 HL(\tau+1) + 8nL \right)^2 \Delta_{k-1}^2 + 12 \eta^2 H^2 \nu^2 + 12 \eta^2 (H\tau+n)^2 \bignorm{\nabla F(\vx_0^k)}^2 \right)\\
    & \,\,\,\,\,\, + \rho^{K-1} \left( 8 \eta^2 n^2 (\tau+1)^2 \bignorm{\nabla F(\vx_0^1)}^2 + 8 \eta^2 n^2 \nu^2 \right)\\
    &\le 3 \rho^{-1} \eta^2 \left(4 HL(\tau+1) + 8nL \right)^2 \sum_{k=2}^{K} \rho^{K-(k-1)} \Delta_{k-1}^2 + \frac{12\eta^2 H^2 \nu^2}{1 - \rho} + 8 \rho^{K-1} \eta^2 n^2 \nu^2\\
    & \,\,\,\,\,\, + 12 \eta^2 (H\tau+n)^2 \sum_{k=2}^K \rho^{K-k} \bignorm{\nabla F(\vx_0^k)}^2 + 8 \eta^2 n^2 (\tau+1)^2 \rho^{K-1} \bignorm{\nabla F(\vx_0^1)}^2.
\end{align*}
From the assumption that $H \le n$, $H\tau + n \le n(\tau + 1)$ holds. Then, we get
\begin{align}
    \sum_{k=1}^K \rho^{K-k} \Delta_k^2
    &\le 3 \rho^{-1} \eta^2 \left(4 HL(\tau+1) + 8nL \right)^2 \sum_{k=1}^{K} \rho^{K-k} \Delta_{k}^2 + \frac{12\eta^2 H^2 \nu^2}{1 - \rho} \nonumber \\
    &\,\,\,\,\,\, + 8 \rho^{K-1} \eta^2 n^2 \nu^2 + 12 \eta^2 n^2 (\tau + 1)^2 \sum_{k=1}^K \rho^{K-k} \bignorm{\nabla F(\vx_0^k)}^2. \label{eq:prop1_1_1}
\end{align}
We now define our step size as:
\begin{align*}
    \eta = \min \bigopen{\frac{1}{64nL(\tau + 1)}, \frac{2}{\mu n K} W_0 \bigopen{\frac{\bigopen{F(\vx_0^1) - F^* + \nu^2/L}\mu^3 n^2 K^2}{192H^2L^2\nu^2}}}.
\end{align*}
We first focus on $\eta \le \frac{1}{64nL(\tau + 1)}$.
With this step size range, $ \rho = 1 - \frac{\eta n \mu}{2} \ge 1 - \frac{\eta n L}{2} > \frac{1}{2}$ and 
\begin{align}
    \eta (4HL(\tau + 1) + 8nL)
    &\le \frac{4HL(\tau + 1)}{64nL(\tau+1)} + \frac{8nL}{64nL(\tau+1)} \nonumber\\
    &\le \frac{H}{16n} + \frac{1}{8(\tau + 1)} < \frac{1}{4} \label{eq:prop1_1_2}
\end{align}
holds. Thereby,
\begin{align*}
    3 \rho^{-1} \eta^2 \left(4 HL(\tau+1) + 8nL \right)^2 < 3 \cdot 2 \cdot \frac{1}{16} < \frac{1}{2}
\end{align*}
holds and (\ref{eq:prop1_1_1}) becomes
\begin{align*}
    \sum_{k=1}^K \rho^{K-k} \Delta_k^2 \le \frac{24\eta^2 H^2 \nu^2}{1 - \rho} + 32 \rho^{K} \eta^2 n^2 \nu^2 + 24 \eta^2 n^2 (\tau + 1)^2 \sum_{k=1}^K \rho^{K-k} \bignorm{\nabla F(\vx_0^k)}^2.
\end{align*}
Substituting this inequality to \lmmref{lmm:grabUBlmm1}, we obtain
\begin{align}
    F(\vx_n^{K}) - F^*
    &\le \rho^K (F(\vx_0^1) - F^*) + \frac{\eta n L^2}{2} \sum_{k=1}^K \rho^{K-k} \Delta_k^2 - \frac{\eta n}{4} \sum_{k=1}^{K} \rho^{K-k} \bignorm{\nabla F(\vx_0^k)}^2 \nonumber \\
    &\le \rho^K (F(\vx_0^1) - F^*) + \frac{12 \eta^3 n L^2 H^2 \nu^2}{1 - \rho} + 16 \rho^K \eta^3 n^3 L^2 \nu^2 \nonumber \\
    &\,\,\,\,\,\, + 12\eta^3 n^3 L^2 (\tau + 1)^2 \sum_{k=1}^K \rho^{K-k} \bignorm{\nabla F(\vx_0^k)}^2 - \frac{\eta n}{4} \sum_{k=1}^{K} \rho^{K-k} \bignorm{\nabla F(\vx_0^k)}^2 \nonumber \\
    &\le \rho^K (F(\vx_0^1) - F^*) + \frac{24 \eta^2 L^2 H^2 \nu^2}{\mu} + 16 \rho^K \eta^3 n^3 L^2 \nu^2, \label{eq:prop1_1_3}
\end{align}
where the last inequality holds because
\begin{align*}
    12\eta^3 n^3 L^2 (\tau + 1)^2
    &= \frac{\eta n}{4} \cdot 48 \eta^2 n^2 L^2 (\tau + 1)^2 \\
    &\le \frac{\eta n}{4} \cdot \frac{48}{64^2}  \quad\quad\quad\quad\quad\quad\quad\bigopen{\because \eta \le \frac{1}{64nL(\tau + 1)}}\\
    &< \frac{\eta n}{4}.
\end{align*}
The RHS of (\ref{eq:prop1_1_3}) can further be extended as
\begin{align}
    &\bigopen{1 - \frac{\eta n \mu}{2}}^K \bigopen{(F(\vx_0^1) - F^*) + 16 \eta^3 n^3 L^2 \nu^2} + \frac{24 \eta^2 L^2 H^2 \nu^2}{\mu} \nonumber\\
    &< e^{-\frac{\eta n \mu K}{2}} \bigopen{F(\vx_0^1) - F^* + \nu^2/L} + \frac{24 \eta^2 L^2 H^2 \nu^2}{\mu}. \label{eq:prop1_1_4}
\end{align}
Taking derivative of (\ref{eq:prop1_1_4}) with respect to $\eta$, we can obtain $\eta$ that minimizes (\ref{eq:prop1_1_4}) is
\begin{align*}
    \eta = \frac{2}{\mu n K} W_0 \bigopen{\frac{\bigopen{F(\vx_0^1) - F^* + \nu^2/L}\mu^3 n^2 K^2}{192H^2L^2\nu^2}},
\end{align*}
where $W_0$ denotes the Lambert W function. By substituting this $\eta$ to (\ref{eq:prop1_1_4}), we finally obtain
\begin{align}
    F(\vx_n^{K}) - F^* = \tilde{\mathcal{O}} \bigopen{\frac{H^2 L^2 \nu^2}{\mu^3 n^2 K^2}}. \label{eq:prop1_1_5}
\end{align}
In addition, to make use of such $\eta$ to obtain (\ref{eq:prop1_1_5}), the following condition
\begin{align*}
    \frac{2}{\mu n K} W_0 \bigopen{\frac{\bigopen{F(\vx_0^1) - F^* + \nu^2/L}\mu^3 n^2 K^2}{192H^2L^2\nu^2}} \le \frac{1}{64nL(\tau + 1)}
\end{align*}
must hold. Thus, we require
\begin{align*}
    K \gtrsim \kappa (\tau + 1)
\end{align*}
to guarantee the convergence rate.
\end{proof}

\subsection{\texorpdfstring{Lemmas used in \propref{prop:grabUBb}}{Lemmas used in Proposition 10}} \label{sec:grabUB_lmm}
\subsubsection{Proof for \lmmref{lmm:grabUBlmm1}}
\grabUBlmma*
\begin{proof} [Proof of \lmmref{lmm:grabUBlmm1}]
The update process within a $k$-th epoch can be written as:
\begin{align*}
    x_0^{k+1} = x_0^k - \eta n \cdot \frac{1}{n} \sum_{t=1}^{n} \nabla f_{\sigma_{k}(t)}\bigopen{\vx_{t-1}^k}.
\end{align*}
Using smoothness and $\inner{\va, \vb} = -\frac{1}{2} \norm{\va}^2 - \frac{1}{2} \norm{\vb}^2 + \frac{1}{2} \norm{\va - \vb}^2$, we get
\begin{align*}
    F(\vx_0^{k+1})
    &\le F(\vx_0^k) - \eta n \inner{\nabla F(\vx_0^k),\frac{1}{n} \sum_{t=1}^{n} \nabla f_{\sigma_{k}(t)}\bigopen{\vx_{t-1}^k}} + \frac{\eta^2 n^2 L}{2} \bignorm{\frac{1}{n} \sum_{t=1}^{n} \nabla f_{\sigma_{k}(t)}\bigopen{\vx_{t-1}^k}}^2\\
    &=F(\vx_0^k) - \frac{\eta n}{2} \bignorm{\nabla F(\vx_0^k)}^2 - \frac{\eta n}{2} \bignorm{\frac{1}{n} \sum_{t=1}^{n} \nabla f_{\sigma_{k}(t)}\bigopen{\vx_{t-1}^k}}^2\\
    &\mathrel{\phantom{=}} + \frac{\eta n}{2} \bignorm{\nabla F(\vx_0^k)- \frac{1}{n} \sum_{t=1}^{n} \nabla f_{\sigma_{k}(t)}\bigopen{\vx_{t-1}^k}}^2 + \frac{\eta^2 n^2 L}{2} \bignorm{\frac{1}{n} \sum_{t=1}^{n} \nabla f_{\sigma_{k}(t)}\bigopen{\vx_{t-1}^k}}^2\\
    &\le F(\vx_0^k) - \frac{\eta n}{2} \bignorm{\nabla F(\vx_0^k)}^2 + \frac{\eta n}{2} \bignorm{\nabla F(\vx_0^k) - \frac{1}{n} \sum_{t=1}^{n} \nabla f_{\sigma_{k}(t)}\bigopen{\vx_{t-1}^k}}^2,
\end{align*}
where we used $\eta n L < 1$ in the last inequality. In addition, we can expand the last term as
\begin{align*}
    \bignorm{\nabla F(\vx_0^k) - \frac{1}{n} \sum_{t=1}^{n} \nabla f_{\sigma_{k}(t)}\bigopen{\vx_{t-1}^k}}^2
    &= \bignorm{\frac{1}{n} \sum_{t=1}^{n} \nabla f_{\sigma_{k}(t)}(\vx_0^k) - \frac{1}{n} \sum_{t=1}^{n} \nabla f_{\sigma_{k}(t)}\bigopen{\vx_{t-1}^k}}^2 &&\\
    &\le \frac{1}{n} \sum_{t=1}^{n} \bignorm{\nabla f_{\sigma_{k}(t)}(\vx_0^k) - \nabla f_{\sigma_{k}(t)}\bigopen{\vx_{t-1}^k}}^2 &&(\because \text{Jensen's Inequality})\\
    &\le \frac{L^2}{n} \sum_{t=1}^{n} \bignorm{\vx_0^k - \vx_{t-1}^k}^2 &&(\because \text{smoothness})\\
    &\le L^2 \Delta_k^2. &&(\because \Delta_k = \max_{m=1, \cdots , n}  \| \vx_m^k - \vx_0^k \|)
\end{align*}
Combining these two results, we get
\begin{align*}
    F(\vx_0^{k+1}) \le F(\vx_0^k) + \frac{\eta n L^2 \Delta_k^2}{2} - \frac{\eta n}{2} \bignorm{\nabla F(\vx_0^k)}^2.
\end{align*}
Using the PŁ inequality, this inequality becomes
\begin{align*}
    F(\vx_0^{k+1})
    &\le F(\vx_0^k) + \frac{\eta n L^2 \Delta_k^2}{2} - \frac{\eta n}{4} \bignorm{\nabla F(\vx_0^k)}^2 - \frac{\eta n}{4} \bignorm{\nabla F(\vx_0^k)}^2\\
    &\le F(\vx_0^k) + \frac{\eta n L^2 \Delta_k^2}{2} - \frac{\eta n \mu}{2} (F(\vx_0^k) - F^*) - \frac{\eta n}{4} \bignorm{\nabla F(\vx_0^k)}^2.
\end{align*}
Define $\rho := 1 - \frac{\eta n \mu}{2}$. Subtracting $F^*$ on both sides, we get
\begin{align*}
    F(\vx_0^{k+1}) - F^* \le \rho (F(\vx_0^k) - F^*) + \frac{\eta n L^2 \Delta_k^2}{2} - \frac{\eta n}{4} \bignorm{\nabla F(\vx_0^k)}^2.
\end{align*}
This inequality holds for all $k \in \{ 1, \cdots, K\}$. Unrolling for entire epochs gives
\begin{align*}
    F(\vx_0^{K+1}) - F^* \le \rho^K (F(\vx_0^1) - F^*) + \frac{\eta n L^2}{2} \sum_{k=1}^K \rho^{K-k} \Delta_k^2 - \frac{\eta n}{4} \sum_{k=1}^{K} \rho^{K-k} \bignorm{\nabla F(\vx_0^k)}^2.
\end{align*}
This ends the proof of the lemma.
\end{proof}

\subsubsection{Proof for \lmmref{lmm:grabUBlmm2}}
\grabUBlmmb*
\begin{proof} [Proof of \lmmref{lmm:grabUBlmm2}]
We first consider the situation after the first epoch. For $m \in [n]$ and $k \in [K]\setminus \{1\}$, proper additions and subtractions give us
\begin{alignat*}{2}
    \vx_m^k
    &= \vx_0^k & &- \eta \sum_{t=1}^{m} \nabla f_{\sigma_k(t)} \left(\vx_{t-1}^k \right)\\
    &= \vx_0^k & &- \eta \sum_{t=1}^{m} \nabla f_{\sigma_k(t)} \left(\vx_{\sigma_{k-1}^{-1} (\sigma_k (t))-1}^{k-1} \right)\\
    & & &- \eta \sum_{t=1}^{m} \left( \nabla f_{\sigma_k(t)} \left(\vx_{t-1}^k \right) - \nabla f_{\sigma_k(t)} \left(\vx_{\sigma_{k-1}^{-1} (\sigma_k (t))-1}^{k-1} \right) \right)\\
    &= \vx_0^k & &- \eta \sum_{t=1}^{m} \left( \nabla f_{\sigma_k(t)} \left(\vx_{\sigma_{k-1}^{-1} (\sigma_k (t))-1}^{k-1} \right) - \frac{1}{n} \sum_{s=1}^{n} \nabla f_{\sigma_{k-1}(s)} \left( \vx_{s-1}^{k-1} \right) \right)\\
    & & &- \frac{\eta m}{n} \sum_{s=1}^{n} \nabla f_{\sigma_{k-1}(s)} \left( \vx_{s-1}^{k-1} \right) \\
    & & &- \eta \sum_{t=1}^{m} \left( \nabla f_{\sigma_k(t)} \left(\vx_{t-1}^k \right) - \nabla f_{\sigma_k(t)} \left(\vx_{\sigma_{k-1}^{-1} (\sigma_k (t))-1}^{k-1} \right) \right)\\
    &= \vx_0^k & &- \eta \sum_{t=1}^{m} \left( \nabla f_{\sigma_k(t)} \left(\vx_{\sigma_{k-1}^{-1} (\sigma_k (t))-1}^{k-1} \right) - \frac{1}{n} \sum_{s=1}^{n} \nabla f_{\sigma_{k-1}(s)} \left( \vx_{s-1}^{k-1} \right) \right)\\
    & & &- \eta m \nabla F \left( \vx_n^{k-1} \right)\\
    & & &- \frac{\eta m}{n} \sum_{s=1}^{n} \left( \nabla f_{\sigma_{k-1}(s)} \left( \vx_{s-1}^{k-1} \right) - \nabla f_{\sigma_{k-1}(s)} \left( \vx_n^{k-1} \right) \right)\\
    & & &- \eta \sum_{t=1}^{m} \left( \nabla f_{\sigma_k(t)} \left(\vx_{t-1}^k \right) - \nabla f_{\sigma_k(t)} \left(\vx_{\sigma_{k-1}^{-1} (\sigma_k (t))-1}^{k-1} \right) \right).
\end{alignat*}
Here, $\sigma_{k-1}^{-1} (t)$ indicates in which iteration is the $t$-th sample used at the $(k-1)$-th epoch and $\nabla f_{\sigma_k(t)} \bigopen{\vx_{\sigma_{k-1}^{-1} (\sigma_k (t))-1}^{k-1}}$ indicates the gradient with respect to the same sample used in the $t$-th iteration of the $k$-th epoch, but which was computed previously in the $(k-1)$-th epoch. Using the triangle inequality, we gain
\begin{align}
    \bignorm{ \vx_m^k - \vx_0^k }
    &\le \eta \bignorm{ \sum_{t=1}^{m} \left( \nabla f_{\sigma_k(t)} \left(\vx_{\sigma_{k-1}^{-1} (\sigma_k (t))-1}^{k-1} \right) - \frac{1}{n} \sum_{s=1}^{n} \nabla f_{\sigma_{k-1}(s)} \left( \vx_{s-1}^{k-1} \right) \right) } \nonumber \\
    &\,\,\,\,\,\,+ \eta m \bignorm{ \nabla F \left( \vx_n^{k-1} \right) } \nonumber \\
    &\,\,\,\,\,\,+ \frac{\eta m}{n} \bignorm{ \sum_{s=1}^{n} \left( \nabla f_{\sigma_{k-1}(s)} \left( \vx_{s-1}^{k-1} \right) - \nabla f_{\sigma_{k-1}(s)} \left( \vx_n^{k-1} \right) \right) } \nonumber \\
    &\,\,\,\,\,\,+ \eta \bignorm{ \sum_{t=1}^{m} \bigopen{ \nabla f_{\sigma_k(t)} \left(\vx_{t-1}^k \right) - \nabla f_{\sigma_k(t)} \left(\vx_{\sigma_{k-1}^{-1} (\sigma_k (t))-1}^{k-1} \right) }}. \label{eq:902_1}
\end{align}
Here, the first term in (\ref{eq:902_1}) is the term in which \textit{Herding} intervenes and it enables us to gain the upper bound. To do so, we first upper bound the norm of each component as\\
\begin{alignat*}{2}
    & & &\bignorm{ \nabla f_{\sigma_k(t)} \left(\vx_{\sigma_{k-1}^{-1} (\sigma_k (t))-1}^{k-1} \right) - \frac{1}{n} \sum_{s=1}^{n} \nabla f_{\sigma_{k-1}(s)} \left( \vx_{s-1}^{k-1} \right) }\\
    &\le & &\bignorm{ \nabla f_{\sigma_k(t)} \left(\vx_{\sigma_{k-1}^{-1} (\sigma_k (t))-1}^{k-1} \right) - \frac{1}{n} \sum_{s=1}^{n} \nabla f_{\sigma_{k-1}(s)} \left(\vx_{\sigma_{k-1}^{-1} (\sigma_k (t))-1}^{k-1} \right) }\\
    & & &+\bignorm{ \frac{1}{n} \sum_{s=1}^{n} \nabla f_{\sigma_{k-1}(s)} \left(\vx_{\sigma_{k-1}^{-1} (\sigma_k (t))-1}^{k-1} \right) - \frac{1}{n} \sum_{s=1}^{n} \nabla f_{\sigma_{k-1}(s)} \left( \vx_{s-1}^{k-1} \right) }\\
    &\le & & \bigopen{\nu + \tau \bignorm{ \nabla F \left(\vx_{\sigma_{k-1}^{-1} (\sigma_k (t))-1}^{k-1} \right) }} + \frac{L}{n} \sum_{s=1}^{n} \bignorm{ \vx_{\sigma_{k-1}^{-1} (\sigma_k (t))-1}^{k-1} - \vx_{s-1}^{k-1} }\\
    &\le & & \, \nu + \tau \left(\bignorm{ \nabla F \left( \vx_0^k \right) } + \bignorm{ \nabla F \left( \vx_0^{k-1} \right) - \nabla F \left( \vx_0^k \right) } + \bignorm{ \nabla F \left(\vx_{\sigma_{k-1}^{-1} (\sigma_k (t))-1}^{k-1} \right) - \nabla F \left( \vx_0^{k-1} \right) } \right)\\
    & & &+ \frac{L}{n} \sum_{s=1}^{n} \left( \bignorm{ \vx_{\sigma_{k-1}^{-1} (\sigma_k (t))-1}^{k-1} - \vx_0^{k-1} } + \bignorm{ \vx_0^{k-1} - \vx_{s-1}^{k-1} } \right)\\
    &\le & & \, \nu + \tau \left(\bignorm{ \nabla F \left( \vx_0^k \right) } + 2L\Delta_{k-1} \right) + 2L\Delta_{k-1}\\
    &= & & \, \nu + \tau \cdot \bignorm{ \nabla F \left( \vx_0^k \right) } + 2L\left( \tau + 1 \right) \cdot \Delta_{k-1}.
\end{alignat*}
Now, define $z_t^k$ as
\begin{align*}
    z_t^k := \frac{\nabla f_{\sigma_k(t)} \left(\vx_{\sigma_{k-1}^{-1} (\sigma_k (t))-1}^{k-1} \right) - \frac{1}{n} \sum_{s=1}^{n} \nabla f_{\sigma_{k-1}(s)} \left( \vx_{s-1}^{k-1} \right)}{\nu + \tau \cdot \bignorm{ \nabla F \left( \vx_0^k \right) } + 2L\left( \tau + 1 \right) \cdot \Delta_{k-1}}
\end{align*}
for $t \in [n]$. Then, $\bignorm{z_t^k} \le 1$ holds.\\
We now apply Herding algorithm to upper bound the first term of (\ref{eq:902_1}).
Since $\|z_t^k\| \le 1$, we then get following inequality for all $m \in [n]$:
\begin{align}
    \bignorm{ \sum_{t=1}^{m} \left( \nabla f_{\sigma_k(t)} \left(\vx_{\sigma_{k-1}^{-1} (\sigma_k (t))-1}^{k-1} \right) - \frac{1}{n} \sum_{s=1}^{n} \nabla f_{\sigma_{k-1}(s)} \left( \vx_{s-1}^{k-1} \right) \right) } \nonumber \\
    \le H \left( \nu + \tau \cdot \bignorm{ \nabla F \left( \vx_0^k \right) } + 2L\left( \tau + 1 \right) \cdot \Delta_{k-1} \right). \label{eq:902_2}
\end{align}
For the remaining terms in (\ref{eq:902_1}), we can upper bound each of them by
\begin{align}
    \bignorm{ \sum_{s=1}^{n} \left( \nabla f_{\sigma_{k-1}(s)} \left( \vx_{s-1}^{k-1} \right) - \nabla f_{\sigma_{k-1}(s)} \left( \vx_n^{k-1} \right) \right) }
    &\le \sum_{s=1}^{n} \bignorm{ \nabla f_{\sigma_{k-1}(s)} \left( \vx_{s-1}^{k-1} \right) - \nabla f_{\sigma_{k-1}(s)} \left( \vx_n^{k-1} \right) } \nonumber \\
    &\le L\sum_{s=1}^{n} \bignorm{ \vx_{s-1}^{k-1} - \vx_n^{k-1} } \nonumber \\
    &\le L\sum_{s=1}^{n} \left( \bignorm{ \vx_{s-1}^{k-1} - \vx_0^{k-1} } + \bignorm{ \vx_0^{k-1} - \vx_n^{k-1} } \right) \nonumber \\
    &\le 2nL \Delta_{k-1} \label{eq:902_3}
\end{align}
and
\begin{align}
    &\bignorm{ \sum_{t=1}^{m} \left( \nabla f_{\sigma_k(t)} \left(\vx_{t-1}^k \right) - \nabla f_{\sigma_k(t)} \left(\vx_{\sigma_{k-1}^{-1} (\sigma_k (t))-1}^{k-1} \right) \right) } \nonumber\\
    &\le \sum_{t=1}^{m} \bignorm{ \nabla f_{\sigma_k(t)} \left(\vx_{t-1}^k \right) - \nabla f_{\sigma_k(t)} \left(\vx_{\sigma_{k-1}^{-1} (\sigma_k (t))-1}^{k-1} \right) } \nonumber \\
    &\le L \sum_{t=1}^{m} \bignorm{ \vx_{t-1}^k - \vx_{\sigma_{k-1}^{-1} (\sigma_k (t))-1}^{k-1} } \nonumber \\
    &\le L \sum_{t=1}^{m} \left( \bignorm{ \vx_{t-1}^k - \vx_0^k } + \bignorm{ \vx_0^k - \vx_0^{k-1} } + \bignorm{ \vx_0^{k-1} - \vx_{\sigma_{k-1}^{-1} (\sigma_k (t))-1}^{k-1} } \right) \nonumber \\
    &\le mL \left( \Delta_{k} + 2\Delta_{k-1} \right). \label{eq:902_4}
\end{align}
By summing up (\ref{eq:902_2})-(\ref{eq:902_4}) and taking a max over $m \in \{1, \cdots, n \}$ on both side of (\ref{eq:902_1}),
\begin{align*}
    \Delta_k
    &\le \eta H \left( \nu + \tau \cdot \left \Vert \nabla F \left( \vx_0^k \right) \right \Vert + 2L\left( \tau + 1 \right) \cdot \Delta_{k-1} \right)\\
    &\,\,\,\,+ \eta n \left \Vert \nabla F \left( \vx_0^k \right) \right \Vert + \frac{\eta n}{n} \cdot 2nL \Delta_{k-1} + \eta nL \left( \Delta_{k} + 2\Delta_{k-1} \right)\\
    &\le \eta H \nu + (\eta H \tau + \eta n) \cdot \left \Vert \nabla F \left( \vx_0^k \right) \right \Vert + \left( 2\eta HL(\tau+1) + 4\eta nL \right) \cdot \Delta_{k-1} + \eta nL \Delta_{k}.
\end{align*}
Using $\eta nL \le \frac{1}{2}$, we finally get
\begin{align*}
    \Delta_k \le 2\eta H \nu + (2\eta H \tau + 2\eta n) \cdot \left \Vert \nabla F \left( \vx_0^k \right) \right \Vert + \left( 4\eta HL(\tau+1) + 8\eta nL \right) \cdot \Delta_{k-1}.
\end{align*}
We now move on to the first epoch case. By properly decomposing the term, we gain
\begin{alignat*}{2}
    \vx_m^1
    &= \vx_0^1 & &- \eta \sum_{t=1}^{m} \nabla f_{\sigma_1(t)} \left( \vx_{t-1}^{1} \right)\\
    &= \vx_0^1 & &- \eta \sum_{t=1}^{m} \frac{1}{n} \sum_{s=1}^{n} \nabla f_{\sigma_1(s)} \left( \vx_0^{1} \right)\\
    & & &- \eta \sum_{t=1}^{m} \left( \nabla f_{\sigma_1(t)} \left( \vx_0^1 \right) - \frac{1}{n} \sum_{s=1}^{n} \nabla f_{\sigma_1(s)} \left( \vx_0^{1} \right) \right)\\
    & & &- \eta \sum_{t=1}^{m} \left( \nabla f_{\sigma_1(t)} \left( \vx_{t-1}^{1} \right) - \nabla f_{\sigma_{1} (t)} \left( \vx_0^1 \right) \right).
\end{alignat*}
In a similar way to the technique we used above, we have
\begin{align*}
    \bignorm{\vx_m^1 - \vx_0^1}
    &\le \eta \bignorm{\sum_{t=1}^{m} \frac{1}{n} \sum_{s=1}^{n} \nabla f_{\sigma_1(s)} \left( \vx_0^{1} \right)}\\
    &\,\,\,\,\,\,+\eta \bignorm{\sum_{t=1}^{m} \left( \nabla f_{\sigma_1(t)} \left( \vx_0^1 \right) - \frac{1}{n} \sum_{s=1}^{n} \nabla f_{\sigma_1(s)} \left( \vx_0^{1} \right) \right)}\\
    &\,\,\,\,\,\,+\eta \bignorm{\sum_{t=1}^{m} \left( \nabla f_{\sigma_1(t)} \left( \vx_{t-1}^{1} \right) - \nabla f_{\sigma_{1} (t)} \left( \vx_0^1 \right) \right)}\\
    &\le \eta \sum_{t=1}^{m} \bignorm{\nabla F \bigopen{\vx_0^{1}}}\\
    &\,\,\,\,\,\,+\eta \sum_{t=1}^{m} \bignorm{\nabla f_{\sigma_1(t)} \left( \vx_0^1 \right) - \frac{1}{n} \sum_{s=1}^{n} \nabla f_{\sigma_1(s)} \left( \vx_0^{1} \right)}\\
    &\,\,\,\,\,\,+\eta \sum_{t=1}^{m} L \bignorm{\vx_{t-1}^1 - \vx_0^1}.
\end{align*}
Taking a max over $m \in \{1, \cdots, n\}$ in both sides, we gain
\begin{align*}
    \Delta_1 \le \eta n \left \Vert \nabla F \left( \vx_0^1 \right) \right \Vert + \eta n \left( \nu + \tau \cdot \left \Vert \nabla F \left( \vx_0^1 \right) \right \Vert \right) + \eta n L \Delta_1
\end{align*}
and using the fact that $\eta n L \le \frac{1}{2}$, we finally obtain
\begin{align*}
    \Delta_1 \le 2 \eta n \nu + 2 \eta n (\tau+1) \cdot \bignorm{\nabla F(\vx_0^1)}.
\end{align*}
\end{proof}

\end{document}